\documentclass{article}

    \PassOptionsToPackage{numbers, sort&compress}{natbib}

\usepackage[preprint]{neurips_2025}

\usepackage[utf8]{inputenc} 
\usepackage[T1]{fontenc}    
\usepackage{hyperref}       
\usepackage{url}            
\usepackage{booktabs}       
\usepackage{amsfonts}       
\usepackage{nicefrac}       
\usepackage{microtype}      
\usepackage{xcolor}         

\usepackage{amsmath,amsthm,amssymb,bbm}
\usepackage{subfigure}
\usepackage{multirow,enumitem}
\usepackage{mathtools}
\usepackage{wrapfig}
\usepackage{algorithm2e}

\newtheorem{lemma}{Lemma}

\newtheorem{prop}{Proposition}

\ifodd 1
\newcommand{\com}[1]{{\color{red}\textbf{Parinaz's Comment}: #1}}
\newcommand{\resp}[1]{{\color{cyan}\textbf{Response}: #1}} 
\else

\newcommand{\com}[1]{}
\newcommand{\comr}[1]{}
\newcommand{\resp}[1]{}
\fi

\title{Friends in Unexpected Places: Enhancing Local Fairness in Federated Learning through Clustering}

\author{Yifan Yang \\
The Ohio State University\\
\texttt{yang.5483@osu.edu} \\
\And
Ali Payani \\
Cisco Research \\
\texttt{apayani@cisco.com} \\
\And
Parinaz Naghizadeh \\
UC, San Diego\\
\texttt{parinaz@ucsd.edu}
}

\begin{document}

\maketitle

\begin{abstract}
Federated Learning (FL) has been a pivotal paradigm for collaborative training of machine learning models across distributed datasets. In heterogeneous settings, it has been observed that a single shared FL model can lead to low local accuracy, motivating \emph{personalized FL} algorithms. In parallel, fair FL algorithms have been proposed to enforce group fairness on the global models. Again, in heterogeneous settings, global and local fairness do not necessarily align, motivating the recent literature on \emph{locally fair FL}. In this paper, we propose new FL algorithms for heterogeneous settings, \emph{spanning the space between personalized and locally fair FL}. Building on existing clustering-based personalized FL methods, we incorporate a new fairness metric into cluster assignment, enabling a tunable balance between local accuracy and fairness. Our methods match or exceed the performance of existing locally fair FL approaches, without explicit fairness intervention. We further demonstrate (numerically and analytically) that personalization \emph{alone} can improve local fairness and that our methods exploit this alignment when present.
\end{abstract}

\section{Introduction} \label{sec:intro}


Federated Learning (FL) has been the pivotal paradigm for collaboratively training machine learning models across distributed datasets/clients in a privacy-preserving manner \citep{kairouz2021advances}. Shared, \emph{global} models learned through FL can effectively aggregate gradient information from multiple clients, and (potentially) outperform \emph{standalone} models--- those trained individually by each client in the absence of collaboration. 
However, when clients have heterogeneous datasets, the convergence speed of FL algorithms can considerably deteriorate \citep{li2019convergence,zhao2018federated}, and clients with less typical data distributions may experience low local accuracy \citep{tan2022towards,karimireddy2020scaffold,li2020federated}. To address these limitations, a spectrum of \emph{personalized FL} techniques (e.g., \cite{li2020federated,ghosh2020efficient,briggs2020federated,sattler2020clustered,fallah2020personalized,mansour2020three}) have been proposed to enhance the \emph{local accuracy} of the learned models while keeping some of the benefits of collaborative learning. 

Beyond ensuring model accuracy, it has become increasingly important to train machine learning models that satisfy \emph{(group) fairness} \citep{barocas-hardt-narayanan}. This means ensuring that the learned models treat individuals from different demographic groups equally by, e.g., maintaining similar selection rates or true positive rates across groups defined by sensitive attributes (e.g., race, gender). Motivated by this, a number of works have proposed \emph{fair FL} algorithms; see \cite{shi2023towards,salazar2024survey} for surveys. However, most of these existing works focus on \emph{global} fairness--- fairness of a shared, global model assessed on the global data distribution. Such model is not necessarily \emph{locally} fair when clients have heterogeneous datasets. For example, both population demographics and healthcare data distributions vary geographically \citep{swift2002guidance}. Then, a fair federated learning model trained across state hospitals can satisfy a desired fairness constraint at the state level, but may still be unfair at individual hospitals if local demographics differ significantly. To address this challenge, recent works have formally studied tradeoffs between local and global fairness \cite{hamman2023demystifying}, and proposed FL algorithms that can achieve local fairness \citep{meerza2024glocalfair, zhang2025sffl, makhija2024achieving, zhou2025post}. 

\begin{wrapfigure}[17]{r}{0.5\textwidth}\vspace{-0.24in}
	\centering
	\subfigure{
		\includegraphics[width=0.5\textwidth]{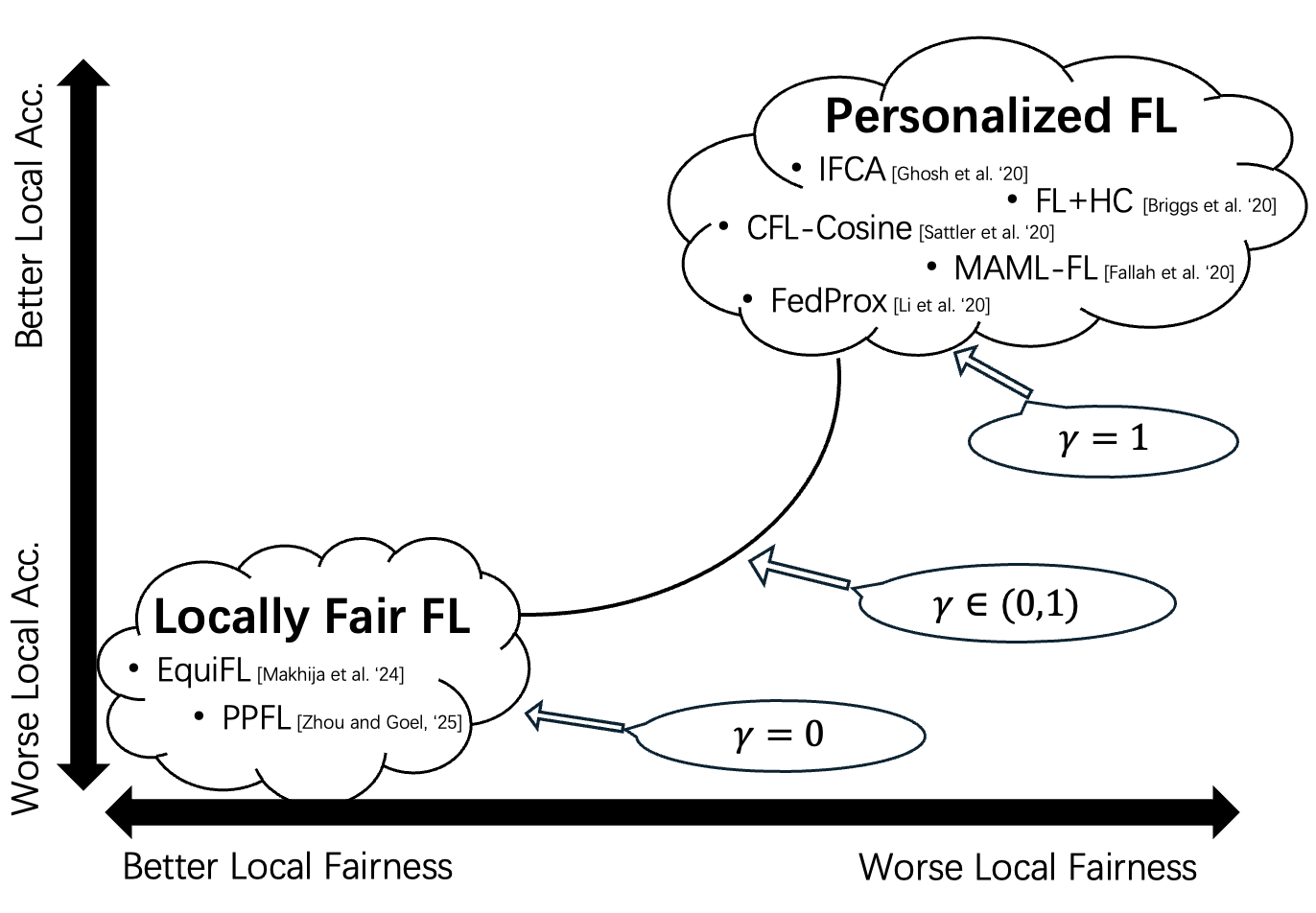}
	}
	\vspace{-0.27in}
	\caption{Our proposed algorithms span the gap between personalized FL and locally fair FL methods, and lead to a new class of locally fair methods. The parameter $\gamma\in[0,1]$ adjusts the balance between local accuracy and local fairness.}
	\label{fig:big-picture}
\end{wrapfigure} 
In this paper, we are similarly interested in federated learning algorithms that can perform well \emph{locally} when clients have heterogeneous datasets, and in particular, focus on attaining a tunable fairness-accuracy tradeoff at the local client level. \emph{We propose federated learning algorithms that span the range between personalized federated learning and locally fair federated learning algorithms.} See Figure~\ref{fig:big-picture} for an illustration. In this plot, the $x$-axis and $y$-axis show the local fairness gap and the local accuracy, respectively. 
We find that personalized FL algorithms \cite{ghosh2020efficient,briggs2020federated,sattler2020clustered,fallah2020personalized,li2020federated} improve local accuracy at the expense of fairness, whereas locally fair FL algorithms \cite{makhija2024achieving, zhou2025post} sacrifice local accuracy to improve fairness guarantees locally. 

To cover this gap, our proposed algorithms, \texttt{Fair-FCA} and \texttt{Fair-FL+HC}, take inspiration from clustering methods used for personalization in FL \cite{ghosh2020efficient,briggs2020federated}. They integrate a new fairness metric into the cluster assignment step, allowing us to improve local fairness while balancing accuracy and fairness at the client level. Our methods can be viewed as a way of strategically grouping either ``similar'' or ``useful'' clients together. 
Importantly, depending on the desired local fairness-accuracy tradeoff, these may end up being the clients with the most similar or the most complementary (in terms of demographic statistics) data distributions. 

Specifically, our algorithms use a parameter $\gamma\in[0,1]$ to adjust the balance between local accuracy and local fairness during cluster assignment, with larger $\gamma$'s indicating more focus on local accuracy. At $\gamma=1$, our proposed algorithms, \texttt{Fair-FCA} and \texttt{Fair-FL+HC}, reduce to the existing \texttt{IFCA} \citep{ghosh2020efficient} and \texttt{FL+HC} \citep{briggs2020federated} which they are built on. At $\gamma=0$, our algorithms are of independent interest, as they present two new algorithms for improving local fairness in FL. Existing approaches to locally fair FL \citep{meerza2024glocalfair, zhang2025sffl, makhija2024achieving, zhou2025post} aim to either prevent bias propagation during collaborative training by mitigating the local unfairness or reducing the influence of biased clients via weighted aggregation, or modify predictions probabilistically (post-training) to satisfy fairness constraints. In contrast, we show that our methods can achieve comparable and at times better local fairness by simply clustering clients without any additional interventions (i.e., without explicit pre-, in-, or post-processing steps). 

Finally, we explore the connections between personalization and local fairness in FL. Specifically, we show that personalization \emph{alone} can enhance local fairness. We illustrate this alignment between personalization and fairness through extensive numerical experiments on a range of personalized FL methods (clustering \cite{ghosh2020efficient,briggs2020federated,sattler2020clustered}, local-finetuning \cite{fallah2020personalized}, and regularization-based \cite{li2020federated}), on real-world (\texttt{Adult} \citep{Dua:2019} and \texttt{Retiring Adult} \citep{ding2021retiring}) and synthetic datasets, and for several notions of group fairness (statistical parity, equality of opportunity, and equalized odds \cite{barocas-hardt-narayanan}). We highlight two potential factors driving the alignment: \emph{statistical} advantages of increased data diversity due to the collaborative nature of FL, and \emph{computational} advantages when there is an alignment in local accuracy and fairness. We note that our proposed algorithms are effectively taking advantage of the alignment of personalization and fairness, whenever possible, to improve the local fairness-accuracy tradeoff. 

\textbf{Summary of findings and contributions.}\\
\emph{1. New fairness-aware and personalized federated learning algorithms.} We propose two new algorithms, \texttt{Fair-FCA} and \texttt{Fair-FL+HC}, which take both local accuracy and local fairness into account when (iteratively) determining clients' cluster memberships, allowing for a tunable fairness-accuracy trade-off at the client level.\\
\emph{2. New clustering-based algorithms for improving local fairness.} Our proposed tuneable algorithms lead to new methods for locally fair FL. In contrast to existing methods for locally fair FL, our methods do not include any explicit fairness interventions (pre-, in-, or post-processing steps) yet can achieve comparable or better local fairness by simply clustering clients.\\ 
\emph{3. Unintended fairness benefits of personalization.}  We conduct extensive numerical experiments, under different notions of fairness, and using both real-world and synthetic data, to show that personalization \emph{alone} can improve local fairness as an unintended benefit. This is an alignment our algorithms are exploiting when possible. We highlight the potential statistical and computational reasons leading to this alignment, and provide analytical support under certain conditions (Propositions ~\ref{prop3} and \ref{prop1}).

\section{Related Work} \label{sec:related}
\vspace{-0.1in}

Our work is at the intersection of two literatures: personalized FL and fairness in FL. We review works on personalized FL including clustering-based methods (and other related work) in Appendix~\ref{app:related}. In terms of fairness in FL, we note that this term has taken different interpretations in the FL literature. Much of the early works in fair FL \citep{li2019fair,li2021ditto,zhang2021unified,wang2021federated} primarily focused on \emph{performance fairness}, which seeks to achieve uniform accuracy across all clients. However, even if a trained model attains uniform performance, it can still exhibit bias against certain demographic groups. We in contrast, focus on notions of \emph{group fairness} in FL \cite{barocas-hardt-narayanan,salazar2024survey}. Even within this literature, the majority of the works has focused on improving \emph{global} group fairness \cite{abay2020mitigating,galvez2021enforcing,wang2023mitigating,zeng2021improving,ezzeldin2023fairfed,liu2025fairness}. In contrast, we focus on \emph{local} group fairness. 

Our work is most closely aligned with the recent works on improving local group fairness in FL \cite{meerza2024glocalfair,zhang2025sffl,makhija2024achieving,zhou2025post}. 
\citet{meerza2024glocalfair} integrate local fairness constraints with fairness-aware clustering-based aggregation, leveraging Gini coefficients to jointly enhance both global and local group fairness. Similarly, \citet{zhang2025sffl} incorporate locally fair training using the EM algorithm and adjust aggregation weights through reweighting based on distance to achieve fair aggregation. \citet{makhija2024achieving} enforce fairness constraints in the local optimization problem to prevent bias propagation during collaboration. \citet{zhou2025post} introduce fairness post-processing techniques (model output fairness post-processing and final layer fairness fine-tuning) to improve local fairness. Compared to these works, our approach to attaining local fairness is different: we demonstrate that improved local fairness, along with a better fairness-accuracy tradeoff, can be achieved by clustering clients based on a fairness-aware assignment metric (or even through personalization alone).

Lastly, some existing works have, similar to ours, noticed connections between fairness and personalization/clustered FL techniques, but they differ from ours in either their notion of fairness, context, or scope. 
\citet{wang2024analyzing} use healthcare data to demonstrate that Ditto \cite{li2021ditto} (a personalized FL algorithm which enhances performance fairness) achieves better local group fairness compared to standalone learning. \citet{nafea2022proportional} add a notion of fairness into cluster identity assignment (similar to us) but to ensure proportional fairness among protected groups (a notion different from group fairness). \citet{kyllo2023inflorescence} examine the impact of fairness-unaware clustering on a number of fairness notions. In contrast, we study how and why a \emph{range} of personalized FL algorithms (not just clustering-based) may improve local fairness. 
\vspace{-0.1in}
\section{Fairness-Aware Federated Clustering Algorithms}\label{sec:new-algorithm}
\vspace{-0.1in}

Our goal is to develop FL algorithms that strike a tunable balance between local accuracy and local fairness, spanning the range between personalized FL and locally fair FL methods. To this end, we start from clustering-based personalized FL approaches, which improve local accuracy in heterogeneous settings. 
The underlying idea of these methods is that clients with similar data distributions (assessed based similarity of their model performance \citep{ghosh2020efficient}, model parameter \citep{briggs2020federated}, or gradient updates \citep{sattler2020clustered}) will benefit from forming smaller ``teams'' together. Inspired by this, \emph{we instead allow clients to join forces based on both fairness and accuracy benefits}, which as we show, may be due to them having similar \emph{or} (appropriately) different local datasets. 

We illustrate the viability of this idea by proposing \texttt{Fair-FCA} and \texttt{Fair-FL+HC}, which build on the existing \texttt{IFCA} algorithm \citep{ghosh2020efficient} and \texttt{FL+HC} algorithm \citep{briggs2020federated}, respectively. We choose these algorithms as starting points, since many existing algorithms in the clustered FL literature are built on the \texttt{IFCA} framework (e.g., \cite{li2021federated, chung2022federated, huang2023active,ma2024structured}) and the \texttt{FL+HC} framework (e.g., \cite{jothimurugesan2023federated,luo2023privacy, li2023hierarchical,sun2024collaborate}). 

\vspace{-0.1in}
\subsection{Problem setting and preliminaries}\label{sec:setup}
\vspace{-0.1in}

We consider an FL setting with $N$ clients, where each client $i$ is tasked with a binary classification problem. The client's dataset consists of samples $z = (x,y,g)$, where $x \in \mathbb{R}^d$ represents the feature vector, $y \in \{0,1\}$ is the true label, and $g \in \{a,b\}$ is a binary protected attribute (e.g., race, sex). A client $i$ has access to $n_i$ such samples, $Z_i:=\{z_{ij}\}^{n_i}_{j=1}$, drawn independently from a joint feature-label-group distribution with probability density functions $\mathcal{G}^{y,i}_g(x)$. These distributions differ across clients, which causes the conflicts between model accuracy/fairness at local and global levels. 

\emph{Evaluating local accuracy.} Let $f(z, \theta)$ denote the loss function associated with data point $z$ under model $\theta$. This could be, for instance, the misclassification loss. Then, the local empirical loss of client $i$ is given by $F(Z_i,\theta):= \frac{1}{n_i}\sum_{j=1}^{n_i} f(z_{ij},\theta)$.  

\emph{Evaluating local fairness.} 
Consider a learned model $\theta$, and let $\hat{y}(\theta)$ denote the labels assigned by it. 
We assess the \emph{group fairness} of $\theta$ according to three commonly studied notions of group fairness.
\begin{enumerate}
    \item \emph{Statistical Parity} (\texttt{SP}) \citep{dwork2012fairness} assesses the gap between the selection rate of each group: (i.e., $\Delta_{\texttt{SP}}(\theta):=|\mathbb{P}(\hat{y}(\theta) = 1 |g = a) - \mathbb{P}(\hat{y}(\theta) = 1 |g = b)|$;
    \item \emph{Equality of Opportunity} (\texttt{EqOp}) \citep{hardt2016equality} finds the gap between true positive rates on each group: (i.e., $\Delta_{\texttt{EqOp}}(\theta) := |\mathbb{P}(\hat{y}(\theta) = 1 |g = a, y = 1) - \mathbb{P}(\hat{y}(\theta) = 1 |g = b, y = 1)|)$;
    \item \emph{Equalized Odd} (\texttt{EO}) \citep{hardt2016equality} is set to the gap between true positive or false positive rates between groups, whichever larger (i.e., $\Delta_{\texttt{EO}}(\theta) := \max_{i\in\{0,1\}}|\mathbb{P}(\hat{y}(\theta) = 1 |g = a, y = i) - \mathbb{P}(\hat{y}(\theta) = 1 |g = b, y = i)|$). 
\end{enumerate}
Here, the probability is with respect to the data distributions $\mathcal{G}^{y,i}_g(x)$ of client $i$. These fairness metrics can be evaluated empirically on the client's data realization $Z_i$ (setting the probabilities to the number of data points satisfying its condition divided by the group sample size). Let $\Psi^f(Z_i, \theta)$ denote the empirical local fairness of model $\theta$ for fairness metric $f \in \{\texttt{SP}, \texttt{EqOp}, \texttt{EO} \}$, assessed on $Z_i$. 

\subsection{Integrating fairness metrics in cluster identity assignment}


\noindent\textbf{The \texttt{Fair-FCA} algorithm.} This algorithm iterates over two steps: (1) cluster identity assignment, and (2) training of cluster-specific models. Specifically, let $\Theta^t_k$ denote cluster $k$'s model at time step $t$. The cluster identity for client $i$ at time $t$, denoted $c^t(i)$, is determined by:
\begin{equation}
    c^{t}(i) = \arg\min_{k \in [K]} \gamma F(Z_i,\Theta^{t}_k) + (1-\gamma) \Psi^f(Z_i,\Theta^{t}_k) \label{eq: cluster_assignment}
\end{equation}
Here, $K$ be the total number of clusters (a hyperparameter), and $\gamma$ is a hyperparameter that strikes a desired balance between accuracy and fairness. For $\gamma=1$, we recover the \texttt{IFCA} algorithm; for $\gamma=0$, we obtain a clustered FL algorithm that prioritizes only (local) $f$-fairness when grouping clients. For $0<\gamma<1$, we obtain clusters that provide each client with the best fairness-accuracy tradeoff among those attainable if the client were to join each cluster. 

Let $C^t_k$ be the set of clients whose cluster identity is $k$ at the end of this assignment process (i.e., $C^t_k = \{i \in [n]: c^t(i) = k\}$). Once clients get assigned clusters, each client $i$ starts from its corresponding cluster model $\Theta^t_{c^t(i)}$, and locally runs gradient steps  $\theta_i^t = \Theta^t_{c^t(i)} - \eta \nabla_{\theta_i} F(Z_i,\Theta^t_{c^t(i)})$ to update it. Then, the updated local models $\theta_i^t$ are sent to the central server, who uses these to update the cluster models to $\Theta^{t+1}_{1:K}$ by taking the weighted average of the local models of clients in corresponding clusters. Formally, $\Theta^{t+1}_k = \Theta^t_k - \sum_{i \in C^t_k} \frac{n_i}{\sum_i n_i} (\Theta^t_k  - \theta^t_i), \forall k \in [K]$. The pseudo-code for \texttt{Fair-FCA} is shown in Algorithm~\ref{alg:one-one}. 

\SetKwComment{Comment}{/* }{ */}
\RestyleAlgo{ruled}
\begin{algorithm}[t]
\caption{\texttt{Fair-FCA}}\label{alg:one-one}

\textbf{Input}: Number of clusters $K$, number of clients $N$, number of local updates $E$, cluster model initialization $\Theta_{1:K}$, learning rate $\eta$, fairness-accuracy tradeoff $\gamma$, fairness $f \in \{\texttt{SP}, \texttt{EqOp}, \texttt{EO}\}$.\\
\textbf{Initialize}: Start clusters $k\in[K]$ by randomly selecting one client for each.\\
\While{not converge}{ 
   
  \For{client $i \in [n]$}{
    \textbf{Find} cluster identity:\\
    \hspace{0.3in} $c(i) = \arg\min_{k \in [K]} \gamma F(Z_i,\Theta_k) + (1-\gamma)\Psi^{f}(Z_i,\Theta_k)$\\
    \textbf{Initialize} $\theta_i = \Theta_{c(i)}$\\
    \textbf{Perform} $E$ steps of local update \\
    \hspace{0.3in} $\theta_i = \theta_i - \eta \nabla_{\theta_i} F(Z_i,\theta_{i})$ \\
    \textbf{Upload} $\theta_i$ to server
  }
  \textbf{Update} the cluster model $\Theta_{1:K}$\\
  \hspace{0.3in} $\Theta_k = \Theta_k - \sum_{i \in C_k} \frac{n_i}{\sum_i n_i} (\Theta_k  - \theta_i)$\\
  \textbf{Send} new cluster models $\Theta_{1:K}$ to all clients
}
\textbf{Output}: Cluster models $\Theta_{1:K}$, Cluster identity $c(i), \forall i \in [n]$.
\end{algorithm}

\noindent\textbf{The \texttt{Fair-FL+HC} algorithm.}
Initially, the algorithm runs the regular \texttt{FedAvg} procedure for a predetermined number of rounds before clustering. Once a global model $\theta^{FA}$ is obtained, each client receives the model and performs several local updates to personalize their local models $\theta_i$. 

Like the \texttt{FL+HC} algorithm, the \texttt{Fair-FL+HC} also employs hierarchical clustering with a set of hyperparameters $P$ to group clients by minimizing intra-cluster variance, measured using the $L_2$ Euclidean distance metric. We extend this approach by also considering fairness performance. 
\begin{equation}
    \text{Clusters} = \text{HierarchicalClustering}(\gamma \mathbf{D} + (1-\gamma) \mathbf{\Psi}^f(Z, \theta), P)
     \label{eq: cluster_assignment_Fair-FL+HC}
\end{equation}
Here, $\mathbf{D}$ is a symmetric matrix where each entry  $D_{i,j}$ represents the Euclidean distance between $\theta_i$ and $\theta_j$. $\mathbf{\Psi}^f(Z, \theta):= max (\Psi^f(Z_i, \theta_j), \Psi^f(Z_j, \theta_i))$, with $f \in \{\texttt{SP},\texttt{EqOp},\texttt{EO}\}$ is also a symmetric matrix that captures the worst-case $f$-fairness performance when client $i$'s model is evaluated on client $j$'s local data, or vice versa. The parameter $\gamma$ balances between fairness and accuracy considerations; when $\gamma=1$, we recover the \texttt{FL+HC} algorithm. Once clustering is completed, each cluster trains its model independently using \texttt{FedAvg}. The pseudo-code 
for \texttt{Fair-FL+HC} 
is shown in Appendix~\ref{app:pesudocode}.

\subsection{Comparison with existing locally fair FL algorithms ($\gamma=0$)}\label{subsec:comparison_gamma0}

We first benchmark our proposed methods against the baseline \texttt{FedAvg} algorithm \citep{mcmahan2017communication}, as well as two fair FL algorithms specifically designed to improve local group fairness: \texttt{EquiFL} \citep{makhija2024achieving} and \texttt{PPFL} \citep{zhou2025post}. The \texttt{EquiFL} algorithm \citep{makhija2024achieving}, an in-processing approach, enforces fairness constraints during local model training. The \texttt{PPFL} algorithm \citep{zhou2025post}, a post-processing approach, probabilistically adjusts model predictions after training to satisfy fairness criteria. Our code is available at: \url{https://github.com/Yifankevin/Enhancing-Local-Fairness-in-Federated-Learning-through-Clustering}.

We begin our experiments with a synthetic dataset, focusing on average local statistical parity (\texttt{SP}) fairness. We randomly generate 6 clients, each with different levels of imbalance in the number of samples (e.g., balanced, mildly imbalanced, highly imbalanced). Since \texttt{SP} fairness depends on both group/label rates and data features, the synthetic setting allows us to isolate and analyze the impact of each factor on fairness performance one at a time. Additional experiments using different parameter settings and real-world datasets, and full experimental details are provided in Appendix~\ref{app_imbalance_tables_and sweep}.

\begin{figure}[ht]
\vspace{-0.1in}
	\centering
	\subfigure[Group Imbalance]{
		\includegraphics[width=0.3\textwidth]{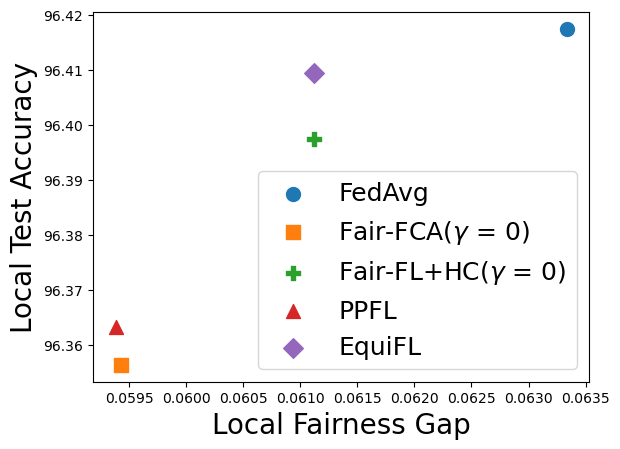} 
	}
	\hspace{-0.1in}
	\subfigure[Label Imbalance]{		\includegraphics[width=0.3\textwidth]{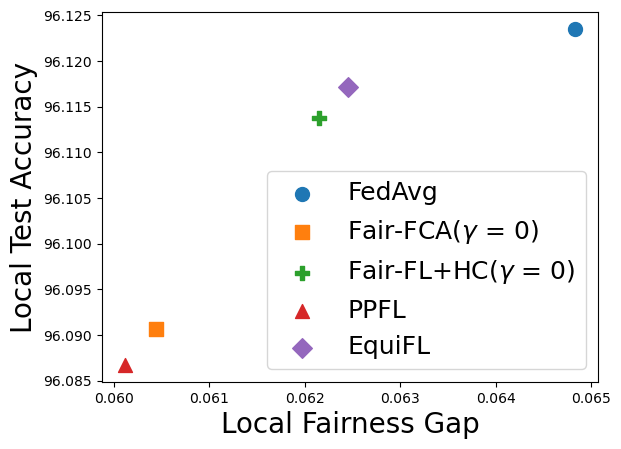}
	}
        \hspace{-0.05in}
	\subfigure[Feature Imbalance]{
\includegraphics[width=0.285\textwidth]{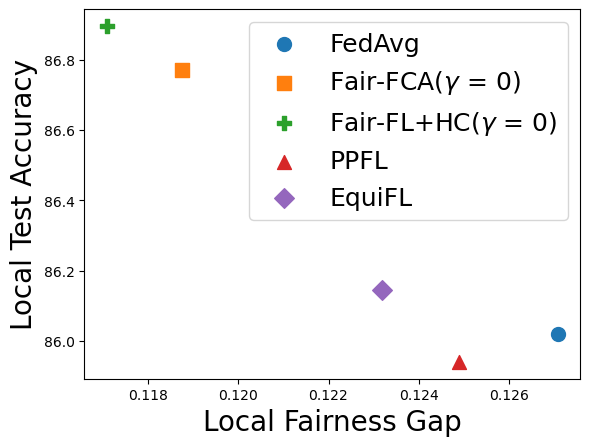} 
	}
 \vspace{-0.1in}
	\caption{\texttt{Fair-FCA} and \texttt{Fair-FL+HC} on synthetic datasets ($\gamma = 0, f=\texttt{SP}$)}
	\label{fig:synthetic_imbalance}
  \vspace{-0.1in}
\end{figure}

Our results in Fig.~\ref{fig:synthetic_imbalance} show that clustering clients by their local fairness metrics improves the fairness performance compared with \texttt{FedAvg} and achieves performance comparable to that of existing methods. Interestingly, when group or label rates are imbalanced, \texttt{Fair‑FCA} outperforms \texttt{Fair‑FL+HC} in terms of local fairness. We attribute this to bias propagation during the warm-start phase in the \texttt{Fair‑FL+HC}, which tends to group highly imbalanced clients with mildly imbalanced ones. In contrast, when feature distributions are imbalanced, \texttt{Fair‑FL+HC} yields better fairness performance than \texttt{Fair‑FCA}. We believe this is due to its hierarchical clustering approach, which is more sensitive to relative differences in client performance and thus better captures distributional mismatches.

This insight is further validated by experiments on the \texttt{Adult} \citep{Dua:2019} and \texttt{Retiring Adult} \citep{ding2021retiring} datasets, shown in Fig.~\ref{fig:experiments_adult_retiring}. The \texttt{Adult} dataset involves predicting whether an individual earns more than \$50k annually based on demographic and socioeconomic features. We randomly generate 5 clients, each with different levels of imbalance in the number of samples based on sex. To explore the impact of more complex distributional differences, we also evaluate on the \texttt{Retiring Adult} dataset, which includes census data from all 50 U.S. states and Puerto Rico. Each state is treated as a client, with data samples consisting of multi-dimensional feature vector $x$ (e.g., age, education, citizenship), a true label $y$, and a protected attribute $g$ (e.g., sex). To amplify distributional differences, we manually scale the feature set ($x$) by 60\% for the states with IDs \{1, 10, 20, 30, 40, 50\}. Our experiments here show the average \texttt{SP} fairness on the ACSIncome (Income) classification task. We again observe that our proposed methods  achieves local fairness comparable to, or better than, that of existing locally fair FL methods. 

\begin{figure}[ht]
\vspace{-0.1in}
	\centering
	\subfigure[Adult (Sex)]{
		\includegraphics[width=0.3\textwidth]{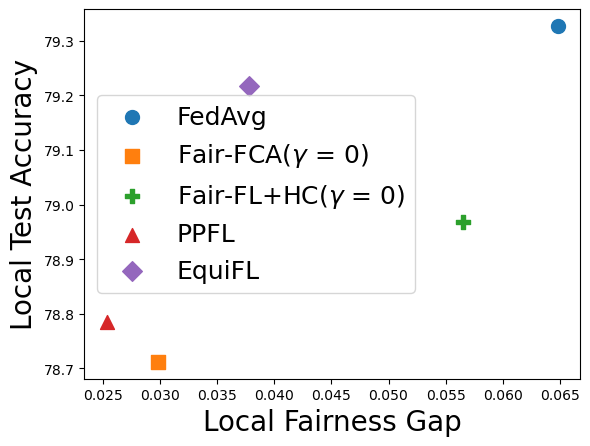} 
	}
        \hspace{-0.05in}
	\subfigure[Income (Sex)]{
\includegraphics[width=0.3\textwidth]{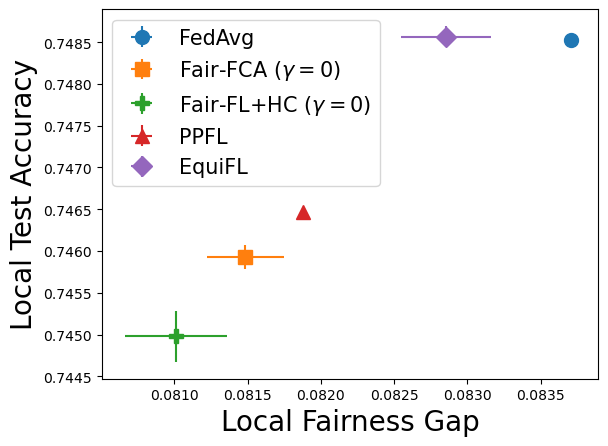} 
	}
 \vspace{-0.1in}
	\caption{\texttt{Fair-FCA} and \texttt{Fair-FL+HC} on \emph{Adult} and \emph{Retiring Adult} datasets ($\gamma = 0, f=\texttt{SP}$)} 
	\label{fig:experiments_adult_retiring}
  \vspace{-0.1in}
\end{figure}

\subsection{Tuneable fairness-accuracy tradeoff using \texttt{Fair-FCA} and \texttt{Fair-FL+HC} $\gamma\in(0,1)$}\label{subsec:tradeoff}
Another merit of our approach is that it can offer a tunable trade-off between accuracy and fairness at the client level. We begin by conducting a numerical experiment on a synthetic dataset to illustrate the ability of \texttt{Fair-FCA} and \texttt{Fair-FL+HC} to strike the desired balance between fairness and accuracy. Additional experiments conducted on real-world datasets are presented in Appendix~\ref{app_experiment_trade_offs}.

We consider a total of 8 clients that could (potentially) be clustered into two clusters. Among these, 6 clients (Client ID: 2,4,5,6,7,8) have similar data distributions, with 4 clients (Client ID: 4,6,7,8) sharing identical distributions across the two protected groups $a,b$ (low fairness gap). The remaining 2 clients (Client ID: 1,3) have different data distributions compared to the first six, but they also share identical distributions across the two protected groups. We consider $f=\texttt{SP}$. Let $\gamma_1, \gamma_2 \in[0,1]$ be the hyperparameters of the \texttt{Fair-FCA} and \texttt{Fair-FL+HC} algorithms, respectively. 

\begin{wrapfigure}[14]{r}{0.42\textwidth}\vspace{-0.3in}
	\centering
	\subfigure{
		\includegraphics[width=0.42\textwidth]{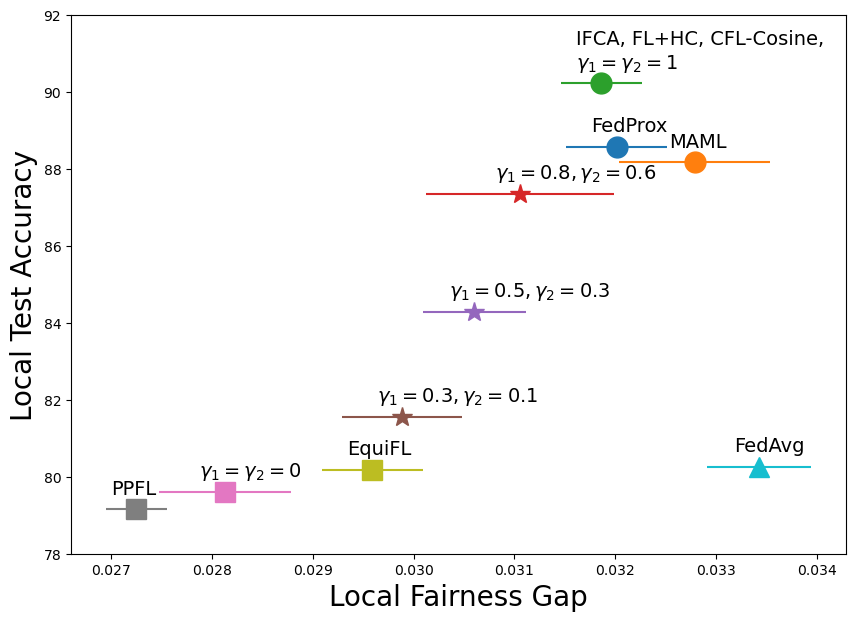}
	}
	\vspace{-0.3in}
	\caption{Comparison of our methods with existing personalized FL and locally fair FL methods, under different $\gamma_1, \gamma_2$.}
	\label{fig:synthetic_tradeoff}
\end{wrapfigure} 
Figure~\ref{fig:synthetic_tradeoff} shows our proposed algorithms span the gap between personalized FL methods (e.g., \texttt{IFCA}, \texttt{FL+HC}, etc.) and locally fair FL methods (e.g., \texttt{PPFL}, \texttt{EquiFL}). When $\gamma_1 = \gamma_2 = 1$, both \texttt{Fair-FCA} and \texttt{Fair-FL+HC} prioritize accuracy, recovering existing \texttt{IFCA} and \texttt{FL+HC} methods; by design, this is attained by grouping the 6 clients having similar data distributions together (\{1,3\} and \{2,4,5,6,7,8\}). In contrast, when $\gamma_1 = \gamma_2 = 0$, both \texttt{Fair-FCA} and \texttt{Fair-FL+HC} focus only on \texttt{SP} fairness by clustering clients that have identical distributions on the two protected groups together (\{2,5\} and \{1,3,4,6,7,8\}), achieving comparable or better local fairness to existing locally fair FL methods. Lastly, by setting $\gamma_1, \gamma_2 \in (0,1)$, we can effectively account for both accuracy and fairness when clustering, covering the gap as desired.


\section{Personalization alone can also improve fairness} \label{sec:numerical}
Our findings in Section~\ref{subsec:comparison_gamma0} demonstrated that our proposed algorithms effectively improve local fairness by setting their tuneable parameter to $\gamma=0$. In Section~\ref{subsec:tradeoff}, we further show that these algorithms attain a tuneable fairness-accuracy trade-off to $\gamma\in(0,1)$. In this section, we now move to the extreme of $\gamma=1$, which leads to the existing \texttt{IFCA} \citep{ghosh2020efficient} and \texttt{FL+HC} \citep{briggs2020federated} algorithms that our methods build on. Interestingly, we show that even without fairness considerations, \emph{personalization alone can still enhance local fairness as an unintended benefit}. One of the advantages of our proposed methods is therefore that they are exploiting this alignment when present.

We consider several classes of personalized FL methods to illustrate this alignment; 
this is to show that our insights on the alignment of personalization and fairness hold irrespective of how personalization is achieved. In more detail, we run experiments on the following personalized FL methods. The \texttt{IFCA} algorithm \citep{ghosh2020efficient} alternates between clustering clients based on model performance and optimizing parameters within each cluster. The \texttt{FL+HC} algorithm \citep{briggs2020federated} employs hierarchical clustering to minimize intra-cluster variance, measured by the Euclidean distance between models. The \texttt{CFL-Cosine} algorithm \citep{sattler2020clustered} partitions clients into two clusters by minimizing the maximum cosine similarity of their gradient updates. Beyond clustering, the \texttt{MAML-FL} algorithm \citep{fallah2020personalized} extends \texttt{FedAvg} by allowing clients to fine-tune the global model through extra local gradient steps. Similarly, the \texttt{FedProx} algorithm \citep{li2020federated} adds $l_2$ regularization to balance local and global model learning. We compare the local accuracy and local fairness of these algorithms against FedAvg and standalone learning. 
We will run these experiments on the \texttt{Adult} \citep{Dua:2019} and \texttt{Retiring Adult} \citep{ding2021retiring} datasets, comparing the average local statistical parity (\texttt{SP}) fairness achieved by different FL algorithms. 
Similar experiments supporting our findings under other notions of fairness (\texttt{EqOp}, \texttt{EO}) are given in Appendix~\ref{app:numerical_other_notion}. We then substantiate our findings with analytical support in Section~\ref{sec:analytical-support-overview}. 

\vspace{-0.1in}
\subsection{Imbalanced groups: statistical advantages of collaboration} \label{subsec: imbalanced}
\vspace{-0.1in}

We first consider the ACSEmployment task in the \texttt{Retiring Adult} dataset with ``race'' as the protected attribute. Fig~\ref{fig:race_emp_a} shows the fraction of samples in each group-label, from several states, highlighting an imbalance between samples from the White and Non-White groups. This is further evident in Figure~\ref{fig:race_emp_b},  which shows that most states have only $\sim 10\%$ qualified (label 1) samples from the Non-White group, in contrast to $\sim 35\%$ qualified samples from the White group. 

\begin{figure}[ht]
\vspace{-0.1in}
	\centering
	\subfigure[Fraction of samples]{
		\includegraphics[width=0.3\textwidth]{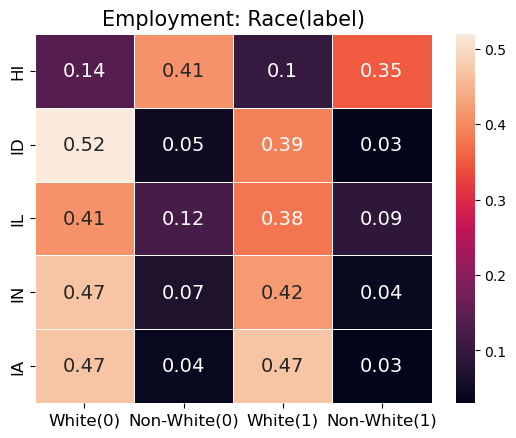} \label{fig:race_emp_a}
	}
	\hspace{-0.15in}
	\subfigure[Normalized sample frequency]{		\includegraphics[width=0.31\textwidth]{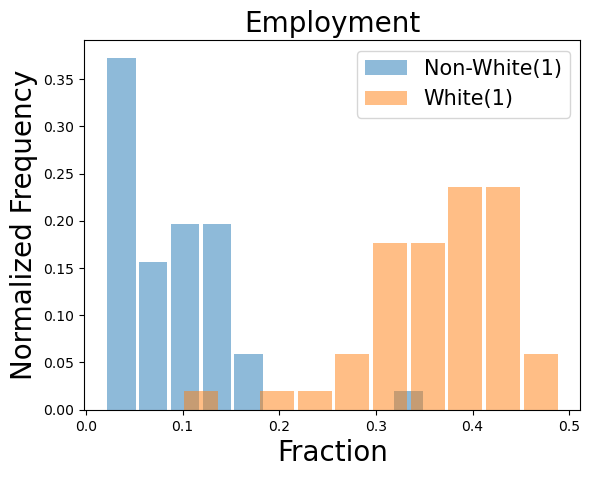} \label{fig:race_emp_b}
	}
        \hspace{-0.05in}
	\subfigure[Local accuracy vs. fairness gap]{
\includegraphics[width=0.31\textwidth]{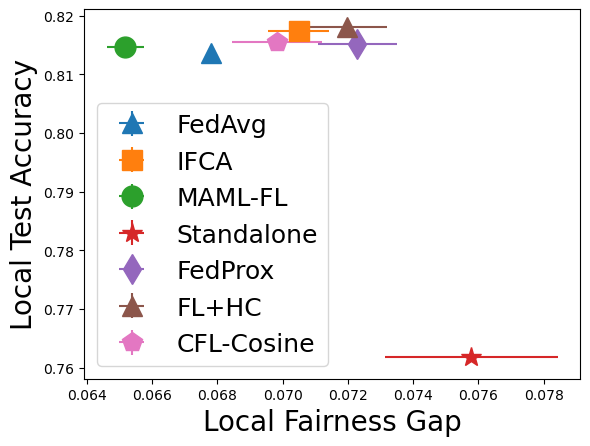} \label{fig:race_emp_c}
	}
 \vspace{-0.1in}
	\caption{Experiments on the ACSEmployment task with imbalanced groups (race).}
	\label{fig:race_emp}
  \vspace{-0.1in}
\end{figure}

Fig~\ref{fig:race_emp_c} shows that all collaborative training algorithms (\texttt{FedAvg}, \texttt{MAML-FL}, \texttt{IFCA}, \texttt{FedProx}, \texttt{FL+HC}, and \texttt{CFL-Cosine}) achieve better local fairness (smaller gap) compared to \texttt{Standalone} learning. This is due to the \emph{statistical benefits} of collaboration: each client has limited samples in the non-White group, leading to poorly trained models with high local fairness gap (and low accuracy). In contrast, collaborative training in essence has access to more data, improving both metrics. For the same reason, the \texttt{IFCA}, \texttt{FL+HC} and \texttt{CFL-Cosine} algorithms, which partition clients into multiple clusters, has (slightly) worse local fairness compared to \texttt{FedAvg}. Similarly, the \texttt{FedProx} algorithm imposes a regularization term that prevents the local updates from deviating too much from the global model, making it less fair compared to \texttt{FedAvg}. In comparison, the \texttt{MAML-FL} algorithm, which effectively sees the global dataset (when training the global model that is later fine-tuned by each client), has better local fairness compared to \texttt{FedAvg}, indicating that personalization can improve both local accuracy (as intended) and local fairness (as a side benefit). 

\vspace{-0.1in}
\subsection{Better-balanced groups: computational advantages of collaboration}
\vspace{-0.1in}

We next consider better-balanced data, to show advantages of collaborative and personalized training beyond the statistical benefits of (effectively) expanding training data. We again consider the ACSEmployment task, but now with ``sex'' as the protected attribute. Fig~\ref{fig:sex_emp_a} shows that data samples are more evenly distributed across groups and labels in this problem. Figure~\ref{fig:sex_emp_b} further confirms that clients exhibit similar sample fractions of label 1 individuals in male and female groups. 

\begin{figure}[ht]
\vspace{-0.1in}
	\centering
	\subfigure[Fraction of samples]{
		\includegraphics[width=0.3\textwidth]{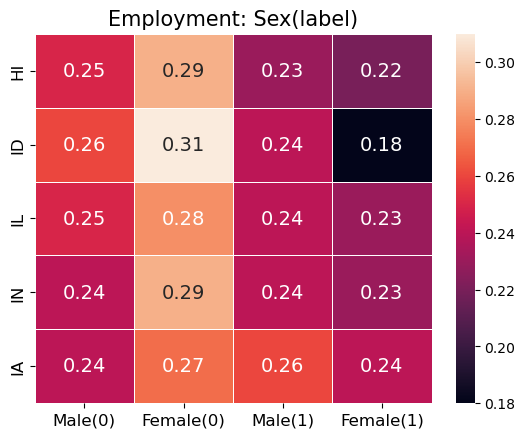} \label{fig:sex_emp_a}
	}
	\hspace{-0.1in}
	\subfigure[Normalized sample frequency]{		\includegraphics[width=0.31\textwidth]{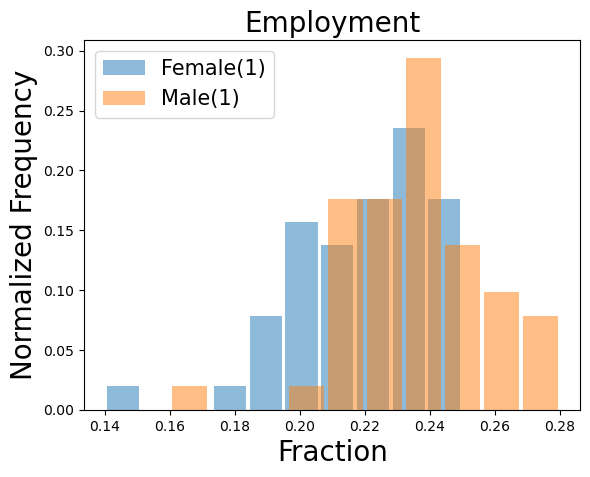} \label{fig:sex_emp_b}
	}
        \hspace{-0.05in}
	\subfigure[Local accuracy vs. fairness gap]{
\includegraphics[width=0.31\textwidth]{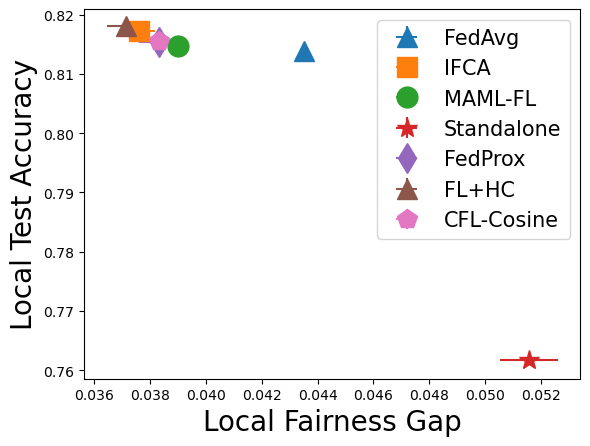} \label{fig:sex_emp_c}
	}
 \vspace{-0.1in}
	\caption{Experiments on the ACSEmployment task with better-balanced groups (sex).}
	\label{fig:sex_emp}
  \vspace{-0.1in}
\end{figure}

Fig~\ref{fig:sex_emp_c} shows that all collaborative training algorithms still have better local fairness compared to \texttt{Standalone} learning. Furthermore, we observe that all personalized learning algorithms (\texttt{IFCA}, \texttt{FL+HC}, \texttt{CFL-Cosine}, \texttt{MAML-FL}, and \texttt{FedProx}) improve both local accuracy and local fairness compared to \texttt{FedAvg}. This is due to the \emph{computational advantages} of (personalized) collaborative learning: for each client, due to similarity of the data for the male and female groups (as seen in Figure~\ref{fig:sex_emp_b}) the objective of maximizing local accuracy is aligned with reducing the local fairness gap. Therefore, collaboration improves local fairness, with personalization further enhancing the model's local accuracy and therefore its fairness.  

We also note that (local) accuracy and fairness may not necessarily be aligned. Our next experiment shows that personalization can still improve fairness in such tasks compared to non-personalized \texttt{FedAvg}, which (we interpret) is driven by a combination of statistical and computational benefits. Specifically, we conduct experiments on another task, ACSIncome, with ``sex'' as the protected attribute. Fig~\ref{fig:sex_inc_a} shows that for this task, the fraction of samples is comparable across groups for label 0 data, but differs for label 1 data. From Fig~\ref{fig:sex_inc_c}, we observe that this time, all collaborative training algorithms improve accuracy but have \emph{worse} local fairness compared to \texttt{Standalone} learning; this is because improving (local) accuracy is not aligned with fairness in this task. That said, we observe that the personalized FL algorithms slightly improve local fairness compared to \texttt{FedAvg}. We interpret this as the statistical advantage of (effectively) observing more label 1 data (as \texttt{FedAvg} does, too), combined with a computational advantage of not overfitting a global model to the majority label 0 data (unlike what \texttt{FedAvg} may be doing). 

\begin{figure}[ht]
\vspace{-0.1in}
	\centering
	\subfigure[Fraction of samples]{
		\includegraphics[width=0.3\textwidth]{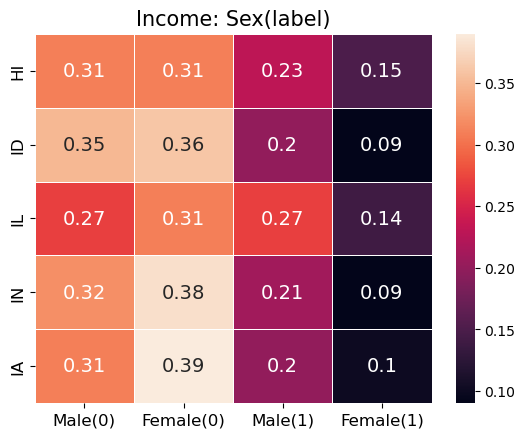} \label{fig:sex_inc_a}
	}
	\hspace{-0.1in}
	\subfigure[Normalized sample frequency]{		\includegraphics[width=0.31\textwidth]{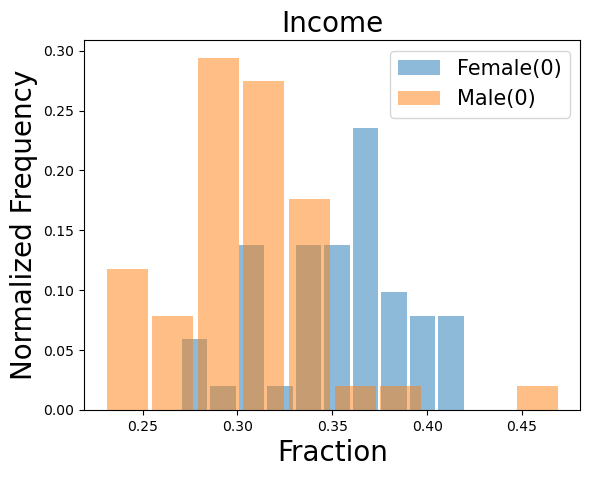} \label{fig:sex_inc_b}
	}
        \hspace{-0.05in}
	\subfigure[Local accuracy vs. fairness gap]{
\includegraphics[width=0.31\textwidth]{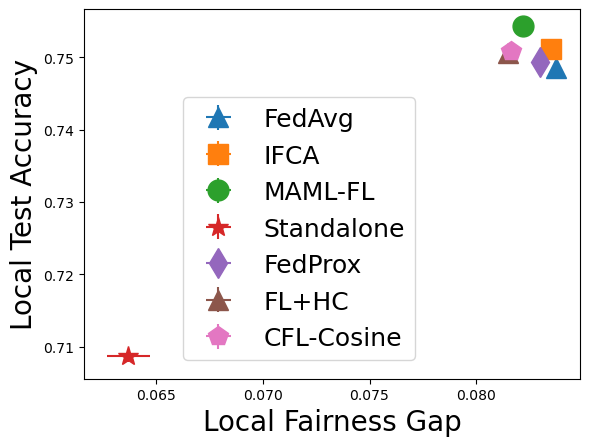} \label{fig:sex_inc_c}
	}
 \vspace{-0.1in}
	\caption{Experiments on the ACSIncome task with sex as the protected attribute.}
	\label{fig:sex_inc}
  \vspace{-0.1in}
\end{figure}

\vspace{-0.1in}
\subsection{Experiments on the \texttt{Adult} dataset}\label{sec:adult-numerical}
\vspace{-0.1in}

We also contrast these methods on the \texttt{Adult} dataset \citep{Dua:2019}. Among those 48842 samples, 41762 samples belong to the White group, while 7080 samples are from the Non-White groups. Given the nature of this data's heterogeneity, we randomly generate 5 clients each with an unbalanced number of samples based on race. Additionally, in Appendix~\ref{app:numerical_other_data}, we conduct experiments where samples are distributed with less heterogeneity across 5 clients, as done in other existing FL studies (e.g.~\citep{ezzeldin2023fairfed}). 

\begin{figure}[ht]
\vspace{-0.1in}
	\centering
	\subfigure[Number of samples]{
    \includegraphics[width=0.51\textwidth]{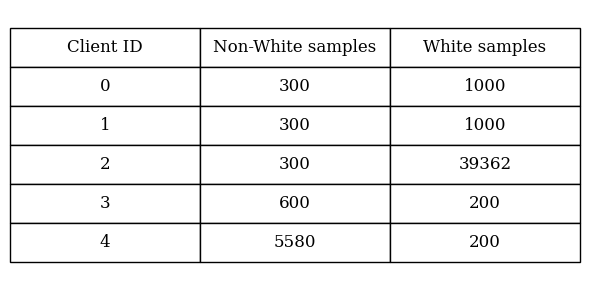} \label{table:adult_client_info}
	}
    \hspace{-0.05in}
	\subfigure[Local accuracy vs. fairness gap]{
\includegraphics[width=0.31\textwidth]{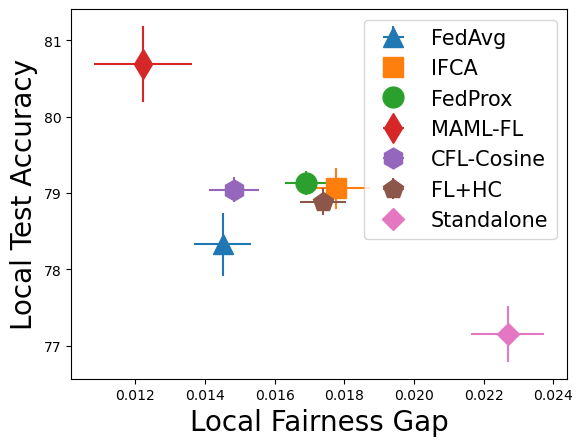} \label{fig:adult_acc_fair}
	}
 \vspace{-0.1in}
	\caption{Experiments on the \texttt{Adult} dataset with race as the protected attribute.}
	\label{fig:adult}
  \vspace{-0.1in}
\end{figure}

From Fig~\ref{fig:adult_acc_fair}, we can see that the results are consistent with our findings in Section~\ref{subsec: imbalanced}. Interestingly, we observe that the \texttt{IFCA} algorithm clusters clients by grouping those with more White samples into one cluster and those with more Non-White samples into another (i.e., \{0,1,2\}, \{3,4\}). In contrast, the \texttt{FL+HC} clusters clients by grouping those with more samples into one cluster and those with less samples into another (i.e., \{0,1,3,4\}, \{2\}). As a result, these two variants of clustering-based algorithms have different performance, having statistical advantages but for different reasons.

\subsection{Analytical support}\label{sec:analytical-support-overview}

To further support our numerical findings, we analytically show that personalized federated clustering algorithms, which group similar clients together to enhance local accuracy, can improve local fairness compared to a non-personalized global model under certain conditions. Specifically, we consider a setting with single-dimensional features that lead to optimal threshold-based classifiers and assume clients are grouped into two clusters $\{C_\alpha, C_\beta\}$ based on similarities in their local datasets. 

Let $\theta^*_{G}$ and $\theta^*_{i}, i\in\{\alpha, \beta\}$ denote the optimal decision threshold for the FedAvg algorithm and for clusters $C_\alpha$ and $C_\beta$, respectively. We define $\Delta_{f}(\theta^*_C)$ as the $f$-fairness performance, $f \in \{\texttt{EqOp}, \texttt{SP}, \texttt{EO}\}$ when using personalized federated clustering algorithms. Specifically, it quantifies the overall $f$-fairness achieved by applying the optimal decision threshold $\theta^*_{i}$ to its corresponding cluster $C_i$. Similarly, we define $\Delta^i_{f}(\theta^*_i)$ as the $f$-fairness measured within cluster $C_i$.

\begin{prop}[Improved overall \texttt{EqOp} through clustering] ~\label{prop3}
Let $\mathcal{G}^{y,c}_g(x)$ be unimodal distributions for $y\in\{0,1\}, g\in\{a,b\}, c\in\{C_\alpha, C_\beta\}$, with modes $m^{y,c}_g$ satisfying $m^{y,c}_{b}\leq m^{y,c}_a, \forall g, c$, and $\alpha^{1,C_\alpha}_g \geq \alpha^{0, C_\alpha}_g, \forall g$. Suppose  $\theta^{*}_{\alpha} < \theta^{*}_{\beta}$. Then, there exists a cluster size $\hat{p}$ such that for $p\geq \hat{p}$, we have $\Delta_{\texttt{EqOp}}(\theta^*_C)\leq\Delta_{\texttt{EqOp}}(\theta^*_G)$, where 
$\hat{p}$ is the solution to $\hat{p}=\min\{1, | \tfrac{\int_{\theta^*_G}^{\theta^*_\beta} \mathcal{G}^{1,\beta}_a(x)-\mathcal{G}^{1,\beta}_b(x)dx}{\int_{\theta^*_\alpha}^{\theta^*_G} \mathcal{G}^{1,\alpha}_a(x)-\mathcal{G}^{1,\alpha}_b(x)dx}|\}$ and $\theta^*_{G}$ is obtained as: $\theta^*_{G} =\arg\min_{\theta} \hat{p}*\sum_{j\in \mathcal{C}_\alpha} \mathcal{L}_j(\theta) + (1-\hat{p})*\sum_{j\in \mathcal{C}_\beta} \mathcal{L}_j(\theta)$.
\end{prop}

The proof is provided in Appendix~\ref{app_proof:prop_eqop}. Intuitively, clients in $C_\alpha$ benefit from their personalized model because increasing the decision threshold (i.e., shifting from $\theta^*_\alpha$ to $\theta^*_G$) reduces the true positive rate of the disadvantaged group $b$ faster than that of the advantaged group $a$, increasing the fairness gap in $C_\alpha$. For clients in $C_\beta$, the opposite effect occurs. However, under the given conditions, the fairness improvement in $C_\beta$ is insufficient to compensate for the fairness degradation in $C_\alpha$, resulting in a unfairer outcome when using $\theta^*_G$. 

\begin{prop}[Improved \texttt{SP} for $C_\alpha$ through clustering] ~\label{prop1}
Let $\mathcal{G}^{y,c}_g(x)$ be Gaussian distributions with equal variance for $y\in\{0,1\}, g\in\{a,b\}, c=C_\alpha$, with means $\mu^{y,c}_g$ satisfying $\mu^{0,c}_g \leq \mu^{1,c}_g, \forall g$. Suppose $\theta^{*}_{\alpha} < \theta^{*}_{\beta}$. If $\alpha^{1,c}_g \geq \alpha^{0,c}_g, \forall g$ and $\alpha^{0,c}_a\exp(\frac{(\bar{\theta} - \mu^{0,c}_a)^2}{-2\sigma^2}) (\bar{\theta} - \mu^{0,c}_a) - \alpha^{1,c}_b\exp(\frac{(\bar{\theta} - \mu^{1,c}_b)^2}{-2\sigma^2}) (\bar{\theta} - \mu^{1,c}_b) \geq \alpha^{0,c}_b\exp(\frac{(\bar{\theta} - \mu^{0,c}_b)^2}{-2\sigma^2}) (\bar{\theta} - \mu^{0,c}_b) - \alpha^{1,c}_a\exp(\frac{(\bar{\theta} - \mu^{1,c}_a)^2}{-2\sigma^2}) (\bar{\theta} - \mu^{1,c}_a)$, where $\bar{\theta}:=\tfrac{\mu^{1,c}_a + \mu^{0,c}_b+\mu^{1,c}_b + \mu^{0,c}_a}{4}$, then there exists a $\hat{p}$ such that for $p\geq \hat{p}$, $\Delta^\alpha_{\texttt{SP}}(\theta^*_\alpha)\leq\Delta^\alpha_{\texttt{SP}}(\theta^*_G)$.
\end{prop}

\begin{wrapfigure}[9]{r}{0.45\textwidth}\vspace{-0.3in}
	\centering
	\subfigure[\texttt{SP}]{
 \includegraphics[width=0.2\textwidth]{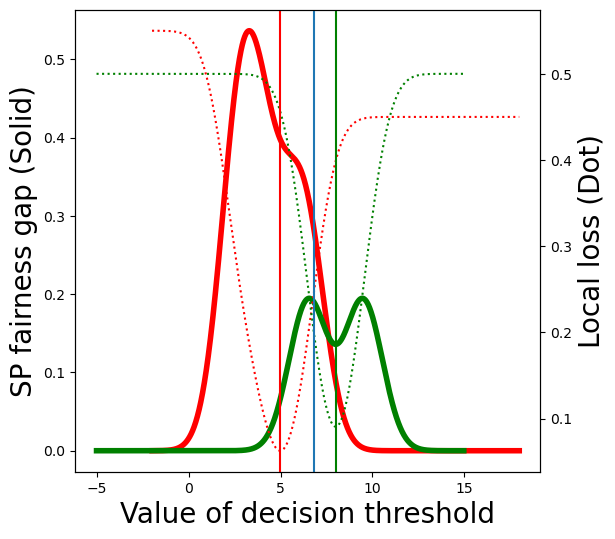}
 \label{Fig:SP_fairness_plot}
	}
 \subfigure[\texttt{EqOp}]{
 \includegraphics[width=0.2\textwidth]{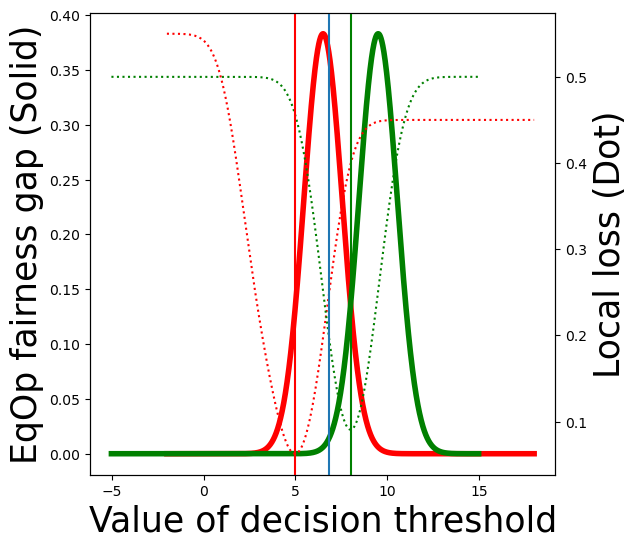}
 \label{Fig:EqOp_fairness_plot}
	}
        \vspace{-0.15in}
	\caption{Fairness gap vs $\theta$.}
	\label{fig:SP-is-hard-in-paper}
\end{wrapfigure}
The proof is provided in Appendix~\ref{app_proof:prop_sp}, and we can reach similar conditions for the cluster $C_\beta$. Intuitively, when there are more label 1 samples in both groups, the global model $\theta^*_G$ will pull the $C_\alpha$ cluster model $\theta^*_\alpha$ up to account for the label imbalance, resulting in a deterioration in both fairness and accuracy for clients in this cluster. Note that Proposition~\ref{prop1} considers the \texttt{SP} fairness, which is impacted by both the group $a$ vs. $b$ feature distributions \emph{as well as} the label rates, rendering it more stringent than \texttt{EqOp} fairness of Proposition~\ref{prop3}. Figure~\ref{fig:SP-is-hard-in-paper} illustrates this by plotting the fairness gap vs. the decision threshold $\theta$ for \texttt{SP} vs. \texttt{EqOp}, showing that \texttt{SP} exhibits less structured changes as the decision threshold moves (e.g., due to the use of a global model). 
\vspace{-0.1in}
\section{Conclusion}\label{sec:conclusion}
\vspace{-0.1in}

We proposed new fairness-aware federated clustering algorithms, \texttt{Fair-FCA} and \texttt{Fair-FL+HC}, which take both fairness and accuracy into account when clustering clients. Our methods effectively span the space between personalized FL and locally fair FL. We find that they can lead to improved local fairness (matching or exceeding existing locally fair FL methods, without any explicit fairness interventions) while allowing for a tunable trade-off between accuracy and fairness at the client level. We have also both numerically and analytically shown that there can be (unintended) fairness benefits to personalization in FL, which our clustering-based fair FL algorithms are exploiting whenever possible. Identifying methods to integrate fairness considerations into other (non-clustering based) personalized FL algorithms, and extending our analytical findings (both for clustered FL algorithms, and to other classes of personalized FL methods), are main directions of future work.

\clearpage

\ack{This work is supported by Cisco Research, and by the National Science Foundation (NSF) program on Fairness in AI in collaboration with Amazon under Award No. IIS-2040800. Any opinions, findings, and conclusions or recommendations expressed in this material are those of the authors and do not necessarily reflect the views of the NSF, Amazon, or Cisco.}
\bibliographystyle{unsrtnat}
\bibliography{reference}

\begin{thebibliography}{67}
\providecommand{\natexlab}[1]{#1}
\providecommand{\url}[1]{\texttt{#1}}
\expandafter\ifx\csname urlstyle\endcsname\relax
  \providecommand{\doi}[1]{doi: #1}\else
  \providecommand{\doi}{doi: \begingroup \urlstyle{rm}\Url}\fi

\bibitem[Kairouz et~al.(2021)Kairouz, McMahan, Avent, Bellet, Bennis, Bhagoji, Bonawitz, Charles, Cormode, Cummings, et~al.]{kairouz2021advances}
Peter Kairouz, H~Brendan McMahan, Brendan Avent, Aur{\'e}lien Bellet, Mehdi Bennis, Arjun~Nitin Bhagoji, Kallista Bonawitz, Zachary Charles, Graham Cormode, Rachel Cummings, et~al.
\newblock Advances and open problems in federated learning.
\newblock \emph{Foundations and Trends{\textregistered} in Machine Learning}, 14\penalty0 (1--2):\penalty0 1--210, 2021.

\bibitem[Li et~al.(2019{\natexlab{a}})Li, Huang, Yang, Wang, and Zhang]{li2019convergence}
Xiang Li, Kaixuan Huang, Wenhao Yang, Shusen Wang, and Zhihua Zhang.
\newblock On the convergence of fedavg on non-iid data.
\newblock \emph{arXiv preprint arXiv:1907.02189}, 2019{\natexlab{a}}.

\bibitem[Zhao et~al.(2018)Zhao, Li, Lai, Suda, Civin, and Chandra]{zhao2018federated}
Yue Zhao, Meng Li, Liangzhen Lai, Naveen Suda, Damon Civin, and Vikas Chandra.
\newblock Federated learning with non-iid data.
\newblock \emph{arXiv preprint arXiv:1806.00582}, 2018.

\bibitem[Tan et~al.(2022)Tan, Yu, Cui, and Yang]{tan2022towards}
Alysa~Ziying Tan, Han Yu, Lizhen Cui, and Qiang Yang.
\newblock Towards personalized federated learning.
\newblock \emph{IEEE Transactions on Neural Networks and Learning Systems}, 2022.

\bibitem[Karimireddy et~al.(2020)Karimireddy, Kale, Mohri, Reddi, Stich, and Suresh]{karimireddy2020scaffold}
Sai~Praneeth Karimireddy, Satyen Kale, Mehryar Mohri, Sashank Reddi, Sebastian Stich, and Ananda~Theertha Suresh.
\newblock Scaffold: Stochastic controlled averaging for federated learning.
\newblock In \emph{International conference on machine learning}, pages 5132--5143. PMLR, 2020.

\bibitem[Li et~al.(2020)Li, Sahu, Zaheer, Sanjabi, Talwalkar, and Smith]{li2020federated}
Tian Li, Anit~Kumar Sahu, Manzil Zaheer, Maziar Sanjabi, Ameet Talwalkar, and Virginia Smith.
\newblock Federated optimization in heterogeneous networks.
\newblock \emph{Proceedings of Machine learning and systems}, 2:\penalty0 429--450, 2020.

\bibitem[Ghosh et~al.(2020)Ghosh, Chung, Yin, and Ramchandran]{ghosh2020efficient}
Avishek Ghosh, Jichan Chung, Dong Yin, and Kannan Ramchandran.
\newblock An efficient framework for clustered federated learning.
\newblock \emph{Advances in Neural Information Processing Systems}, 33:\penalty0 19586--19597, 2020.

\bibitem[Briggs et~al.(2020)Briggs, Fan, and Andras]{briggs2020federated}
Christopher Briggs, Zhong Fan, and Peter Andras.
\newblock Federated learning with hierarchical clustering of local updates to improve training on non-iid data.
\newblock In \emph{2020 international joint conference on neural networks (IJCNN)}, pages 1--9. IEEE, 2020.

\bibitem[Sattler et~al.(2020)Sattler, M{\"u}ller, and Samek]{sattler2020clustered}
Felix Sattler, Klaus-Robert M{\"u}ller, and Wojciech Samek.
\newblock Clustered federated learning: Model-agnostic distributed multitask optimization under privacy constraints.
\newblock \emph{IEEE transactions on neural networks and learning systems}, 32\penalty0 (8):\penalty0 3710--3722, 2020.

\bibitem[Fallah et~al.(2020)Fallah, Mokhtari, and Ozdaglar]{fallah2020personalized}
Alireza Fallah, Aryan Mokhtari, and Asuman Ozdaglar.
\newblock Personalized federated learning: A meta-learning approach.
\newblock \emph{arXiv preprint arXiv:2002.07948}, 2020.

\bibitem[Mansour et~al.(2020)Mansour, Mohri, Ro, and Suresh]{mansour2020three}
Yishay Mansour, Mehryar Mohri, Jae Ro, and Ananda~Theertha Suresh.
\newblock Three approaches for personalization with applications to federated learning.
\newblock \emph{arXiv preprint arXiv:2002.10619}, 2020.

\bibitem[Barocas et~al.(2019)Barocas, Hardt, and Narayanan]{barocas-hardt-narayanan}
Solon Barocas, Moritz Hardt, and Arvind Narayanan.
\newblock \emph{Fairness and Machine Learning: Limitations and Opportunities}.
\newblock fairmlbook.org, 2019.
\newblock \url{http://www.fairmlbook.org}.

\bibitem[Shi et~al.(2023)Shi, Yu, and Leung]{shi2023towards}
Yuxin Shi, Han Yu, and Cyril Leung.
\newblock Towards fairness-aware federated learning.
\newblock \emph{IEEE Transactions on Neural Networks and Learning Systems}, 2023.

\bibitem[Salazar et~al.(2024)Salazar, Ara{\'u}jo, Cano, and Abreu]{salazar2024survey}
Teresa Salazar, Helder Ara{\'u}jo, Alberto Cano, and Pedro~Henriques Abreu.
\newblock A survey on group fairness in federated learning: Challenges, taxonomy of solutions and directions for future research.
\newblock \emph{arXiv preprint arXiv:2410.03855}, 2024.

\bibitem[Swift(2002)]{swift2002guidance}
Elaine~K Swift.
\newblock Guidance for the national healthcare disparities report.
\newblock 2002.

\bibitem[Hamman and Dutta(2023)]{hamman2023demystifying}
Faisal Hamman and Sanghamitra Dutta.
\newblock Demystifying local and global fairness trade-offs in federated learning using information theory.
\newblock In \emph{Federated Learning and Analytics in Practice: Algorithms, Systems, Applications, and Opportunities}, 2023.

\bibitem[Meerza et~al.(2024)Meerza, Liu, Zhang, and Liu]{meerza2024glocalfair}
Syed Irfan~Ali Meerza, Luyang Liu, Jiaxin Zhang, and Jian Liu.
\newblock Glocalfair: Jointly improving global and local group fairness in federated learning.
\newblock \emph{arXiv preprint arXiv:2401.03562}, 2024.

\bibitem[Zhang et~al.(2025)Zhang, Li, Wu, Zhao, and Palaiahnakote]{zhang2025sffl}
Jiale Zhang, Ye~Li, Di~Wu, Yanchao Zhao, and Shivakumara Palaiahnakote.
\newblock Sffl: Self-aware fairness federated learning framework for heterogeneous data distributions.
\newblock \emph{Expert Systems with Applications}, 269:\penalty0 126418, 2025.

\bibitem[Makhija et~al.(2024)Makhija, Han, Ghosh, and Kim]{makhija2024achieving}
Disha Makhija, Xing Han, Joydeep Ghosh, and Yejin Kim.
\newblock Achieving fairness across local and global models in federated learning.
\newblock \emph{arXiv preprint arXiv:2406.17102}, 2024.

\bibitem[Zhou and Goel(2025)]{zhou2025post}
Yi~Zhou and Naman Goel.
\newblock A post-processing-based fair federated learning framework.
\newblock \emph{arXiv preprint arXiv:2501.15318}, 2025.

\bibitem[Dua and Graff(2017)]{Dua:2019}
Dheeru Dua and Casey Graff.
\newblock {UCI} machine learning repository, 2017.
\newblock URL \url{http://archive.ics.uci.edu/ml}.

\bibitem[Ding et~al.(2021)Ding, Hardt, Miller, and Schmidt]{ding2021retiring}
Frances Ding, Moritz Hardt, John Miller, and Ludwig Schmidt.
\newblock Retiring adult: New datasets for fair machine learning.
\newblock \emph{Advances in neural information processing systems}, 34:\penalty0 6478--6490, 2021.

\bibitem[Li et~al.(2019{\natexlab{b}})Li, Sanjabi, Beirami, and Smith]{li2019fair}
Tian Li, Maziar Sanjabi, Ahmad Beirami, and Virginia Smith.
\newblock Fair resource allocation in federated learning.
\newblock \emph{arXiv preprint arXiv:1905.10497}, 2019{\natexlab{b}}.

\bibitem[Li et~al.(2021{\natexlab{a}})Li, Hu, Beirami, and Smith]{li2021ditto}
Tian Li, Shengyuan Hu, Ahmad Beirami, and Virginia Smith.
\newblock Ditto: Fair and robust federated learning through personalization.
\newblock In \emph{International Conference on Machine Learning}, pages 6357--6368. PMLR, 2021{\natexlab{a}}.

\bibitem[Zhang et~al.(2021)Zhang, Kuang, Liu, Chen, Wu, Wu, Lu, Shao, and Xiao]{zhang2021unified}
Fengda Zhang, Kun Kuang, Yuxuan Liu, Long Chen, Chao Wu, Fei Wu, Jiaxun Lu, Yunfeng Shao, and Jun Xiao.
\newblock Unified group fairness on federated learning.
\newblock \emph{arXiv preprint arXiv:2111.04986}, 2021.

\bibitem[Wang et~al.(2021)Wang, Fan, Qi, Wen, Wang, and Yu]{wang2021federated}
Zheng Wang, Xiaoliang Fan, Jianzhong Qi, Chenglu Wen, Cheng Wang, and Rongshan Yu.
\newblock Federated learning with fair averaging.
\newblock \emph{arXiv preprint arXiv:2104.14937}, 2021.

\bibitem[Abay et~al.(2020)Abay, Zhou, Baracaldo, Rajamoni, Chuba, and Ludwig]{abay2020mitigating}
Annie Abay, Yi~Zhou, Nathalie Baracaldo, Shashank Rajamoni, Ebube Chuba, and Heiko Ludwig.
\newblock Mitigating bias in federated learning.
\newblock \emph{arXiv preprint arXiv:2012.02447}, 2020.

\bibitem[G{\'a}lvez et~al.(2021)G{\'a}lvez, Granqvist, van Dalen, and Seigel]{galvez2021enforcing}
Borja~Rodr{\'\i}guez G{\'a}lvez, Filip Granqvist, Rogier van Dalen, and Matt Seigel.
\newblock Enforcing fairness in private federated learning via the modified method of differential multipliers.
\newblock In \emph{NeurIPS 2021 Workshop Privacy in Machine Learning}, 2021.

\bibitem[Wang et~al.(2023)Wang, Payani, Lee, and Kompella]{wang2023mitigating}
Ganghua Wang, Ali Payani, Myungjin Lee, and Ramana Kompella.
\newblock Mitigating group bias in federated learning: Beyond local fairness.
\newblock \emph{arXiv preprint arXiv:2305.09931}, 2023.

\bibitem[Zeng et~al.(2021)Zeng, Chen, and Lee]{zeng2021improving}
Yuchen Zeng, Hongxu Chen, and Kangwook Lee.
\newblock Improving fairness via federated learning.
\newblock \emph{arXiv preprint arXiv:2110.15545}, 2021.

\bibitem[Ezzeldin et~al.(2023)Ezzeldin, Yan, He, Ferrara, and Avestimehr]{ezzeldin2023fairfed}
Yahya~H Ezzeldin, Shen Yan, Chaoyang He, Emilio Ferrara, and A~Salman Avestimehr.
\newblock Fairfed: Enabling group fairness in federated learning.
\newblock In \emph{Proceedings of the AAAI conference on artificial intelligence}, volume~37, pages 7494--7502, 2023.

\bibitem[Liu et~al.(2025)Liu, Sarkani, and Mazzuchi]{liu2025fairness}
Ziyan Liu, Shahram Sarkani, and Thomas Mazzuchi.
\newblock Fairness-optimized dynamic aggregation (foda): A novel approach to equitable federated learning in heterogeneous environments.
\newblock \emph{Available at SSRN 5145003}, 2025.

\bibitem[Wang et~al.(2024)Wang, Zhang, Cai, Gong, Choo, and Guo]{wang2024analyzing}
Tongnian Wang, Kai Zhang, Jiannan Cai, Yanmin Gong, Kim-Kwang~Raymond Choo, and Yuanxiong Guo.
\newblock Analyzing the impact of personalization on fairness in federated learning for healthcare.
\newblock \emph{Journal of Healthcare Informatics Research}, 8\penalty0 (2):\penalty0 181--205, 2024.

\bibitem[Nafea et~al.(2022)Nafea, Shin, and Yener]{nafea2022proportional}
Mohamed Nafea, Eugine Shin, and Aylin Yener.
\newblock Proportional fair clustered federated learning.
\newblock In \emph{2022 IEEE International Symposium on Information Theory (ISIT)}. IEEE, 2022.

\bibitem[Kyllo and Mashhadi(2023)]{kyllo2023inflorescence}
Alex Kyllo and Afra Mashhadi.
\newblock Inflorescence: A framework for evaluating fairness with clustered federated learning.
\newblock In \emph{Adjunct Proceedings of the 2023 ACM International Joint Conference on Pervasive and Ubiquitous Computing \& the 2023 ACM International Symposium on Wearable Computing}, pages 374--380, 2023.

\bibitem[Li et~al.(2021{\natexlab{b}})Li, Li, and Varshney]{li2021federated}
Chengxi Li, Gang Li, and Pramod~K Varshney.
\newblock Federated learning with soft clustering.
\newblock \emph{IEEE Internet of Things Journal}, 9\penalty0 (10):\penalty0 7773--7782, 2021{\natexlab{b}}.

\bibitem[Chung et~al.(2022)Chung, Lee, and Ramchandran]{chung2022federated}
Jichan Chung, Kangwook Lee, and Kannan Ramchandran.
\newblock Federated unsupervised clustering with generative models.
\newblock In \emph{AAAI 2022 international workshop on trustable, verifiable and auditable federated learning}, volume~4, 2022.

\bibitem[Huang et~al.(2023)Huang, Shi, Feng, Niu, Cheng, Huang, and Liu]{huang2023active}
Honglan Huang, Wei Shi, Yanghe Feng, Chaoyue Niu, Guangquan Cheng, Jincai Huang, and Zhong Liu.
\newblock Active client selection for clustered federated learning.
\newblock \emph{IEEE Transactions on Neural Networks and Learning Systems}, 2023.

\bibitem[Ma et~al.(2024)Ma, Zhou, Long, Jiang, and Zhang]{ma2024structured}
Jie Ma, Tianyi Zhou, Guodong Long, Jing Jiang, and Chengqi Zhang.
\newblock Structured federated learning through clustered additive modeling.
\newblock \emph{Advances in Neural Information Processing Systems}, 36, 2024.

\bibitem[Jothimurugesan et~al.(2023)Jothimurugesan, Hsieh, Wang, Joshi, and Gibbons]{jothimurugesan2023federated}
Ellango Jothimurugesan, Kevin Hsieh, Jianyu Wang, Gauri Joshi, and Phillip~B Gibbons.
\newblock Federated learning under distributed concept drift.
\newblock In \emph{International Conference on Artificial Intelligence and Statistics}, pages 5834--5853. PMLR, 2023.

\bibitem[Luo et~al.(2023)Luo, Fu, Luo, Liu, Deng, and Wang]{luo2023privacy}
Songwei Luo, Shaojing Fu, Yuchuan Luo, Lin Liu, Yanxiang Deng, and Shixiong Wang.
\newblock Privacy-preserving federated learning with hierarchical clustering to improve training on non-iid data.
\newblock In \emph{International Conference on Network and System Security}, pages 195--216. Springer, 2023.

\bibitem[Li et~al.(2023)Li, Wang, and An]{li2023hierarchical}
Youpeng Li, Xuyu Wang, and Lingling An.
\newblock Hierarchical clustering-based personalized federated learning for robust and fair human activity recognition.
\newblock \emph{Proceedings of the ACM on Interactive, Mobile, Wearable and Ubiquitous Technologies}, 7\penalty0 (1):\penalty0 1--38, 2023.

\bibitem[Sun et~al.(2024)Sun, Kountouris, and Zhang]{sun2024collaborate}
Yuchang Sun, Marios Kountouris, and Jun Zhang.
\newblock How to collaborate: Towards maximizing the generalization performance in cross-silo federated learning.
\newblock \emph{arXiv preprint arXiv:2401.13236}, 2024.

\bibitem[Dwork et~al.(2012)Dwork, Hardt, Pitassi, Reingold, and Zemel]{dwork2012fairness}
Cynthia Dwork, Moritz Hardt, Toniann Pitassi, Omer Reingold, and Richard Zemel.
\newblock Fairness through awareness.
\newblock In \emph{Proceedings of the 3rd innovations in theoretical computer science conference}, pages 214--226, 2012.

\bibitem[Hardt et~al.(2016)Hardt, Price, and Srebro]{hardt2016equality}
Moritz Hardt, Eric Price, and Nati Srebro.
\newblock Equality of opportunity in supervised learning.
\newblock \emph{Advances in neural information processing systems}, 29, 2016.

\bibitem[McMahan et~al.(2017)McMahan, Moore, Ramage, Hampson, and y~Arcas]{mcmahan2017communication}
Brendan McMahan, Eider Moore, Daniel Ramage, Seth Hampson, and Blaise~Aguera y~Arcas.
\newblock Communication-efficient learning of deep networks from decentralized data.
\newblock In \emph{Artificial intelligence and statistics}, pages 1273--1282. PMLR, 2017.

\bibitem[Nardi et~al.(2022)Nardi, Valerio, and Passarella]{nardi2022anomaly}
Mirko Nardi, Lorenzo Valerio, and Andrea Passarella.
\newblock Anomaly detection through unsupervised federated learning.
\newblock \emph{arXiv preprint arXiv:2209.04184}, 2022.

\bibitem[Zheng et~al.(2022)Zheng, Naghizadeh, and Yener]{zheng2022diple}
Xue Zheng, Parinaz Naghizadeh, and Aylin Yener.
\newblock Diple: Learning directed collaboration graphs for peer-to-peer personalized learning.
\newblock In \emph{2022 IEEE Information Theory Workshop (ITW)}, pages 446--451. IEEE, 2022.

\bibitem[Arivazhagan et~al.(2019)Arivazhagan, Aggarwal, Singh, and Choudhary]{arivazhagan2019federated}
Manoj~Ghuhan Arivazhagan, Vinay Aggarwal, Aaditya~Kumar Singh, and Sunav Choudhary.
\newblock Federated learning with personalization layers.
\newblock \emph{arXiv preprint arXiv:1912.00818}, 2019.

\bibitem[Jiang and Lin(2022)]{jiang2022test}
Liangze Jiang and Tao Lin.
\newblock Test-time robust personalization for federated learning.
\newblock \emph{arXiv preprint arXiv:2205.10920}, 2022.

\bibitem[Hanzely and Richt{\'a}rik(2020)]{hanzely2020federated}
Filip Hanzely and Peter Richt{\'a}rik.
\newblock Federated learning of a mixture of global and local models.
\newblock \emph{arXiv preprint arXiv:2002.05516}, 2020.

\bibitem[Sahu et~al.(2018)Sahu, Li, Sanjabi, Zaheer, Talwalkar, and Smith]{sahu2018convergence}
Anit~Kumar Sahu, Tian Li, Maziar Sanjabi, Manzil Zaheer, Ameet Talwalkar, and Virginia Smith.
\newblock On the convergence of federated optimization in heterogeneous networks.
\newblock \emph{arXiv preprint arXiv:1812.06127}, 3:\penalty0 3, 2018.

\bibitem[T~Dinh et~al.(2020)T~Dinh, Tran, and Nguyen]{t2020personalized}
Canh T~Dinh, Nguyen Tran, and Josh Nguyen.
\newblock Personalized federated learning with moreau envelopes.
\newblock \emph{Advances in Neural Information Processing Systems}, 33:\penalty0 21394--21405, 2020.

\bibitem[Huang et~al.(2021)Huang, Chu, Zhou, Wang, Liu, Pei, and Zhang]{huang2021personalized}
Yutao Huang, Lingyang Chu, Zirui Zhou, Lanjun Wang, Jiangchuan Liu, Jian Pei, and Yong Zhang.
\newblock Personalized cross-silo federated learning on non-iid data.
\newblock In \emph{AAAI}, pages 7865--7873, 2021.

\bibitem[Deng et~al.(2020)Deng, Kamani, and Mahdavi]{deng2020adaptive}
Yuyang Deng, Mohammad~Mahdi Kamani, and Mehrdad Mahdavi.
\newblock Adaptive personalized federated learning.
\newblock \emph{arXiv preprint arXiv:2003.13461}, 2020.

\bibitem[Zec et~al.(2020)Zec, Martinsson, Mogren, S{\"u}tfeld, and Gillblad]{zec2020federated}
Edvin~Listo Zec, John Martinsson, Olof Mogren, Leon~Ren{\'e} S{\"u}tfeld, and Daniel Gillblad.
\newblock Federated learning using mixture of experts.
\newblock 2020.

\bibitem[Peterson et~al.(2019)Peterson, Kanani, and Marathe]{peterson2019private}
Daniel Peterson, Pallika Kanani, and Virendra~J Marathe.
\newblock Private federated learning with domain adaptation.
\newblock \emph{arXiv preprint arXiv:1912.06733}, 2019.

\bibitem[Rafi et~al.(2024)Rafi, Noor, Hussain, and Chae]{rafi2024fairness}
Taki~Hasan Rafi, Faiza~Anan Noor, Tahmid Hussain, and Dong-Kyu Chae.
\newblock Fairness and privacy preserving in federated learning: A survey.
\newblock \emph{Information Fusion}, 105:\penalty0 102198, 2024.

\bibitem[Mohri et~al.(2019)Mohri, Sivek, and Suresh]{mohri2019agnostic}
Mehryar Mohri, Gary Sivek, and Ananda~Theertha Suresh.
\newblock Agnostic federated learning.
\newblock In \emph{International Conference on Machine Learning}, pages 4615--4625. PMLR, 2019.

\bibitem[Cui et~al.(2021)Cui, Pan, Liang, Zhang, and Wang]{cui2021addressing}
Sen Cui, Weishen Pan, Jian Liang, Changshui Zhang, and Fei Wang.
\newblock Addressing algorithmic disparity and performance inconsistency in federated learning.
\newblock \emph{Advances in Neural Information Processing Systems}, 34:\penalty0 26091--26102, 2021.

\bibitem[Papadaki et~al.(2021)Papadaki, Martinez, Bertran, Sapiro, and Rodrigues]{papadaki2021federating}
Afroditi Papadaki, Natalia Martinez, Martin Bertran, Guillermo Sapiro, and Miguel Rodrigues.
\newblock Federating for learning group fair models.
\newblock \emph{arXiv preprint arXiv:2110.01999}, 2021.

\bibitem[Huang et~al.(2020)Huang, Li, Wang, Du, and Zhang]{huang2020fairness}
Wei Huang, Tianrui Li, Dexian Wang, Shengdong Du, and Junbo Zhang.
\newblock Fairness and accuracy in federated learning.
\newblock \emph{arXiv preprint arXiv:2012.10069}, 2020.

\bibitem[Chu et~al.(2021)Chu, Wang, Dong, Pei, Zhou, and Zhang]{chu2021fedfair}
Lingyang Chu, Lanjun Wang, Yanjie Dong, Jian Pei, Zirui Zhou, and Yong Zhang.
\newblock Fedfair: Training fair models in cross-silo federated learning.
\newblock \emph{arXiv preprint arXiv:2109.05662}, 2021.

\bibitem[Zhang et~al.(2022)Zhang, Malekmohammadi, Chen, and Yu]{zhang2022proportional}
Guojun Zhang, Saber Malekmohammadi, Xi~Chen, and Yaoliang Yu.
\newblock Proportional fairness in federated learning.
\newblock \emph{arXiv preprint arXiv:2202.01666}, 2022.

\bibitem[Lyu et~al.(2020)Lyu, Xu, Wang, and Yu]{lyu2020collaborative}
Lingjuan Lyu, Xinyi Xu, Qian Wang, and Han Yu.
\newblock Collaborative fairness in federated learning.
\newblock \emph{Federated Learning: Privacy and Incentive}, pages 189--204, 2020.

\bibitem[Corbett-Davies et~al.(2017)Corbett-Davies, Pierson, Feller, Goel, and Huq]{corbett2017algorithmic}
Sam Corbett-Davies, Emma Pierson, Avi Feller, Sharad Goel, and Aziz Huq.
\newblock Algorithmic decision making and the cost of fairness.
\newblock In \emph{Proceedings of the 23rd acm sigkdd international conference on knowledge discovery and data mining}, pages 797--806, 2017.

\bibitem[Raab and Liu(2021)]{raab2021unintended}
Reilly Raab and Yang Liu.
\newblock Unintended selection: Persistent qualification rate disparities and interventions.
\newblock \emph{Advances in Neural Information Processing Systems}, 34:\penalty0 26053--26065, 2021.

\end{thebibliography}


\clearpage
\appendix
\section{Limitations and extensions}\label{app:limit}
\textbf{Personalization techniques and theoretical support.} Our numerical and analytical results primarily focus on the \texttt{IFCA} framework. Although we also evaluate the performance of the \texttt{MAML-FL} algorithm in our experiments, additional experimental results using other types of local fine-tuning methods could be valuable to explore. Additionally, extending our analytical findings (both for clustered FL algorithms, and to other classes of personalized FL methods), and identifying methods to integrate fairness considerations into other (non-clustering based) personalized FL algorithms, are main directions of future work. 

\textbf{Combining our algorithm with existing FairFL methods.} Our proposed algorithms focus on modifying the cluster selection process by incorporating fairness considerations. Once clusters are formed, standard weighted aggregation is used to compute the cluster models. These algorithms can be further combined with existing in-processing FairFL methods. For example, fairness constraints can be introduced into the local optimization problems, and fairness-aware aggregation techniques can be applied when constructing the cluster models. Exploring such combinations represents a promising direction for future work.

\textbf{FL setting.} Our theoretical support for the idea that personalization can improve fairness assumes identical clients within two clusters. Extending this framework to include multiple clusters and account for client heterogeneity within each cluster is a straightforward extension. Additionally, allowing for client drop-out, a scenario not considered in the current FL setting, presents another avenue for exploration.

\section{Additional related works}\label{app:related}
Our work is at the intersection of two literatures: personalized FL, and achieving fairness in FL.

\textbf{Personalized FL:} Existing literature can be categorized based on how personalization is achieved. We investigate personalization through \emph{clustering},  \emph{local fine-tuning}, \emph{model regularization}, \emph{model interpolation} and \emph{data interpolation}.\\
\textbf{\textit{Clustering:}} \citet{mansour2020three} use a hypothesis-based clustering approach by minimizing the sum of loss over all clusters. \citet{sattler2020clustered} use the idea that cosine similarity between weight updates of different clients is highly indicative of the similarity of data distribution. \citet{nardi2022anomaly} use a decentralized learning idea by exchanging local models with other clients to find the neighbor/group which has a high accuracy even using other clients' models. \citet{zheng2022diple} learn a weighted and directed graph that indicates the relevance between clients. \citet{ghosh2020efficient} use a distributed learning idea by broadcasting all clients' models to others, and collecting back the cluster identity from clients who can identify good performance when using others' models.\\
\textbf{\textit{Local fine-tuning:}} \citet{fallah2020personalized} propose using a Model Agnostic Meta Learning (MAML) framework, where clients run additional local gradient steps to personalize the global model. \citet{arivazhagan2019federated,jiang2022test} propose using deep learning models with a combination of feature extraction layers (base) and global/local head (personalization). \citet{jiang2022test}, inspired by \citet{arivazhagan2019federated}, further consider robustifying against distribution shifts.\\
\textbf{\textit{Model regularization:}} \citet{hanzely2020federated,sahu2018convergence,li2020federated, li2021ditto} add a regularization term with mixing parameters to penalize the distance between the local and global models. In particular, \citet{sahu2018convergence} has a pre-set regularization parameter and allows for system heterogeneity. \citet{li2021ditto} consider improving accuracy, while being robust to data and model poisoning attacks,  and fair. Similarly, \citet{t2020personalized} formulate a bi-level optimization problem, which helps decouple personalized model optimization from learning the global model. \citet{huang2021personalized} propose the FedAMP algorithm which also introduces an additional regularization term, but differs from previous works in that they encourage similar clients to collaborate more.\\
\textbf{\textit{Model interpolation:}} \citet{hanzely2020federated} also study a mixed model (local and global model) with a tuning parameter. In their model, as the mixing parameter decreases, it relaxes the local model to be similar to the global model, which can be more personalized. \citet{mansour2020three} propose an idea to combine the global and local model with weight $\alpha$, and \citet{deng2020adaptive} adaptively find the optimal $\alpha^*$ as a trade-off at each round for the best performance. \citet{zec2020federated, peterson2019private} both consider using a gating model as a mixing parameter between local and global models. However, \citet{peterson2019private} consider a linear gating model and differentially private FL under domain adaptation, while \citet{zec2020federated} split data into two parts used for local and global learning, and they further consider a dropout scenario and the same gating model structure as local and global models. \\
\textbf{\textit{Data interpolation:}} As also suggested in \citet{mansour2020three}, in addition to the model interpolation, it is possible to combine the local and global data and train a model on their combination. \citet{zhao2018federated} create a subset of data that is globally shared across all clients. However, this method is facing the risk of information leaking. \\

\textbf{FL with Fairness.} This literature, surveyed recently in \citet{shi2023towards,rafi2024fairness}, can also be categorized depending on the adopted notion of fairness. In addition to works considering social (group) fairness in FL reviewed in the main paper, we review other types of FL fairness in detail below.\\
\textbf{\textit{Global fairness:}} 
Our work is related to improving global group fairness in FL. One approach to achieving this is through regularization techniques. \citet{abay2020mitigating} propose both pre-processing (reweighting with differential privacy) and in-processing (adding fairness-aware regularizer) methods to mitigate biases while protecting data privacy. \citet{galvez2021enforcing} extend the modified method of differential multipliers to empirical risk minimization, enforcing group fairness through fairness constraints. \citet{wang2023mitigating} show that global fairness can be expressed in terms of proper group-based local fairness and data heterogeneity levels, proposing FedGFT to improve global fairness by directly optimizing a regularized objective function. Another approach to improve global fairness involves adjusting aggregation weights. \citet{zeng2021improving} refine weights through a bi-level optimization approach. \citet{ezzeldin2023fairfed} propose fairness-aware aggregation by adjusting weights to prioritize clients with lower local fairness. Similarly, \citet{liu2025fairness} incorporate fairness considerations into aggregation by reweighting based on distributional differences from the global distribution. Unlike these works, which primarily focus on achieving global group fairness in FL, our study emphasizes local group fairness, as the learned models are ultimately deployed at the client level.\\
\textbf{\textit{Performance fairness:}} This line of work measures fairness based on how well the learned model(s) can achieve uniform accuracy across all clients. \citet{li2019fair} propose the $q$-fair FL algorithm which minimizes the aggregate reweighted loss. The idea is that the clients with higher loss will be assigned a higher weight so as to encourage more uniform accuracy across clients. \citet{li2021ditto} further extend this by considering robustness and poisoning attacks; here, performance fairness and robustness are achieved through a personalized FL method. 
\citet{zhang2021unified} aim to achieve small disparity in accuracy across the groups of client-wise, attribute-wise, and potential clients with agnostic distribution, simultaneously. \citet{wang2021federated} discuss the (performance) unfairness caused by conflicting gradients. They detect this conflict through the notion of cosine similarity, and iteratively eliminate it before aggregation by modifying the direction and magnitude of the gradients.\\
\textbf{\textit{Good-Intent fairness:}} The good-intent fairness aims to minimize the maximum loss for the protected group. \citet{mohri2019agnostic} propose a new framework of agnostic FL to mitigate the bias in the training procedure via minimax optimization. Similarly, \citet{cui2021addressing} consider a constrained multi-objective optimization problem to enforce the fairness constraint on all clients. They then maximize the worst client with fairness constraints through a gradient-based procedure. \citet{papadaki2021federating} show that a model that is minimax fair w.r.t. clients is equivalent to a relaxed minimax fair model w.r.t. demographic group. They also show their proposed algorithm leads to the same minimax group fairness performance guarantee as the centralized approaches.\\
\textbf{\textit{Other types of fairness:}} There are also other types of fairness considered in the FL literature. For instance, \citet{huang2020fairness} studied the unfairness caused by the heterogeneous nature of FL, which leads to the possibility of preference for certain clients in the training process. They propose an optimization algorithm combined with a double momentum gradient and weighting strategy to create a fairer and more accurate model. \citet{chu2021fedfair} measure fairness as the absolute loss difference between protected groups and labels, a variant of equality opportunity fairness constraint. They propose an estimation method to accurately measure fairness without violating data privacy and incorporate fairness as a constraint to achieve a fairer model with high accuracy performance. Similarly, \citet{zhang2022proportional} study a new notion of fairness, proportional fairness, in FL, which is based on the relative change of each client's performance. They connect with the Nash bargaining solution in the cooperative gaming theory and maximize the product of client utilities, where the total relative utility cannot be improved. Similarly, \citet{lyu2020collaborative} study collaborative fairness, meaning that a client who has a higher contribution to learning should be rewarded with a better-performing local model. They introduce a collaborative fair FL framework that incorporates with reputation mechanism to enforce clients with different contributions converge to different models. Their approach could also be viewed as a variant of clustering that separates clients based on their contributions.  
\section{Pseudo-codes for \texttt{Fair-FL+HC}}\label{app:pesudocode}


   

The pseudo-code for \texttt{Fair-FL+HC} is shown in Algorithm~\ref{alg:two}. 

\SetKwComment{Comment}{/* }{ */}
\RestyleAlgo{ruled}
\begin{algorithm}[h!]
\caption{\texttt{Fair-FL+HC}}\label{alg:two}

\textbf{Input}: Number of rounds $k_1$/$k_2$ before/after clustering, number of clients $n$, number of local updates $E$, learning rate $\eta$, Model initialization $\theta$, Set of hyperparameters for the
hierarchical clustering algorithm $P$, fairness-accuracy tradeoff $\gamma$, fairness $f \in \{\texttt{SP}, \texttt{EqOp}, \texttt{EO}\}$.\\

\For{each round $t \in [k_1]$}{
    \textbf{Output} \texttt{FedAvg} model: $\theta = \texttt{FedAvg}(\theta, n)$
  }
  
\For{each client $i \in [n]$ in parallel}{
    \textbf{Initialize} $\theta_i = \theta$\\
    \textbf{Perform} $E$ steps of local update: $\theta_i = \theta_i - \eta \nabla_{\theta_i} F(Z_i,\theta_{i})$ \\
  }
$\text{Clusters} = \text{HierarchicalClustering}(\gamma \mathbf{D} + (1-\gamma) \mathbf{\Psi}^f(Z, \theta), P)$\\
\For{each cluster $c \in \text{Cluster}$ in parallel}{
\For{each client $i \in c$ in parallel}{
    \textbf{Perform} \texttt{FedAvg} procedures: $\texttt{FedAvg}(\theta_i, n_c), n_c \in c$ \\}
  }
\end{algorithm}

\section{Proofs} \label{app:proof}
\subsection{Proof of Proposition~\ref{prop3}}\label{app_proof:prop_eqop}

Our proof follows these steps: First, we establish the relative positions of the optimal decision thresholds for the clustered FL algorithm $(\theta^*_\alpha, \theta^*_\beta)$ and the FedAvg algorithm $\theta^*_G$ (Lemma~\ref{lemma1:relative_location}). Next, we analyze the cluster-wise fairness performance when using cluster-specific optimal decision thresholds compared to the FedAvg solution (Lemma~\ref{lemma2:improved_eqop} and \ref{lemma3:deteriorated_eqop}). Finally, we integrate these results to demonstrate that the overall cluster-wise average fairness performance, based on cluster-specific optimal decision thresholds, outperforms the FedAvg solution (Proposition 1).

Let $f^{y,c}_g(x)$ denote the feature-label distribution of group $g$ label $y$ data, for clients in cluster $c$. We assume that clients use their local data to train threshold-based binary classifiers $h_{\theta}(x):\mathbb{R}\rightarrow \{0,1\}$,\footnote{Our analysis assumes one-dimensional features and threshold classifiers. The former can be viewed as the one-dimensional representation of multi-dimensional features obtained from the last layer outputs of a neural network. For the latter, existing works \citep{corbett2017algorithmic, raab2021unintended} show that threshold classifiers are optimal when multi-dimensional features can be properly mapped into a one-dimensional space.} where samples with features $x \geq \theta$ are classified as label 1  (i.e., $\hat{y}(\theta)=1$). Clients select these thresholds to minimize classification errors. Formally, consider a client $j$ from cluster $c$. Let $r^c_g$ denote the fraction of its samples belonging to group $g$, and let $\alpha^{y,c}_g$ represent the fraction of samples from group $g$ with true label $y$. The client determines its optimal decision threshold ${\theta_j^*}$ by empirically solving the following optimization problem:
\begin{align}
\theta_j^* = \arg\min_{{\theta}} ~~ \sum_{g\in \{a,b\}} r^c_g \Big( \alpha^{1,c}_g \int_{-\infty}^{\theta} f^{1,c}_{g}(x)\mathrm{d}x + \alpha^{0,c}_g\int_{\theta}^{+\infty} f^{0,c}_{g}(x)\mathrm{d}x \Big)~.
 \label{eq:alg-obj}
\end{align}

\begin{lemma}[Relative Location of Optimal Thresholds]\label{lemma1:relative_location}
Under the assumptions of our problem setup, the optimal solution $\theta^*_{G}$ for the FedAvg algorithm will lie between $\theta^{*}_\alpha$ and $\theta^{*}_\beta$. 
\end{lemma}

\begin{proof}
We prove this by contradiction. By definition, the optimal threshold for FedAvg algorithm is given by $\theta^*_{G}:=\arg\min p*\sum_{j\in \mathcal{C}_\alpha} \mathcal{L}_j(\theta) + (1-p)*\sum_{j\in \mathcal{C}_\beta} \mathcal{L}_j(\theta)$, while the optimal thresholds for the clustered FL algorithm are $\theta^{*}_i:=\arg\min \sum_{j\in \mathcal{C}_i} \mathcal{L}_j(\theta)$, where $i \in \{\alpha, \beta\}$ and $\mathcal{L}_j$ is the objective function defined in Eq.~\ref{eq:alg-obj}. Without loss of generality, assume $\theta^{*}_\alpha < \theta^{*}_\beta$. We consider the case where $\theta^*_{G} > \theta^{*}_\beta$ (the reverse case can be shown similarly).

First, it is easy to verify that the objective function is convex in $\theta$. If $\theta^*_{G} > \theta^{*}_\beta$, then it $\sum_{j\in \mathcal{C}_\beta} \mathcal{L}_j(\theta^*_{G}) > \sum_{j\in \mathcal{C}_\beta} \mathcal{L}_j(\theta^{*}_\beta)$ since $\theta^{*}_\beta$ minimizes the loss for clients in $\mathcal{C}_\beta$. Similarly, due to convexity, we have $\sum_{j\in \mathcal{C}_\alpha} \mathcal{L}_j(\theta^*_{G}) > \sum_{j\in \mathcal{C}_\alpha} \mathcal{L}_j(\theta^{*}_\beta) > \sum_{j\in \mathcal{C}_\alpha} \mathcal{L}_j(\theta^{*}_\alpha)$. Therefore, $\theta^*_{G}$ cannot be the optimal solution, leading to a contradiction. Therefore, the FedAvg solution must lie between $\theta^{*}_\alpha$ and $\theta^{*}_\beta$.
\end{proof}

\begin{lemma}[Improved \texttt{EqOp} through clustering for $\mathcal{C}_\alpha$]\label{lemma2:improved_eqop}
For cluster $\mathcal{C}_\alpha$, assume $f^{y,c}_g(x), y\in\{0,1\}, g\in\{a,b\}$, are unimodal distributions, with modes $m^{y,c}_g$ such that $m^{0,c}_{g}\leq m^{1,c}_g, \forall g$. Suppose $\theta^{*}_\alpha < \theta^{*}_\beta$. If $\alpha^{1,c}_g \geq \alpha^{0, c}_g, \forall g$, then there exists a cluster size $\hat{p}$ such that for $p\geq \hat{p}$, we have $\Delta^\alpha_{\texttt{EqOp}}(\theta^*_\alpha)\leq \Delta^\alpha_{\texttt{EqOp}}(\theta^*_G)$; that is, the global model is unfairer than the cluster-specific model for $C_\alpha$.
\end{lemma}

\begin{proof}
In cluster $c = \mathcal{C}_\alpha$, we begin by considering the scenario where $r^c_a = r^c_b$, balanced label participation rates, and equalized distances between modes. We assume $m^{y,c}_b < m^{y,c}_a, \forall y$ (the reverse case follows similarly). Let $\Delta^\alpha_{\texttt{EqOp}}(\theta)$ denote the \texttt{EqOp} fairness gap within cluster $\mathcal{C}_\alpha$ at a given decision threshold $\theta$. By definition, it can be expressed as: 
\[\Delta^\alpha_{\texttt{EqOp}}(\theta)= \int_{\theta}^{\infty} f^{1,c}_{a}(x)\mathrm{d}x  - \int_{\theta}^{\infty} f^{1,c}_{b}(x)\mathrm{d}x. \]

Using the Leibniz integral rule, we can find the derivative of $\Delta^\alpha_{\texttt{EqOp}}(\theta)$ w.r.t. $\theta$ as follows:
\[(\Delta^\alpha_{\texttt{EqOp}})^{'}(\theta) = f^{1,c}_{b}(\theta) - f^{1,c}_{a}(\theta)\]
Let $I_{g^y,{g'}^{y'}}$ denote the value of $x$ where the densities $f^{y,c}_g(x)$ and $f^{y',c}_{g'}(x)$ are equal. It follows that the optimal decision threshold $\theta^{*}_\alpha$, obtained from ~\ref{eq:alg-obj}, can be written in closed form as:
\[\theta^{*}_\alpha = I_{a^1,{b}^{0}} = I_{b^1,{a}^{0}}\] 
When $(\Delta^\alpha_{\texttt{EqOp}})^{'}(\theta)=0$, $\theta = \infty, -\infty$ or $I_{a^1,{b}^{1}}$. Furthermore, in the extreme cases where $\theta \rightarrow \infty$ or $-\infty$, the \texttt{EqOp} fairness gap satisfies $\Delta^\alpha_{\texttt{EqOp}}(\infty) = \Delta^\alpha_{\texttt{EqOp}}(-\infty) = 0$. Therefore, to analyze the impact of the FedAvg solution $\theta^{*}_G$ on the \texttt{EqOp} fairness gap, it suffices to examine the sign of $(\Delta^\alpha_{\texttt{EqOp}})^{'}(\theta)$ at the optimal decision threshold $\theta^{*}_\alpha$ obtained from Eq.~\ref{eq:alg-obj}. 

To relax the assumption of equalized distances, we fix the modes of $f^{1,c}_a, f^{0,c}_a, f^{1,c}_b$, while varying the mode of $f^{0,c}_b$. Thus, there are two cases we can discuss: 
\begin{enumerate}
    \item  Case 1: $I_{a^1,{b}^{1}} - I_{a^0,{b}^{1}} < I_{a^0,{b}^{1}} - I_{a^0,{b}^{0}}$\\
    In this case, we could consider a smaller value of the mode of $f^0_b$. As a result, the optimal decision threshold $\theta^{*}_\alpha$ will shift to the left, resulting in a smaller value compared to the equalized distance case. In other words, it means $\theta^{*}_\alpha \leq I_{a^1,{b}^{1}}$, indicating $(\Delta^\alpha_{\texttt{EqOp}})^{'}(\theta^{*}_\alpha) \geq 0$.
    \item Case 2: $I_{a^1,{b}^{1}} - I_{a^0,{b}^{1}} > I_{a^0,{b}^{1}} - I_{a^0,{b}^{0}}$\\
    In this case, we could consider a larger value of the mode of $f^0_b$, while keeping it smaller than that of $f^0_a$, as per our assumption. In the extreme cases where they are equal, the optimal decision threshold determined from ~\ref{eq:alg-obj} would be smaller than $I_{a^0,{a}^{1}}$ because the mode of $f^{1,c}_b$ is less than that of $f^{1,c}_a$. Since $I_{a^0,{a}^{1}}$ is also smaller than $I_{a^1,{b}^{1}}$, we can still conclude $(\Delta^\alpha_{\texttt{EqOp}})^{'}(\theta^{*}_\alpha) \geq 0$.
\end{enumerate}
For the scenario where $r^c_a \neq r^c_b$, changes in $r^c_g$ affect the location of $\theta^{*}_\alpha$, but not the value of $\Delta^\alpha_{\texttt{EqOp}}(\theta)$. According to our distributional assumptions, if $r_a \geq (\text{resp.} \leq) r_b$,  the optimal threshold $\theta^{*}_\alpha$ shifts to favor group  $a$ (resp. $b$), leading to a right (resp. left) shift relative to the equalized case. In the extreme cases where $r_a \rightarrow$ 1 (resp. 0), we have $\theta^{*}_\alpha \rightarrow I_{a^0,{a}^{1}}$ (resp. $I_{b^0,{b}^{1}}$), which remains less than $I_{a^1,{b}^{1}}$, ensuring that $(\Delta^\alpha_{\texttt{EqOp}})^{'}(\theta^{*}_\alpha) \geq 0$.

Additionally, if the majority of samples are labeled as 1 (i.e., $\alpha^{1,c}_g \geq \alpha^{0,c}_g$), the optimal decision threshold $\theta^{*}_\alpha$ will shift leftward to account for label imbalance. Consequently, the sign of $(\Delta^\alpha_{\texttt{EqOp}})^{'}(\theta^{*}_\alpha)$ remains positive. Since $\theta^{*}_\alpha < \theta^*_G <\theta^{*}_\beta$, there exists a cluster size $\hat{p}$ such that for $p\geq \hat{p}$, the FedAvg solution $\theta^{*}_G$ will make the cluster $\mathcal{C}_\alpha$ unfairer (i.e., $\Delta^\alpha_{\texttt{EqOp}}(\theta^*_\alpha)\leq \Delta^\alpha_{\texttt{EqOp}}(\theta^*_G)$).
\end{proof}

\begin{lemma}[Deteriorated \texttt{EqOp} through clustering for $\mathcal{C}_\beta$]\label{lemma3:deteriorated_eqop}
For cluster $\mathcal{C}_\beta$, assume $f^{y,c}_g(x), y\in\{0,1\}, g\in\{a,b\}$, are unimodal distributions, with modes $m^{y,c}_g$ such that $m^{0,c}_{g}\leq m^{1,c}_g, \forall g$. Suppose $\theta^{*}_\alpha < \theta^{*}_\beta$. We have $\Delta^\beta_{\texttt{EqOp}}(\theta^*_\beta)\geq \Delta^\beta_{\texttt{EqOp}}(\theta^*_G)$; that is, the global model is fairer than the cluster-specific model for $C_\beta$.
\end{lemma}

\begin{proof}
    The proof follows the same reasoning as Lemma~\ref{lemma2:improved_eqop}. Since $(\Delta^\beta_{\texttt{EqOp}})^{'}(\theta^{*}_\beta) \geq 0$ and $\theta^{*}_\alpha < \theta^*_G <\theta^{*}_\beta$, it follows that $\Delta^\beta_{\texttt{EqOp}}(\theta^*_\beta)\geq \Delta^\beta_{\texttt{EqOp}}(\theta^*_G)$. 
\end{proof}

Now, we are ready to show the proof of Proposition~\ref{prop3}.

\begin{proof}
    By definition, we have
    \[\Delta_{\texttt{EqOp}}(\theta^{*}_C) = p\Delta^\alpha_{\texttt{EqOp}}(\theta^{*}_\alpha) + (1-p)\Delta^\beta_{\texttt{EqOp}}(\theta^{*}_\beta)\]
    \[\Delta_{\texttt{EqOp}}(\theta^{*}_G) = p\Delta^\alpha_{\texttt{EqOp}}(\theta^{*}_G) + (1-p)\Delta^\beta_{\texttt{EqOp}}(\theta^{*}_G)\]
    From Lemma~\ref{lemma2:improved_eqop}, we know that there exists a cluster size $\hat{p}$ such that for $p \geq \hat{p}$, we have $\Delta^\alpha_{\texttt{EqOp}}(\theta^*_\alpha)\leq \Delta^\alpha_{\texttt{EqOp}}(\theta^*_G)$; From Lemma~\ref{lemma3:deteriorated_eqop}, we know that for any cluster size $p>0$, we have $\Delta^\beta_{\texttt{EqOp}}(\theta^*_\beta)\geq \Delta^\beta_{\texttt{EqOp}}(\theta^*_G)$. Therefore, to ensure $\Delta_{\texttt{EqOp}}(\theta^{*}_C) \leq \Delta_{\texttt{EqOp}}(\theta^{*}_G)$, it requires
    \[p\Delta^\alpha_{\texttt{EqOp}}(\theta^{*}_\alpha) + (1-p)\Delta^\beta_{\texttt{EqOp}}(\theta^{*}_\beta) \leq p\Delta^\alpha_{\texttt{EqOp}}(\theta^{*}_G) + (1-p)\Delta^\beta_{\texttt{EqOp}}(\theta^{*}_G)\]

    Notice that, when $p \rightarrow 0$, we obtain $\theta^*_G \rightarrow \theta^*_\beta$, leading to $\Delta_{\texttt{EqOp}}(\theta^{*}_C) \rightarrow \Delta_{\texttt{EqOp}}(\theta^{*}_G)$. Similarly, when $p \rightarrow 1$, we have $\theta^*_G \rightarrow \theta^*_\alpha$, leading to $\Delta_{\texttt{EqOp}}(\theta^{*}_C) \rightarrow \Delta_{\texttt{EqOp}}(\theta^{*}_G)$. After rearranging terms, we will have 
    \[p \geq \frac{\Delta^\beta_{\texttt{EqOp}}(\theta^{*}_\beta) - \Delta^\beta_{\texttt{EqOp}}(\theta^{*}_G)}{\Delta^\beta_{\texttt{EqOp}}(\theta^{*}_\beta) - \Delta^\beta_{\texttt{EqOp}}(\theta^{*}_G)+\Delta^\alpha_{\texttt{EqOp}}(\theta^{*}_G)-\Delta^\alpha_{\texttt{EqOp}}(\theta^{*}_\alpha)}\]
    
    Since $\Delta^\beta_{\texttt{EqOp}}(\theta^{*}_\beta) \geq \Delta^\beta_{\texttt{EqOp}}(\theta^{*}_G)$, we can further bound this expression as:
    \begin{align*}
        &\frac{\Delta^\beta_{\texttt{EqOp}}(\theta^{*}_\beta) - \Delta^\beta_{\texttt{EqOp}}(\theta^{*}_G)}{\Delta^\beta_{\texttt{EqOp}}(\theta^{*}_\beta) - \Delta^\beta_{\texttt{EqOp}}(\theta^{*}_G)+\Delta^\alpha_{\texttt{EqOp}}(\theta^{*}_G)-\Delta^\alpha_{\texttt{EqOp}}(\theta^{*}_\alpha)}\leq\frac{\Delta^\beta_{\texttt{EqOp}}(\theta^{*}_\beta) - \Delta^\beta_{\texttt{EqOp}}(\theta^{*}_G)}{\Delta^\alpha_{\texttt{EqOp}}(\theta^{*}_G)-\Delta^\alpha_{\texttt{EqOp}}(\theta^{*}_\alpha)}\\
        &\qquad = \frac{\int_{\theta^*_\beta}^{\infty} f^{1,\beta}_a(x) dx - \int_{\theta^*_\beta}^{\infty} f^{1,\beta}_b(x) dx - \int_{\theta^*_G}^{\infty} f^{1,\beta}_a(x) dx - \int_{\theta^*_G}^{\infty} f^{1,\beta}_b(x) dx}{\int_{\theta^*_\alpha}^{\infty} f^{1,\alpha}_a(x) dx - \int_{\theta^*_\alpha}^{\infty} f^{1,\alpha}_b(x) dx - \int_{\theta^*_G}^{\infty} f^{1,\alpha}_a(x) dx - \int_{\theta^*_G}^{\infty} f^{1,\alpha}_b(x) dx}\\
        &\qquad = \Big |\frac{\int_{\theta^*_G}^{\theta^*_\beta} f^{1,\beta}_a(x) - f^{1,\beta}_b(x) dx}{\int_{\theta^*_\alpha}^{\theta^*_G} f^{1,\alpha}_a(x) - f^{1,\alpha}_b(x)dx} \Big |
    \end{align*}
    Therefore, there exists a cluster size $\hat{p}$ such that for $p\geq \hat{p}$, we have $\Delta_{\texttt{EqOp}}(\theta^*_C)<\Delta_{\texttt{EqOp}}(\theta^*_G)$, where $\hat{p}$ is the solution to: $\hat{p}=\min\{1, |\tfrac{\int_{\theta^*_G}^{\theta^*_\beta} f^{1,\beta}_a(x)-f^{1,\beta}_b(x)dx}{\int_{\theta^*_\alpha}^{\theta^*_G} f^{1,\alpha}_a(x)-f^{1,\alpha}_b(x)dx}|\}$, and $\theta^*_{G}$ is obtained as: $\theta^*_{G} =\arg\min_{\theta} \hat{p}*\sum_{j\in \mathcal{C}_\alpha} \mathcal{L}_j(\theta) + (1-\hat{p})*\sum_{j\in \mathcal{C}_\beta} \mathcal{L}_j(\theta)$. 
\end{proof}
\subsection{Proof of Proposition~\ref{prop1}}\label{app_proof:prop_sp}

\begin{proof}
We start with the scenario where $r_a = r_b$, balanced label participation rate, and equalized distance between mean estimates. As the following analysis focuses on the cluster $\mathcal{C}_\alpha$, we drop the cluster notation from the derivation for notation simplicity. We assume $\mu^{y,c}_b < \mu^{y,c}_a, \forall y$ (the reverse case follows similarly). Let $\Delta_{\texttt{SP}}(\theta)$ be the \texttt{SP} fairness gap within the clsuter $\mathcal{C}_\alpha$ at the given decision threshold $\theta$. By definition, it can be expressed as: 
\[\Delta_{\texttt{SP}}(\theta)= \alpha^{1}_a \int_{\theta}^{\infty} f^{1}_{a}(x)\mathrm{d}x + \alpha^{0}_a \int_{\theta}^{\infty} f^{0}_{a}(x)\mathrm{d}x - \alpha^{1}_b \int_{\theta}^{\infty} f^{1}_{b}(x)\mathrm{d}x - \alpha^{0}_b \int_{\theta}^{\infty} f^{0}_{b}(x)\mathrm{d}x. \]
According to the Leibniz integral rule, we can find the derivative of $\Delta_{\texttt{SP}}(\theta)$ w.r.t. $\theta$ as following:
\[\Delta_{\texttt{SP}}^{'}(\theta) = \alpha^{1}_bf^{1}_{b}(\theta) + \alpha^{0}_bf^{0}_{b}(\theta) - \alpha^{1}_a f^{1}_{a}(\theta) - \alpha^{0}_af^{0}_{a}(\theta)\]
According to our distribution assumptions, we can write the above expression in the following closed form with $\alpha = \alpha^y_g \hspace{0.05in}\forall y, g$ 
\[\Delta_{\texttt{SP}}^{'}(\theta) = \frac{\alpha}{\sqrt{2\pi}\sigma}\Big( \exp(-\frac{(\theta-\mu^1_b)^2}{2\sigma^2}) + \exp(-\frac{(\theta-\mu^0_b)^2}{2\sigma^2}) - \exp(-\frac{(\theta-\mu^1_a)^2}{2\sigma^2}) - \exp(-\frac{(\theta-\mu^0_a)^2}{2\sigma^2}) \Big)\]
Furthermore, it is easy to verify that the optimal decision threshold $\theta^{*}_\alpha$ obtained by solving ~\ref{eq:alg-obj} could be written in the closed form such that 
\[\bar{\theta} = \theta^{*}_\alpha = \frac{\mu^1_a + \mu^0_b + \mu^1_b + \mu^0_a}{4}\]
At the optimal solution $\theta^{*}_\alpha$, $\Delta_{\texttt{SP}}^{'}(\theta^{*}_\alpha) = 0$. Similar to the proof of Proposition~\ref{prop3}, to investigate the impact of FedAvg solution $\theta^{*}_G$ on the \texttt{SP} fairness gap, it is equivalent to check how the $\Delta_{\texttt{SP}}^{'}(\theta^{*}_\alpha)$ change in the neighborhood of the optimal solution $\theta^{*}_\alpha$. Also, at extreme cases, we can easily find that the value of \texttt{SP} fairness gap $\Delta_{\texttt{SP}}(\infty) = \Delta_{\texttt{SP}}(-\infty) = 0$. Therefore, if $\Delta_{\texttt{SP}}^{'}(\theta^{*}_\alpha) \geq 0$, then we can conclude that the FedAvg solution $\theta^{*}_G$ would lead to a worse fairness performance compared to the optimal solution $\theta^{*}_\alpha$. Let $\psi_1(\theta) = \exp(-\frac{(\theta-\mu^1_b)^2}{2\sigma^2}) - \exp(-\frac{(\theta-\mu^0_a)^2}{2\sigma^2})$ and $\psi_2(\theta) = \exp(-\frac{(\theta-\mu^1_a)^2}{2\sigma^2}) - \exp(-\frac{(\theta-\mu^0_b)^2}{2\sigma^2})$. At the solution $\theta^{*}_\alpha$, we can find that $\psi_1(\theta^{*}_\alpha) = \psi_2(\theta^{*}_\alpha) = 0$. 

Hence, to investigate how the $\Delta_{\texttt{SP}}^{'}(\theta^{*}_\alpha)$ change, we can find the rate of change for both $\psi_1(\theta)$ and $\psi_2(\theta)$ in the neighborhood of $\theta^{*}_\alpha$ such that 
\[\psi^{'}_1(\theta^{*}_\alpha) = \exp(\tfrac{(\theta^{*}_\alpha-\mu^0_a)^2}{-2\sigma^2})\tfrac{\theta^{*}_\alpha-\mu^0_a}{\sigma} -\exp(\tfrac{(\theta^{*}_\alpha-\mu^1_b)^2}{-2\sigma^2})\tfrac{\theta^{*}_\alpha-\mu^1_b}{\sigma} = \tfrac{1}{\sigma}\exp(\tfrac{(\theta^{*}_\alpha-\mu^1_b)^2}{-2\sigma^2})(\mu^1_b - \mu^0_a)\]
\[\psi^{'}_2(\theta^{*}_\alpha) = \exp(\tfrac{(\theta^{*}_\alpha-\mu^0_b)^2}{-2\sigma^2})\tfrac{\theta^{*}_\alpha-\mu^0_b}{\sigma} -\exp(\tfrac{(\theta^{*}_\alpha-\mu^1_a)^2}{-2\sigma^2})\tfrac{\theta^{*}_\alpha-\mu^1_a}{\sigma} = \tfrac{1}{\sigma}\exp(\tfrac{(\theta^{*}_\alpha-\mu^0_b)^2}{-2\sigma^2})(\mu^1_a - \mu^0_b)\]
By setting $\psi^{'}_1(\theta^{*}_\alpha) \geq \psi^{'}_2(\theta^{*}_\alpha)$, it means the increment of $\psi_1$ is larger than the decrement of $\psi_2$. Therefore, with Lemma~\ref{lemma1:relative_location}, 
there exists a cluster size weight $p$ such that the FedAvg solution $\theta^{*}_G$ will make the cluster $\mathcal{C}_\alpha$ unfairer. 

To relax the assumption of equalized distances between mean estimates, $\Delta_{\texttt{SP}}^{'}(\theta^{*}_\alpha) \geq 0$ still holds following the same reasoning as cases shown in Lemma~\ref{lemma2:improved_eqop}.

For the scenario where $r_a \neq r_b$, similar to the proof of Proposition~\ref{prop3}, changes in $r_g$ affect the location of $\theta^{*}_\alpha$, but not the value of $\Delta_{\texttt{SP}}(\theta)$. According to our distributional assumptions, if $r_a \geq (\text{resp.} \leq) r_b$, the optimal solution $\theta^{*}_\alpha$ shifts to favor group $a$ (resp. $b$), leading to a right (resp. left) shift relative to the equalized case. In the extreme cases where $r_a \rightarrow$ 1 (resp. 0), we have $\theta^{*}_\alpha \rightarrow \frac{\mu^0_a + \mu^1_a}{2}$ (resp. $\frac{\mu^0_b + \mu^1_b}{2}$), which is limited within the range of $(\frac{\mu^0_a + \mu^0_b}{2}, \frac{\mu^1_a + \mu^1_b}{2})$. When $\theta = \frac{\mu^1_a + \mu^1_b}{2}$, we can easily find that $\Delta_{\texttt{SP}}^{'}(\theta) \approx 0$ especially when $\sigma$ is small. In other words, we can conclude that $\Delta_{\texttt{SP}}(\theta) \geq \Delta_{\texttt{SP}}(\theta^{*}_\alpha)$ for any $\theta^{*}_\alpha \in (\frac{\mu^0_b + \mu^1_b}{2}, \frac{\mu^0_a + \mu^1_a}{2})$. Therefore, the claim still holds. 

Additionally, similar to Proposition~\ref{prop3}, when the equalized label participation rate assumption is relaxed, the above proof strategy still holds by considering different $\alpha^y_g$ into the expression. It is worth noting that when the label participation rates are balanced, the fairness $\Delta_{\texttt{SP}}(\theta)$ has two equal-height peaks (e.g., $\Delta_{\texttt{SP}}'(\theta) = 0$) by symmetricity of the Gaussian distribution when $\theta \approx \frac{\mu^1_a+\mu^1_b}{2}$ and $\frac{\mu^0_a+\mu^0_b}{2}$. However, when the majority of samples are labeled as 1, we observe a shift in the decision threshold $\theta^{*}_\alpha \leq \bar{\theta}$ towards the left to account for label imbalance. Consequently, the sign of $\Delta_{\texttt{SP}}^{'}(\theta^{*}_\alpha)$ remains positive. Since $\theta^{*}_\alpha < \theta^*_G <\theta^{*}_\beta$, there exists a cluster size $\hat{p}$ such that for $p\geq \hat{p}$, the FedAvg solution $\theta^{*}_G$ will make the cluster $\mathcal{C}_\alpha$ unfairer (i.e., $\Delta^\alpha_{\texttt{SP}}(\theta^*_\alpha)\leq \Delta^\alpha_{\texttt{SP}}(\theta^*_G)$). In this case, the FedAvg solution pulls $\theta^{*}_\alpha$ upwards, favoring label 1, which results in both accuracy and fairness deteriorating. 
\end{proof}
\section{Experiment details and additional numerical experiments} \label{app:numerical}

\subsection{Dataset and models} \label{app: data_model}

In this section, we detail the data and model used in our experiments. 

\textbf{Retiring Adult dataset.} We use the pre-processed dataset provided by the folktables Python package \citep{ding2021retiring}, which provides access to datasets derived from the US Census. In this package, there are three tasks: ACSEmployment, ACSIncome, and ACSHealth. 
For the ACSEmployemnt task, the goal is to predict whether the person is employed based on its multi-dimensional features; for the ACSIncome task, the goal is to predict whether the person earns more than \$50,000 annually; and for the ACSHealth task, the goal is to predict whether the person is covered by insurance. 

\textbf{Model.} We train a fully connected two-layer neural network model for both tasks, where the hidden layer has 32 neurons for the ACSIncome task, and 64 neurons for the ACSEmployment and ACSHealth tasks. For all tasks, we use the RELU activation function and a batch size of 32. Furthermore, we utilize the SGD optimizer for training, with a learning rate of 0.001 for both FedAvg and MAML algorithms and 0.05 for the clustered FL algorithm. In FL, each client updates the global model for 10 epochs in the FedAvg and MAML algorithms and sends it back to the server, while the clustered FL algorithm that has a larger learning rate updates the global model for 1 epoch. We also follow the encoding procedure for categorical features provided by the folktables Python package. The input feature size is 54, 109 and 154 for the ACSIncome, ACSEmployment and ACSHealth tasks, respectively. In the experiments, we consider either sex (e.g., male and female) or race (e.g., White and Non-White) as the protected attribute.

\textbf{ACSEmployment task with different protected attributes}. 

\begin{figure}[ht]
\vspace{-0.1in}
	\centering
	\subfigure{
 \includegraphics[width=0.45\textwidth]{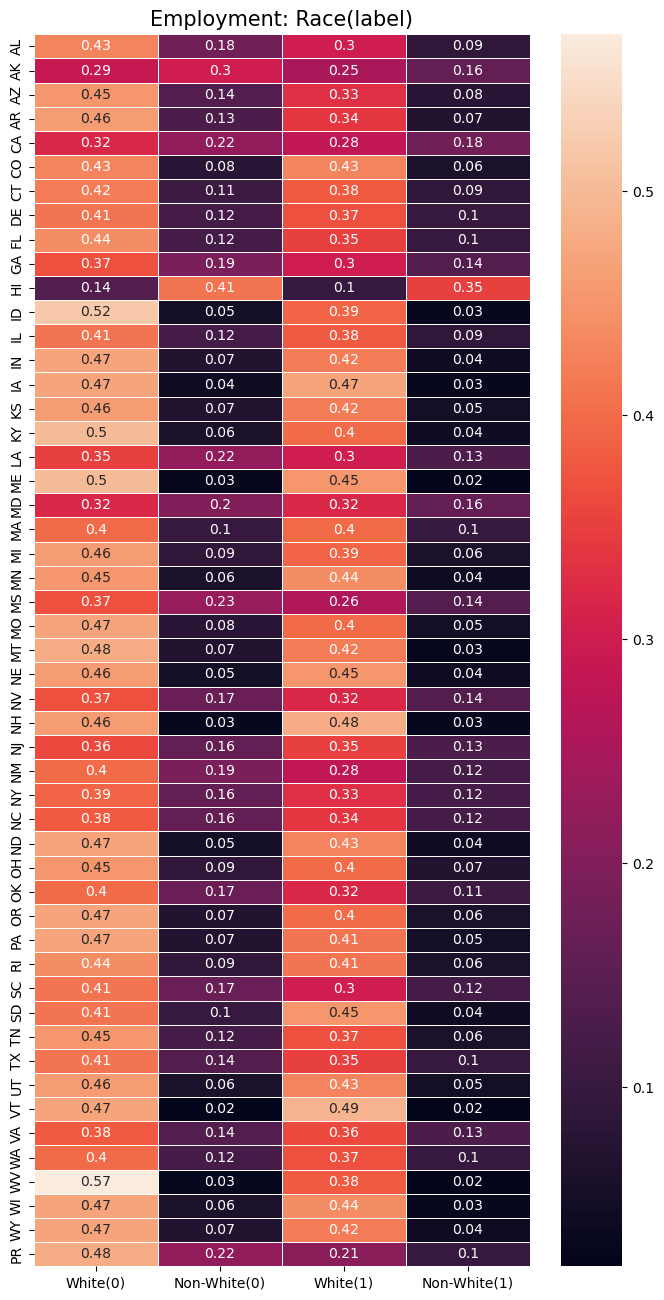}
	}
     \subfigure{
     \includegraphics[width=0.45\textwidth]{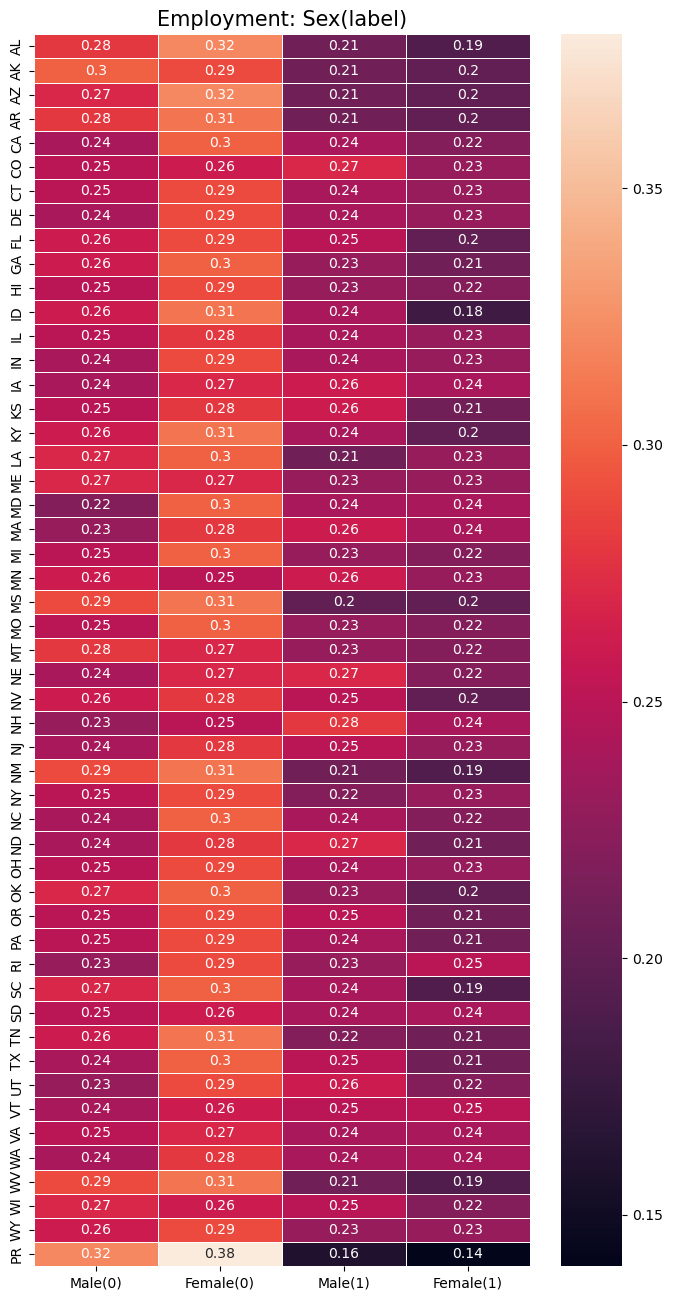}
    	}
     \vspace{-0.1in}
	\caption{Fraction of samples over all states for ACSEmployment}
	 \label{fig:race_emp_full}
\end{figure}

\begin{figure}[ht]
\vspace{-0.1in}
	\centering
	\subfigure[Protected attribute: Race]{
 \includegraphics[width=0.8\textwidth]{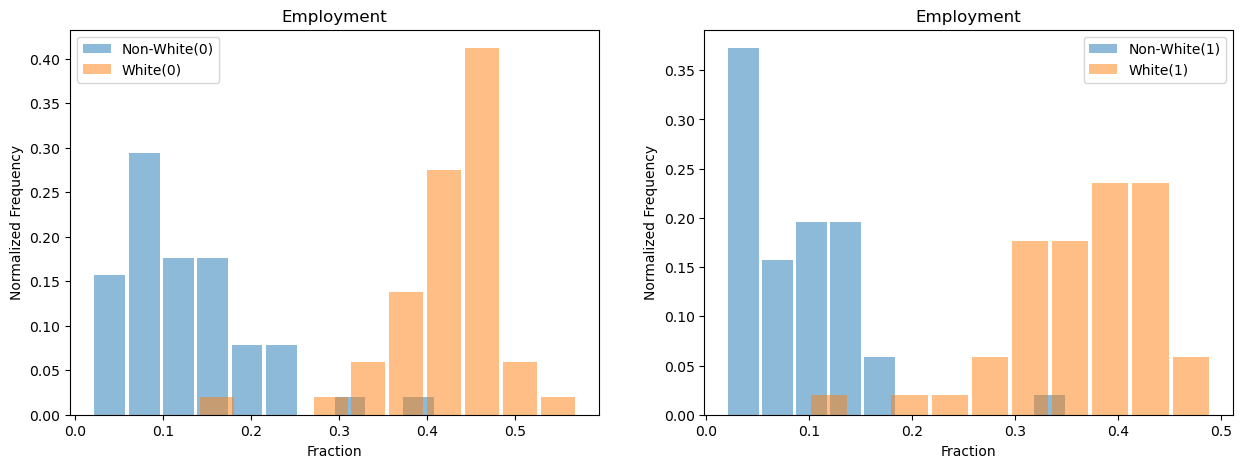}
	}
 \vspace{0.1in}
 \subfigure[Protected attribute: Sex]{
 \includegraphics[width=0.8\textwidth]{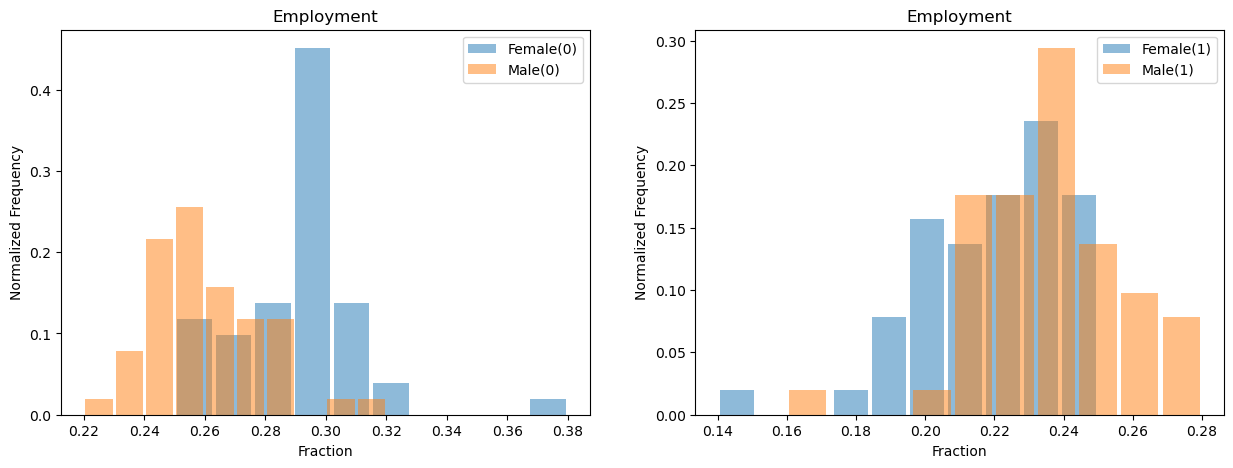}
	}
  \vspace{-0.1in}
	\caption{Normalized frequency of fraction of samples for ACSEmployment}
	 \label{fig:hist_emp_race_full}
\end{figure}

As shown in Figure~\ref{fig:race_emp_full} and \ref{fig:hist_emp_race_full}, within the same ACSEmployment task, the data distributions for race (left) and sex (right) are significantly different. For the protected attribute of sex, the number of samples is nearly even across groups and labels. However, for the protected attribute of race, the White group has significantly more samples compared to Non-White groups for both labels ${0,1}$.

\textbf{ACSIncome and ACSHealth tasks with protected attribute of sex}. 

We can see from Figure\ref{fig:hist_inc_race_full}, \ref{fig:hist_health_sex_full} and ~\ref{fig:sex_inc_full} that the fraction of samples in the ACSIncome task is similar across groups for label 0 data but differs significantly for label 1 data. Additionally, we can observe that the ACSHealth task has similar fractions of samples from each group, akin to the ACSEmployment task, in contrast to the ACSIncome task.

\begin{figure}[ht]
\vspace{-0.1in}
	\centering
 \subfigure[Protected attribute: Sex]{
 \includegraphics[width=0.8\textwidth]{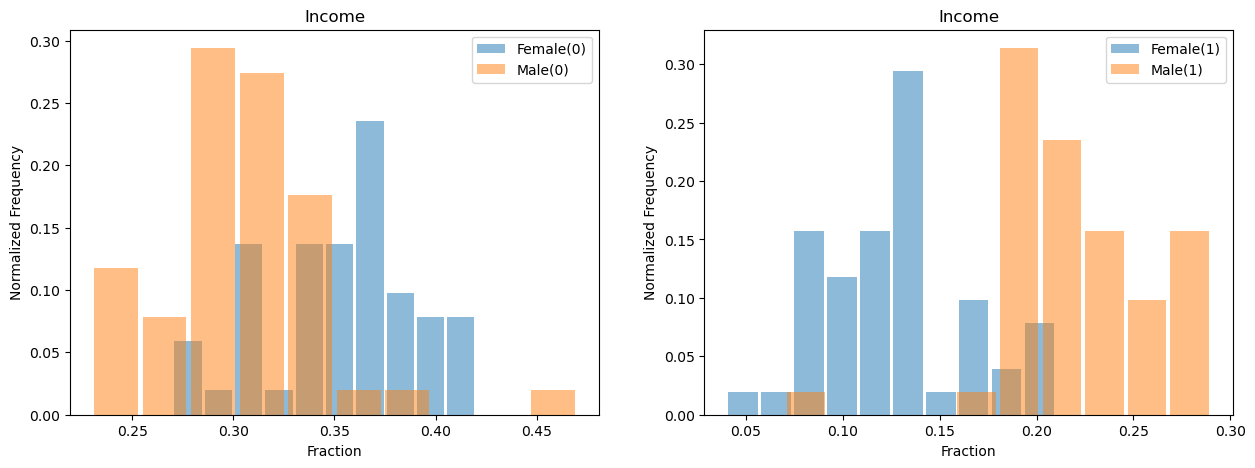}
	}
  \vspace{-0.1in}
	\caption{Normalized frequency of fraction of samples for ACSIncome}
	 \label{fig:hist_inc_race_full}
\end{figure}

\begin{figure}[ht]
\vspace{-0.1in}
	\centering
 \subfigure[Protected attribute: Sex]{
 \includegraphics[width=0.8\textwidth]{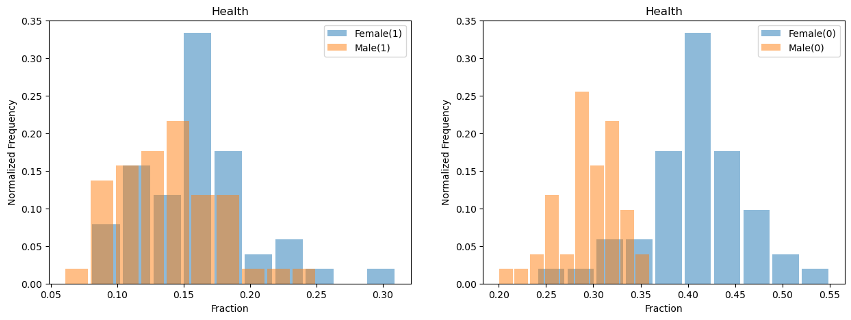}
	}
  \vspace{-0.1in}
	\caption{Normalized frequency of fraction of samples for ACSHealth}
	 \label{fig:hist_health_sex_full}
\end{figure}

\begin{figure}[ht]
\vspace{-0.1in}
	\centering
     \subfigure{
     \includegraphics[width=0.45\textwidth]{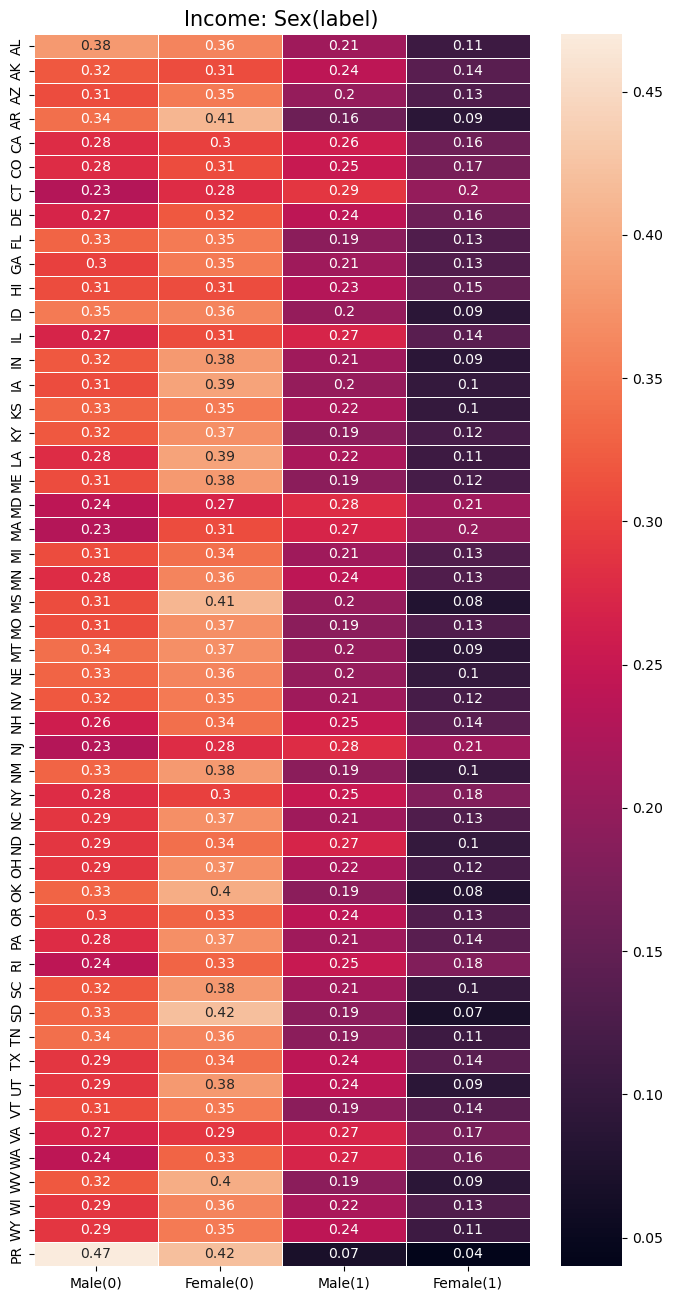}
    	}
     \subfigure{
     \includegraphics[width=0.45\textwidth]{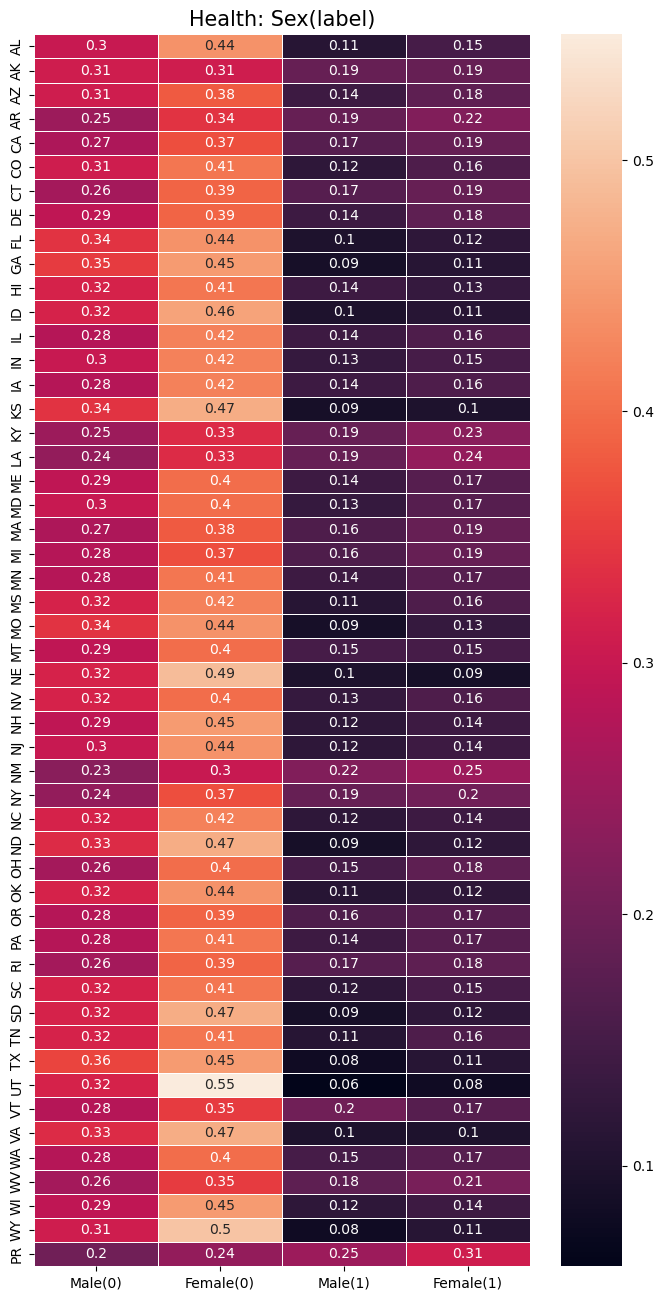}
    	}
     \vspace{-0.1in}
	\caption{Fraction of samples over all states for ACSIncome and ACSHealth}
	 \label{fig:sex_inc_full}
\end{figure}

\subsection{Additional Experiment and setup on the level of imbalance in the number of samples}\label{app_imbalance_tables_and sweep}

\begin{figure}[ht]
\vspace{-0.1in}
	\centering
	\subfigure[Group Imbalance]{		\includegraphics[width=0.32\textwidth]{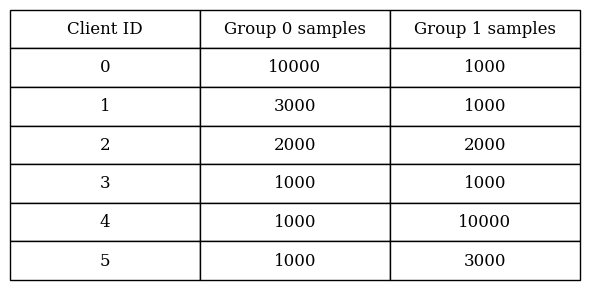} 
	}
        \hspace{-0.05in}
	\subfigure[Label Imbalance]{
\includegraphics[width=0.32\textwidth]{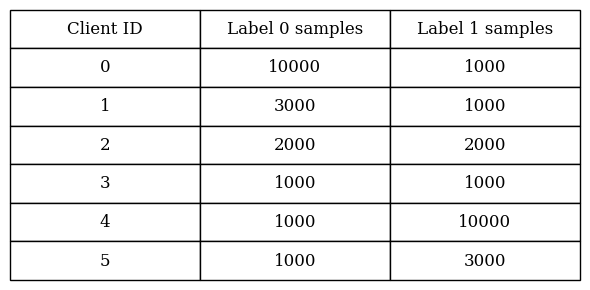} 
	}
     \hspace{-0.05in}
    	\subfigure[Feature Imbalance]{
		\includegraphics[width=0.32\textwidth]{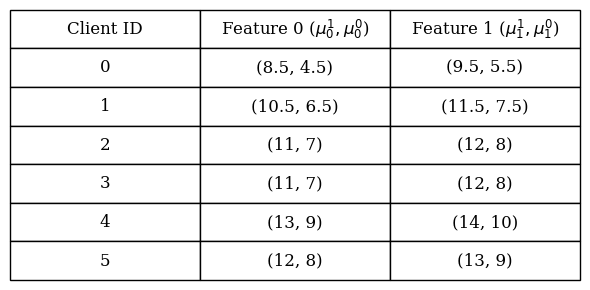} 
	}
	\hspace{-0.1in}
	\caption{Experiment setup details of the level of imbalance}
	\label{fig:table_imbalance}
\end{figure}

\begin{figure}[ht]
\vspace{-0.1in}
	\centering
	\subfigure[Adult (Sex)]{
		\includegraphics[width=0.5\textwidth]{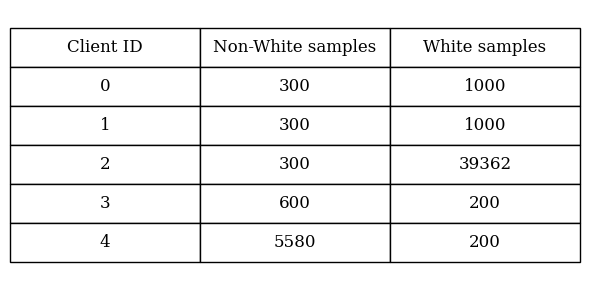} 
	}
	\hspace{-0.1in}
	\subfigure[Local acc. vs. fairness]{		\includegraphics[width=0.3\textwidth]{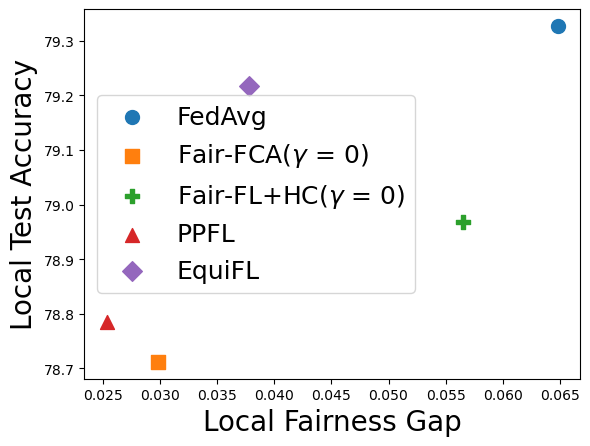}
	}
 \vspace{-0.1in}
	\caption{\texttt{Fair-FCA} and \texttt{Fair-FL+HC} on \emph{Adult} dataset ($\gamma = 0, f=\texttt{SP}$)}
  \vspace{-0.1in}
\end{figure}

The experiment setup for Fig.~\ref{fig:synthetic_imbalance} is shown in Fig.~\ref{fig:table_imbalance}. We now explore the impact of different parameter settings on the local group fairness performance. The experiment results shown in Fig.~\ref{fig:group_imbalance_sweep}-\ref{fig:feature_imbalance_sweep} are consistent with our findings. We also have the experiment results using the real-world dataset shown in Fig.~\ref{fig:experiments_retiring_0}.

\begin{figure}[ht]
\vspace{-0.1in}
	\centering
	\subfigure[Balanced Label Rate]{
		\includegraphics[width=0.3\textwidth]{Figs/synthetic_group.png} 
	}
	\hspace{-0.1in}
	\subfigure[Label 1:0 = 20\% : 80\%]{		\includegraphics[width=0.295\textwidth]{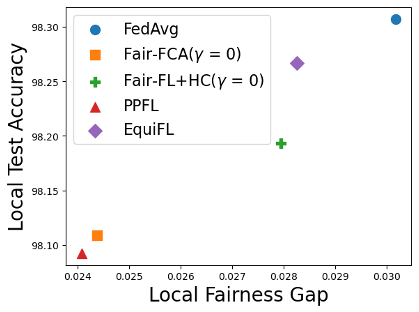}
	}
        \hspace{-0.05in}
	\subfigure[Label 1:0 = 20\% : 80\%]{
\includegraphics[width=0.295\textwidth]{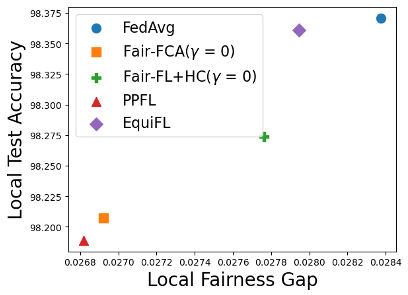} 
	}
 \vspace{-0.1in}
	\caption{\texttt{Fair-FCA} and \texttt{Fair-FL+HC} under different label rate combinations}
	\label{fig:group_imbalance_sweep}
  \vspace{-0.1in}
\end{figure}

\begin{figure}[ht]
\vspace{-0.1in}
	\centering
	\subfigure[Balanced Group Rate]{
		\includegraphics[width=0.3\textwidth]{Figs/synthetic_label.png} 
	}
	\hspace{-0.1in}
	\subfigure[Group 1:0 = 40\% : 60\%]{		\includegraphics[width=0.29\textwidth]{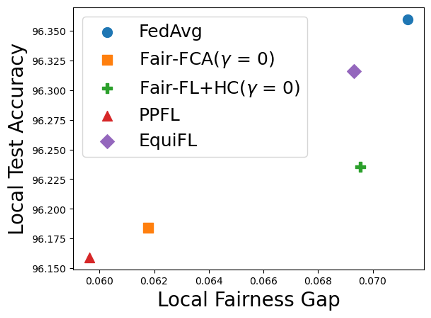}
	}
        \hspace{-0.05in}
	\subfigure[Group 1:0 = 60\% : 40\%]{
\includegraphics[width=0.3\textwidth]{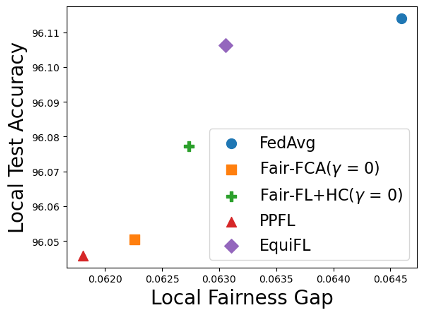} 
	}
 \vspace{-0.1in}
	\caption{\texttt{Fair-FCA} and \texttt{Fair-FL+HC} under different group rate combinations}
	\label{fig:label_imbalance_sweep}
  \vspace{-0.1in}
\end{figure}

\begin{figure}[ht]
\vspace{-0.1in}
	\centering
	\subfigure[Balanced Group/Label Rate]{
		\includegraphics[width=0.3\textwidth]{Figs/synthetic_feature.png} 
	}
	\hspace{-0.1in}
	\subfigure[Group 1:0 = 30\% : 70\%,  Label 1:0 = 75\% : 25\%]{		\includegraphics[width=0.3\textwidth]{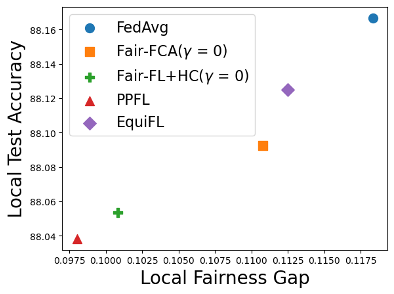}}
        \hspace{-0.05in}
	\subfigure[Group 1:0 = 70\% : 30\%,  Label 1:0 = 25\% : 75\%]{
\includegraphics[width=0.3\textwidth]{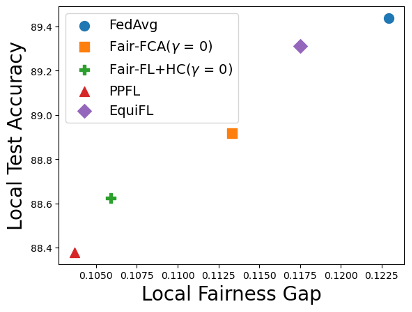} 
	}
 \vspace{-0.1in}
	\caption{\texttt{Fair-FCA} and \texttt{Fair-FL+HC} under different group/label rate combinations}
	\label{fig:feature_imbalance_sweep}
      \vspace{-0.1in}
\end{figure}

\begin{figure}[ht]
\vspace{-0.1in}
	\centering
	\subfigure[Employment (Race)]{
		\includegraphics[width=0.3\textwidth]{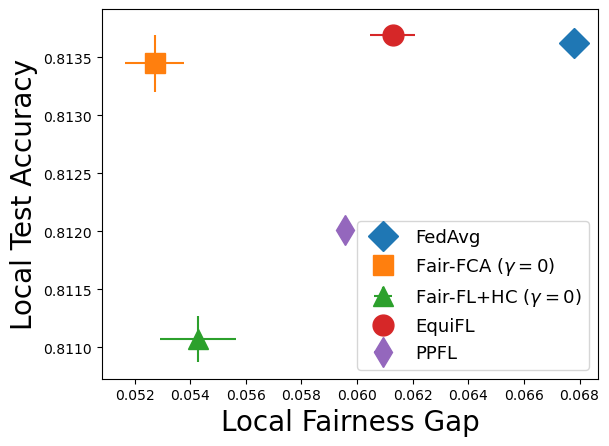} 
	}
	\hspace{-0.1in}
	\subfigure[Employment (Sex)]{		\includegraphics[width=0.3\textwidth]{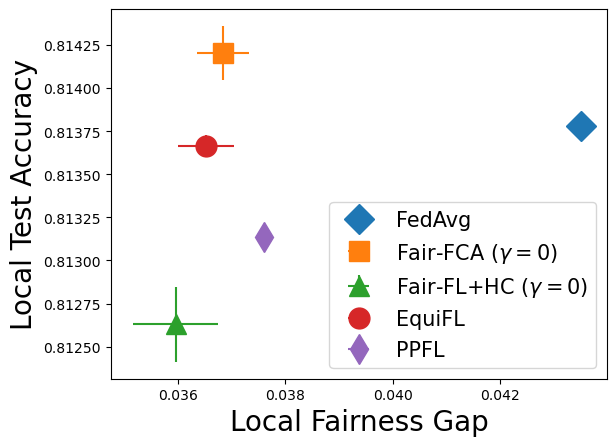}
	}
        \hspace{-0.05in}
	\subfigure[Income (Sex)]{
\includegraphics[width=0.3\textwidth]{Figs/inc_sex_0.png} 
	}
 \vspace{-0.1in}
	\caption{\texttt{Fair-FCA} and \texttt{Fair-FL+HC} with different attributes and tasks ($\gamma = 0, f=\texttt{SP}$)}
	\label{fig:experiments_retiring_0}
  \vspace{-0.1in}
\end{figure}

\subsection{Additional experiments on the fairness-accuracy trade-offs}\label{app_experiment_trade_offs}

We compare our proposed algorithms with $\gamma \in \{0, 0.5\}$ and with $f=\texttt{SP}$, in terms of local accuracy and local group fairness. We proceed with the same experiment setting as in Section~\ref{subsec:comparison_gamma0}. Figure \ref{fig:experiments_new_mix_dp} shows that across different datasets (ACSEmployment, ACSIncome) and protected attributes (race, sex), \texttt{Fair-FCA} and \texttt{Fair-FL+HC} (symbols in squares and triangles) exhibit a fairness-accuracy tradeoff, where improved fairness comes at the cost of some accuracy degradation. 

\begin{figure}[ht]
\vspace{-0.1in}
	\centering
	\subfigure[Employment (Race)]{
		\includegraphics[width=0.3\textwidth]{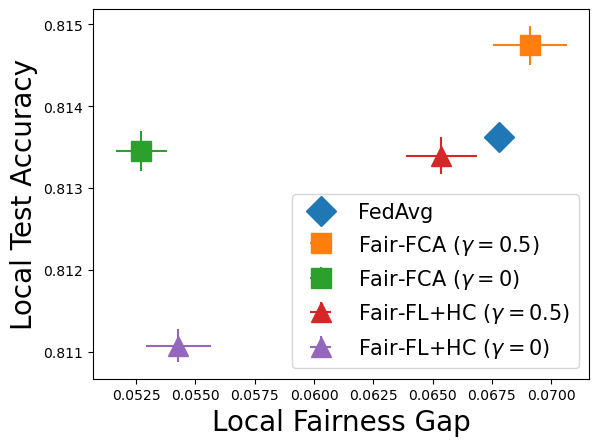} 
	}
	\hspace{-0.1in}
	\subfigure[Employment (Sex)]{		\includegraphics[width=0.3\textwidth]{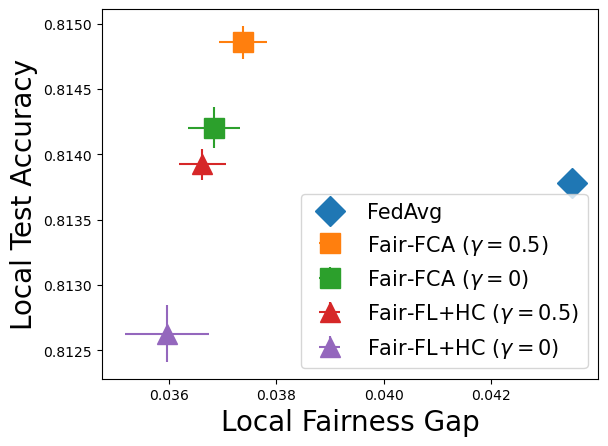} 
	}
        \hspace{-0.05in}
	\subfigure[Income (Sex)]{
\includegraphics[width=0.3\textwidth]{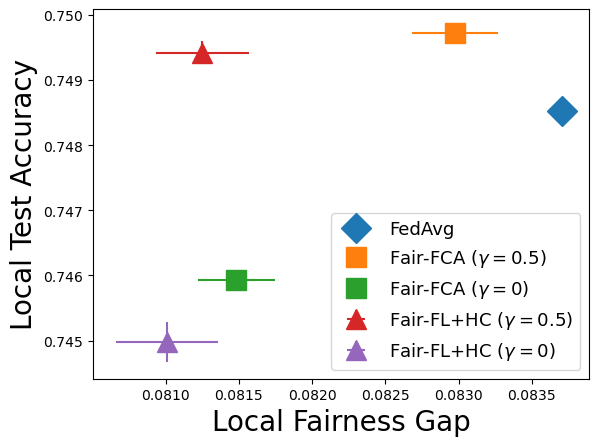} 
	}
 \vspace{-0.1in}
	\caption{\texttt{Fair-FCA} and \texttt{Fair-FL+HC} with different tasks and attributes ($\gamma = \{0, 0.5\}, f=\texttt{SP}$)}
	\label{fig:experiments_new_mix_dp}
  \vspace{-0.1in}
\end{figure}

We also compare our proposed algorithms with other existing personalized FL and locally fair FL methods with $\gamma \in \{0, 0.5, 1\}$ and with $f=\texttt{SP}$, in terms of local accuracy and local group fairness. Different from Figure \ref{fig:experiments_new_mix_dp}, we consider the original Retiring dataset without any feature scaling. Further, instead of considering 51 clients, we choose the top 3 and the bottom 3 clients, sorted by the White Race. Our results in Fig.~\ref{fig:group_retiring} are consistent with the findings in Fig.~\ref{fig:synthetic_tradeoff}.

\begin{figure}[ht]
	\centering
	\subfigure[Clients (Race)]{
		\includegraphics[width=0.4\textwidth]{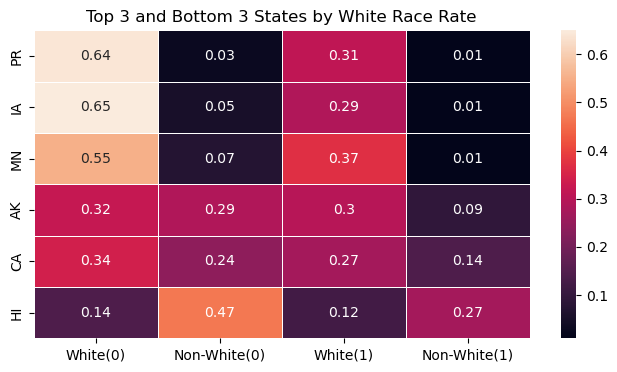} 
	}
	\hspace{-0.1in}
	\subfigure[Employment no scaling]{		\includegraphics[width=0.3\textwidth]{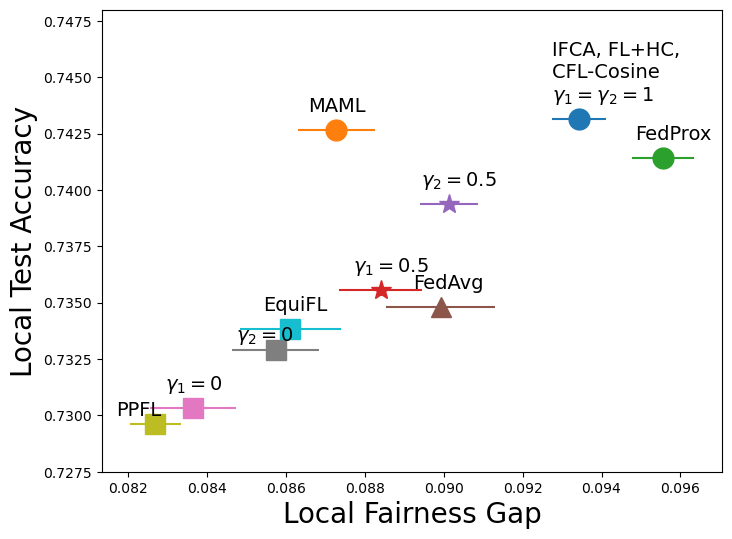} 
	}
 \vspace{-0.1in}
	\caption{Comparison of our methods with existing personalized FL and locally fair FL methods, under $\gamma = \{0, 0.5,1\}, f=\texttt{SP}$ on original Retiring datasets without feature scaling.}
	\label{fig:group_retiring}
  \vspace{-0.1in}
\end{figure}

\subsection{Additional experiments on other types of fairness notions}\label{app:numerical_other_notion}

\begin{figure}[ht]
\vspace{-0.1in}
	\centering
	\subfigure[Employment-sex: \texttt{EO}]{
 \includegraphics[width=0.3\textwidth]{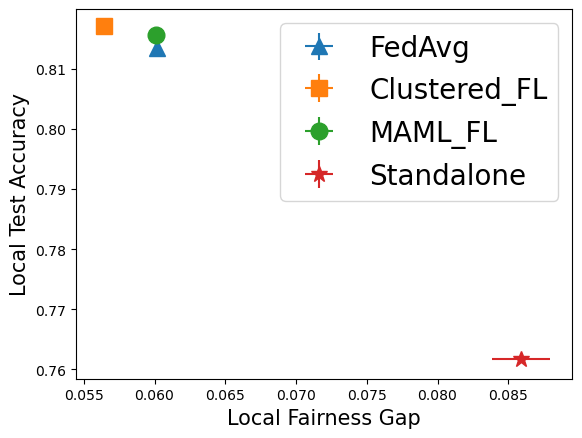}
	}
 \subfigure[Income-sex: \texttt{EO}]{
 \includegraphics[width=0.3\textwidth]{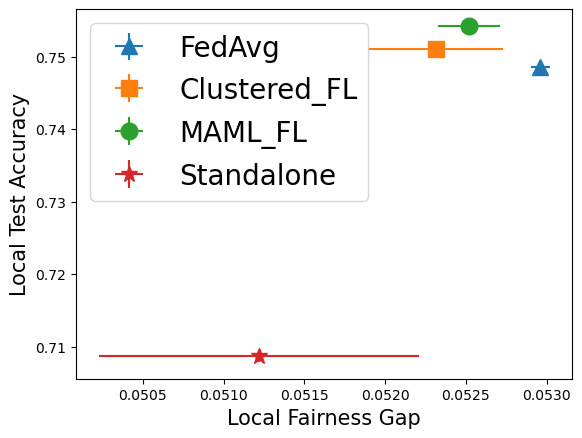}
	}
 \subfigure[Employment-sex: \texttt{EqOp}]{
 \includegraphics[width=0.3\textwidth]{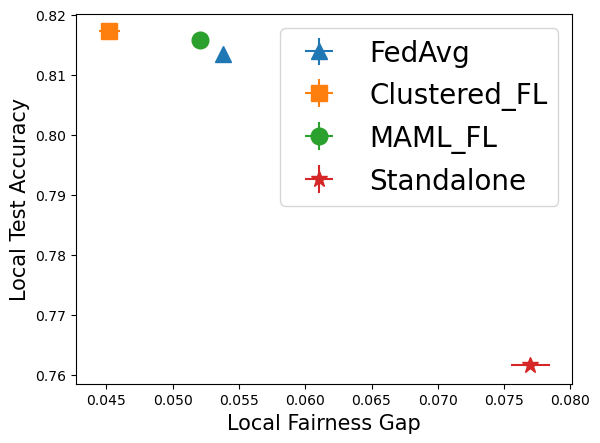}
	}
  \vspace{-0.1in}
	\caption{Personalization could also improve other fairness notions}
	 \label{fig:other_fairness_notions}
\end{figure}

In addition to the \texttt{SP} fairness investigated in Section~\ref{sec:numerical}, we also study the impact of personalization techniques on other types of fairness notions such as \texttt{EO} and \texttt{EqOp}. From Fig.~\ref{fig:other_fairness_notions}, we find that the introduction of personalization techniques can enhance other types of fairness due to the computational advantages of collaboration. However, compared to the improvement of \texttt{SP} and \texttt{EqOp} fairness, the local \texttt{EO} fairness improvement is less significant because the \texttt{EO} matches both the true and false positive rates across two protected groups, rendering it a more stringent criterion.

\subsection{Additional experiments on other dataset and tasks}\label{app:numerical_other_data}

\begin{figure}[ht]
\vspace{-0.1in}
	\centering
	\subfigure[Health-sex: \texttt{SP}]{
 \includegraphics[width=0.3\textwidth]{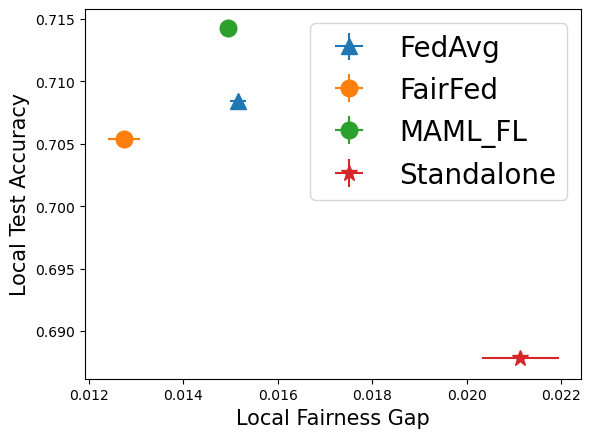}
 \label{Fig:health_sex_SP}
	}
 \subfigure[Income-race: \texttt{SP}]{
 \includegraphics[width=0.3\textwidth]{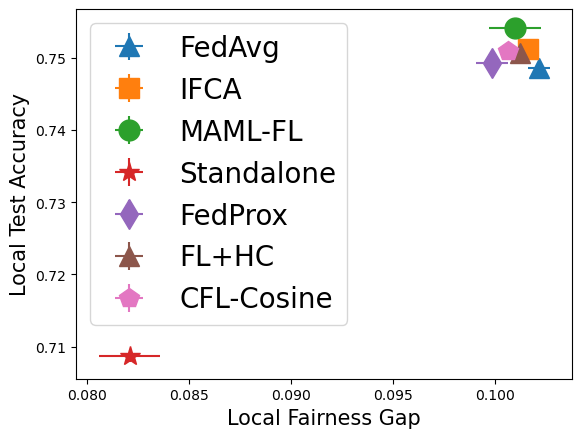}
 \label{Fig:income_race_SP}
	}
  \vspace{-0.1in}
	\caption{Additional experiments on other datasets with different protected attributes}
 \label{fig:additional_experiments_on_other_datasets}
\end{figure}

In addition to the ACSEmployment (sex, race) and ACSIncome (sex) experimental results presented in Section~\ref{sec:numerical}, we conducted additional experiments to explore the impact of \texttt{SP} fairness using new datasets, as illustrated in Fig.~\ref{fig:additional_experiments_on_other_datasets}. Examining the ACSHealth data with sex as the protected attribute, we can see from Fig.~\ref{fig:sex_inc_full} that the fractions of samples from each group across all states are similar to that of the ACSEmployment dataset shown in Fig.~\ref{fig:race_emp_full}, resulting in a similar performance. From Fig.~\ref{fig:additional_experiments_on_other_datasets}, we can see that personalization techniques can improve local fairness as an unintended benefit, similar to the observations from Section~\ref{sec:numerical}.

\begin{figure}[ht]
\vspace{-0.1in}
	\centering
	\subfigure[\texttt{Adult} (Sex): \texttt{SP}]{
 \includegraphics[width=0.22\textwidth]{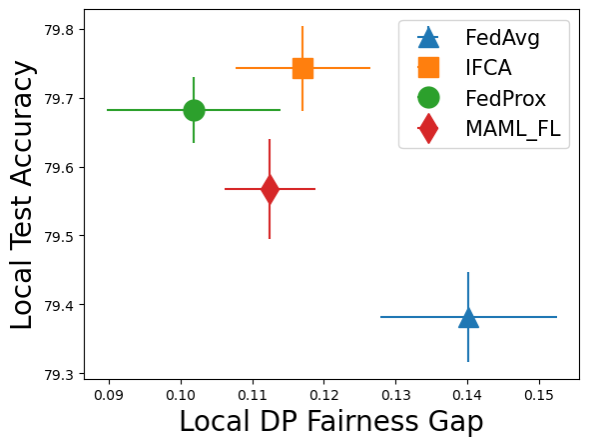}
 \label{Fig:adult-sex-sp-even}
	}
 \subfigure[\texttt{Adult} (Sex): \texttt{EqOp}]{
 \includegraphics[width=0.22\textwidth]{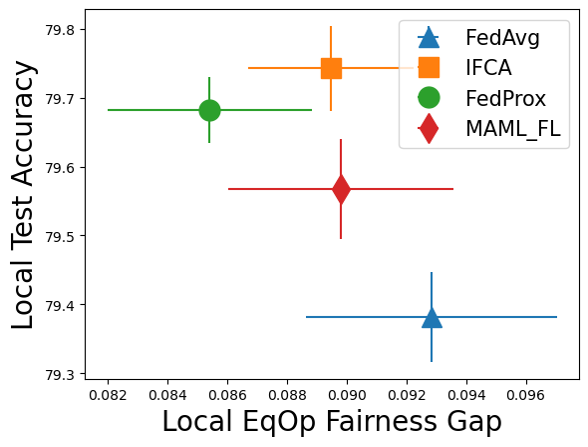}
 \label{Fig:adult-sex-eqop-even}
	}
 \subfigure[\texttt{Adult} (Race): \texttt{SP}]{
 \includegraphics[width=0.22\textwidth]{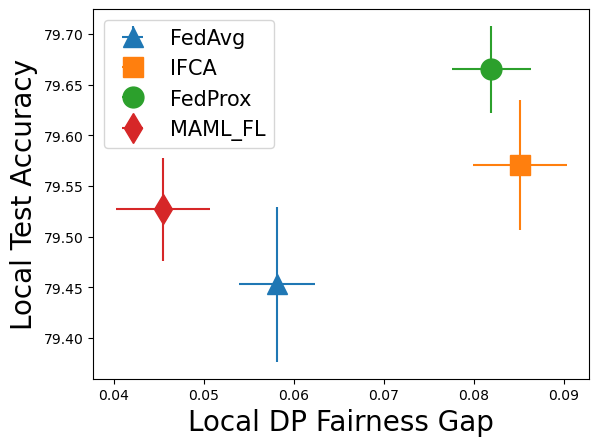}
 \label{Fig:adult-race-sp-even}
	}
 \subfigure[\texttt{Adult} (Race): \texttt{EqOp}]{
 \includegraphics[width=0.22\textwidth]{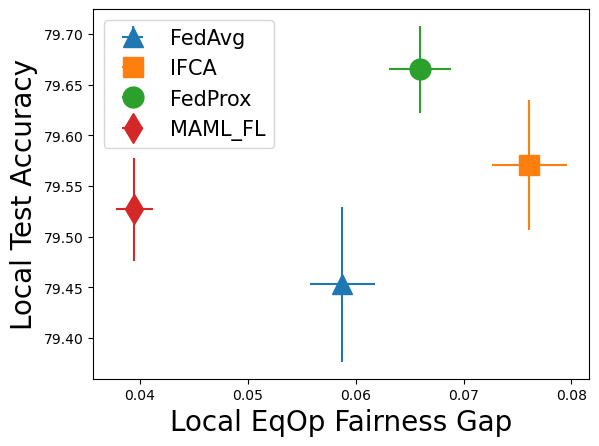}
 \label{Fig:adult-race-eqop-even}
	}
  \vspace{-0.1in}
	\caption{Additional experiments on \texttt{Adult} datasets with samples randomly and evenly distributed}
\label{fig:additional_experiments_on_adult_dataset_even}
\end{figure}

In the \texttt{Adult} dataset, where we randomly and evenly sample data into 5 clients. We could observe that the results are consistent with our findings in Section~\ref{sec:numerical}. That is, when groups are balanced (with sex as the protected attribute), the personalization could also improve the fairness as unintended benefit. However, when groups are unbalanced due to more White samples, the clustered FL algorithms have worse local fairness performance compared to \texttt{FedAvg}, but the \texttt{MAML-FL} algorithm could have a better performance.    

\subsection{Additional experiments on  \texttt{EqOp} and \texttt{EO} fairness}\label{app_experiment_eqop_retiring}
In Section~\ref{sec:new-algorithm}, we compare \texttt{SP} fairness between two algorithms: \texttt{ICFA} and \texttt{Fair-FCA}. Here, we also compare the \texttt{EqOp} and \texttt{EO} fairness between them. The observations from Table~\ref{table_manipulated_data} are also consistent with those from Fig~\ref{fig:experiments_new_mix_dp}, meaning that the \texttt{Fair-FCA} algorithm enables us to establish a better fairness-accuracy tradeoff (a drop in accuracy in return for improved fairness) compared to the \texttt{IFCA} algorithm.  
\begin{table}[ht]
\caption{\texttt{Fair-FCA} with different datasets and protected attributes}
\label{table_manipulated_data}
\begin{center}
\begin{tabular}{c|c|cc|cc}
\multicolumn{1}{c}{\bf Dataset}  &\multicolumn{1}{c}{\bf Algorithm}  &\multirow{1}{*}{\bf \texttt{EqOp}} &\multicolumn{1}{c}{ \bf Acc. (\texttt{EqOp})} &\multirow{1}{*}{\bf \texttt{EO}} &\multicolumn{1}{c}{ \bf Acc. (\texttt{EO})}\\
\hline \\
\multirow{2}{*}{Employment-Race} 
    &\texttt{IFCA}  & 0.07764 &0.8188 & 0.09319 &0.8188 \\
   &\texttt{Fair-FCA}   & 0.07029  &0.8151 & 0.08946 &0.8151 \\
\hline \\
\multirow{2}{*}{Employment-Sex}   &\texttt{IFCA}  & 0.04521 &0.8188 & 0.05808 &0.8188 \\
   &\texttt{Fair-FCA}   & 0.04183 &0.8157  & 0.05655  &0.8151\\  
\hline \\
\multirow{2}{*}{Income-Sex}   &\texttt{IFCA}  & 0.05029 &0.7511 & 0.05231 &0.7511 \\
   &\texttt{Fair-FCA}   & 0.04932 &0.7491 & 0.05161 &0.7489
\end{tabular}
\end{center}
\vspace{-0.2in}
\end{table}

\subsection{Details of setup on synthetic experiment}\label{app_experiment_synthetic_details}

\begin{table}[ht]
\caption{Data distributions over 8 clients}
\begin{center}
\begin{tabular}{c|c|c|c|c}
\multicolumn{1}{c}{\bf Client ID}  &\multicolumn{1}{c}{\bf $f^1_1$}  &\multicolumn{1}{c}{\bf $f^0_1$} &\multicolumn{1}{c}{\bf $f^1_0$} &\multicolumn{1}{c}{\bf $f^0_0$}\\
\hline \\
\multirow{1}{*}{1} 
    &$N(8,1)$  & $N(6,1)$ &$N(8,1)$ &$N(6,1)$  \\
\hline \\
\multirow{1}{*}{2} 
    &$N(12,1)$  & $N(8,1)$ &$N(11,1)$ &$N(7,1)$  \\
\hline \\
\multirow{1}{*}{3} 
    &$N(7.5,1)$  & $N(5.5,1)$ &$N(7.5,1)$ &$N(5.5,1)$  \\
\hline \\
\multirow{1}{*}{4} 
    &$N(12,1)$  & $N(9,1)$ &$N(12,1)$ &$N(9,1)$  \\
\hline \\
\multirow{1}{*}{5} 
    &$N(12,1)$  & $N(8,1)$ &$N(11,1)$ &$N(7,1)$  \\
\hline \\
\multirow{1}{*}{6} 
    &$N(11.5,1)$  & $N(8.5,1)$ &$N(11.5,1)$ &$N(8.5,1)$  \\
\hline \\
\multirow{1}{*}{7} 
    &$N(11,1)$  & $N(8,1)$ &$N(11,1)$ &$N(8,1)$  \\
\hline \\
\multirow{1}{*}{8} 
    &$N(10.5,1)$  & $N(7.5,1)$ &$N(10.5,1)$ &$N(7.5,1)$  \\
\end{tabular}
\end{center}
\end{table}

According to the data distribution information, we can see that clients 2,4,5,6,7,8 have similar data distributions compared to clients 1,3. Also, we can find that clients 1,3,4,6,7,8 share identical data distribution across the two groups. We generate 1200 samples from each distribution and apply a logistic regression classifier for binary classification tasks. We report our experiment results for an average of 5 runs. When $\gamma = 1$, \texttt{Fair-FCA} prioritizes accuracy; by design, this is attained by grouping the 6 clients having similar data distributions together (\{1,3\} and \{2,4,5,6,7,8\}). Similarly, when $\gamma = 0$, \texttt{Fair-FCA} focuses only on \texttt{SP} fairness, this time clustering clients that have identical distributions on the two protected groups together (\{2,5\} and \{1,3,4,6,7,8\}). Lastly, by setting $\gamma \in (0,1)$, we can effectively account for both accuracy and \texttt{SP} fairness when clustering: when $\gamma = 0.3$, the clusters are \{2,4,5\} and \{1,3,6,7,8\}; when $\gamma = 0.5$, the clusters are \{2,4,5,6\} and \{1,3,7,8\}; and when $\gamma = 0.8$, the clusters are \{2,4,5,6,7\} and \{1,3,8\}.

\subsection{Additional experiments using original Retiring Adult dataset without feature scaling}\label{app_experiment_original_retiring}

\begin{table}[ht]
\caption{Algorithm performance comparisons using original Retiring adult dataset}
\label{table_original_data}
\begin{center}
\begin{tabular}{c|c|cc|cc}
\multicolumn{1}{c}{\bf Dataset}  &\multicolumn{1}{c}{\bf Algorithm}  &\multirow{1}{*}{\bf \texttt{SP}} &\multicolumn{1}{c}{ \bf Acc. (\texttt{SP})} &\multirow{1}{*}{\bf \texttt{EqOp}} &\multicolumn{1}{c}{\bf Acc. (\texttt{EqOp})}\\
\hline \\
\multirow{2}{*}{Employment-Sex} 
    &\texttt{IFCA}  & 0.03667 &0.8229 &0.04698 & 0.8229  \\
   &\texttt{Fair-FCA}   & 0.03594 $\downarrow$ &0.8223 $\downarrow$ &0.04633 $\downarrow$ & 0.8224 $\downarrow$ \\
\hline \\
\multirow{2}{*}{Employment-Race}   &\texttt{IFCA}  & 0.07257 &0.8229 &0.07315 & 0.8229  \\
   &\texttt{Fair-FCA}   & 0.07219 $\downarrow$ &0.8224 $\downarrow$&0.06527 $\downarrow$ & 0.8226 $\downarrow$\\  
\hline \\
\multirow{2}{*}{Income-Sex}   &\texttt{IFCA} & 0.08355 &0.7481 &0.04773 & 0.7481  \\
   &\texttt{Fair-FCA}   & 0.08227 $\downarrow$ &0.7469 $\downarrow$ &0.04767 $\downarrow$ & 0.7469 $\downarrow$ \\  
\hline \\
\multirow{2}{*}{Income-Race}   &\texttt{IFCA}  & 0.1012 &0.7481 &0.1100 & 0.7481  \\
   &\texttt{Fair-FCA}   & 0.1011 $\downarrow$ &0.7468 $\downarrow$&0.1086 $\downarrow$ & 0.7466 $\downarrow$ \\  
\end{tabular}
\end{center}
\end{table}

We can see from Table~\ref{table_original_data} that compared to the \texttt{IFCA} algorithm, our \texttt{Fair-FCA} algorithm is experiencing a degradation in accuracy but an improved fairness, meaning an accuracy-fairness tradeoff. These observations are also consistent with our findings when using the Retiring Adult dataset with feature scaling in Section~\ref{sec:new-algorithm}.

\section{Experiments on Synthetic Data}\label{app:numerical-synthetic}

To further validate our propositions, we conduct the following numerical experiments. In the experiments detailed in \ref{app_experiment_equal_distance_group_label}, the setup is the most restrictive, with equalized distance, balanced group rates, and equalized label rates. In subsequent experiments, we relax one factor at a time. Finally, in the experiments described in \ref{app_experiment}, all these assumptions are removed.

\subsection{Experiments under Gaussian distribution with equalized distance, balanced group rate, and equalized label rate}\label{app_experiment_equal_distance_group_label}

\textbf{Numerical illustration.} We now conduct numerical experiments to illustrate the findings in Prop.~\ref{prop3}-~\ref{prop1}. We drop the cluster notation $c$ whenever it is clear from the context. 
The results are presented in Tables~\ref{table1} and~\ref{table2}. We proceed as follows: 10000 random samples in cluster $\mathcal{C}_\alpha$ are drawn from Gaussian distribution for each group $g\in \{a,b\}$ with mean $\mu^{y,\mathcal{C}_\alpha}_g$ and standard deviation $\sigma$. The number of qualified ($y=1$) and unqualified ($y=0$) samples in each group is proportional to the label participation rate $\alpha^{y,\mathcal{C}_\alpha}_g$. Since samples were generated in a consistent manner across different parameter settings, we assumed an optimal decision threshold $\theta^{*}_\beta=8$ for cluster $\mathcal{C}_\beta$, obtained according to the distribution information: $(f^1_1,f^0_1,f^1_0,f^0_0,\sigma) = (10,7,9,6,1)$ with equalized group rate $r_g=0.5, \forall g$ and label participation rate $\alpha^{y,\mathcal{C}_\beta}_g=0.5, \forall g,y$. In Table~\ref{table1}, we consider the scenario where $\alpha^{y,\mathcal{C}_\alpha}_g = 0.5 \hspace{0.05in} \forall g, y$. In contrast, different values of $\alpha^{y,\mathcal{C}_\alpha}_g$ are applied in Table~\ref{table2}. Both results in Table~\ref{table1} and ~\ref{table2} consider an equalized group rate such that $r_a = r_b$ and an equalized distance between mean estimates. 

From Table~\ref{table1}, we can find that it offers crucial insights into the conditions required for Proposition~\ref{prop1} (\texttt{SP}) to hold. For fixed mean estimates $\mu^y_g$ (rows 1-2), we observe that smaller values of $\sigma$ are preferable to satisfy the specified conditions. Similarly, for fixed $\sigma$ (row 1, 3 and row 2, 4), larger differences between $\mu^1_g$ and $\mu^0_g$ are advantageous in fulfilling the conditions. This observation becomes intuitive at the extreme cases where samples are linearly separable with small $\sigma$ or large distance between $\mu^1_g$ and $\mu^0_g$. Therefore, the optimal decision threshold $\theta^{*}_\alpha$ could achieve a perfect classification as well as perfect fairness. Hence, the FedAvg solution $\theta^{*}_G$ deviated from the optimal solution will lead to worse performance in both accuracy and fairness. {We could also observe that for the \texttt{EqOp} fairness, under an equalized label rate, the FedAvg solutions consistently make the cluster $\mathcal{C}_\alpha$ unfairer, which is consistent with our findings in Prop.~\ref{prop3}.}

\begin{table}[ht]\scriptsize
\caption{Cluster $\mathcal{C}_\alpha$ fairness performance with equalized distance, group rate and label rate}
\label{table1}
\vspace{-0.15in}
\begin{center}
\begin{tabular}{ccccc|ccc}
\multicolumn{1}{c}{\bf Distribution}  &\multicolumn{1}{c}{\bf Condition}  &\multirow{2}{*}{\bf $\Delta^\alpha_{\texttt{SP}}(\theta^{*}_\alpha)$} &\multicolumn{2}{c}{\multirow{1}{*}{\bf $\Delta^\alpha_{\texttt{SP}}(\theta^{*}_G)$}} &\multirow{2}{*}{\bf $\Delta^\alpha_{\texttt{EqOp}}(\theta^{*}_\alpha)$} &\multicolumn{2}{c}{\multirow{1}{*}{\bf $\Delta^\alpha_{\texttt{EqOp}}(\theta^{*}_G)$}}\\
\multicolumn{1}{c}{($\mu^1_a$, $\mu^0_a$, $\mu^1_b$, 
$\mu^0_b$, $\sigma$)}&\multicolumn{1}{c}{\bf (\texttt{SP})} &\multicolumn{1}{c}{} &\multicolumn{1}{c}{$p=\frac{2}{3}$} &\multicolumn{1}{c}{$p=\frac{1}{2}$} &\multicolumn{1}{c}{} &\multicolumn{1}{c}{$p=\frac{2}{3}$} &\multicolumn{1}{c}{$p=\frac{1}{2}$} \\
\hline \vspace{-0.1in} \\
(7, 4, 6, 3, 1)   &Yes  & 0.1359 &0.1814 $\uparrow$ &0.1945 
 $\uparrow$ & 0.1359 &0.3413 $\uparrow$ &0.3829 
 $\uparrow$ \\
(7, 4, 6, 3, 2)   &No   & 0.1499 &0.1417 $\downarrow$ &0.1315 $\downarrow$ & 0.1499 &0.1915 $\uparrow$ &0.1974 
 $\uparrow$ \\
(7, 5, 6, 4, 1)   &No   & 0.2417 &0.2297 $\downarrow$ &0.2046 $\downarrow$ & 0.2417 &0.3781 $\uparrow$ &0.3721 
 $\uparrow$\\
(8, 3, 6, 1, 2)   &Yes  & 0.1866 &0.1968 $\uparrow$ &0.2033 $\uparrow$ & 0.1866 &0.3121 $\uparrow$ &0.3590 
 $\uparrow$
\end{tabular}
\end{center}
\vspace{-0.1in}
\end{table}

Table~\ref{table2} reveals insights regarding the influence of label distribution $\alpha^y_g$ on \texttt{SP} and \texttt{EqOp} fairness performance. Specifically, when the majority of samples in both groups are labeled as 1 (rows 1-2), the optimal decision threshold ($\theta^{*}_\alpha$) shifts leftward compared to the balanced scenario. However, with Lemma~\ref{lemma1:relative_location}, the FedAvg solution $\theta^{*}_G$ will be greater than $\theta^{*}_\alpha$. Therefore, we can find that $\theta$ will have even larger fairness gap when it is shifted to the right. Another intriguing observation is that in cases where the majority of samples have different labels (row 3), the FedAvg solution ($\theta^{*}_G$) yields worse fairness performance when $p = 2/3$ or $1/2$ but not when $p=1/3$ (0.1720 $\downarrow$) or $1/4$ (0.1391 $\downarrow$). This indicates the weight $p$ plays a significant role in shaping the overall cluster-wise average fairness performance, especially when assessing the overall cluster-wise average fairness performance. 

\begin{table}[ht]\small
\caption{Cluster $\mathcal{C}_\alpha$ fairness performance with equalized distance and group rate}
\label{table2}
\begin{center}
\begin{tabular}{cccccc}
\multicolumn{1}{c}{\bf Distribution} &\multicolumn{1}{c}{\bf Label rate}  &\multicolumn{1}{c}{\bf Condition}  &\multirow{2}{*}{\bf $\Delta^\alpha_{\texttt{SP}}(\theta^{*}_\alpha)$} &\multicolumn{2}{c}{\multirow{1}{*}{\bf $\Delta^\alpha_{\texttt{SP}}(\theta^{*}_G)$}}\\
\multicolumn{1}{c}{($\mu^1_a$, $\mu^0_a$, $\mu^1_b$, 
$\mu^0_b$, $\sigma$)} &\multicolumn{1}{c}{($\alpha^1_a$, $\alpha^0_a$, $\alpha^1_b$, 
$\alpha^0_b$)} &\multicolumn{1}{c}{\bf (\texttt{SP})} &\multicolumn{1}{c}{} &\multicolumn{1}{c}{$p=\frac{2}{3}$} &\multicolumn{1}{c}{$p=\frac{1}{2}$} \\
\hline \vspace{-0.1in} \\
\multirow{4}{*}{(7, 4, 6, 3, 1)} &(0.7, 0.3, 0.6, 0.4)  &Yes  & 0.2062 &0.2146 $\uparrow$ &0.2463 $\uparrow$\\
 &(0.6, 0.4, 0.7, 0.3)  &Yes  & 0.0453 &0.0514 $\uparrow$ &0.0813 $\uparrow$\\
 &(0.7, 0.3, 0.4, 0.6)  &Yes  & 0.3797 &0.3858 $\uparrow$ &0.3926 $\uparrow$\\
 &(0.6, 0.4, 0.3, 0.7)  &No   & 0.3797 &0.3748 $\downarrow$ &0.2804 $\downarrow$\\
\hline\\
&\multicolumn{1}{c}{} &\multicolumn{1}{c}{\bf (\texttt{EqOp})}  &\multicolumn{1}{c}{$\Delta^\alpha_{\texttt{EqOp}}(\theta^{*}_\alpha)$} &\multicolumn{2}{c}{$\Delta^\alpha_{\texttt{EqOp}}(\theta^{*}_G)$}\\
\multirow{4}{*}{(7, 4, 6, 3, 2)} &(0.7, 0.3, 0.6, 0.4)  &Yes  & 0.0998 &0.1807 $\uparrow$ &0.1923 $\uparrow$\\
&(0.6, 0.4, 0.7, 0.3)  &Yes  & 0.0975 &0.1198 $\uparrow$ &0.1796 $\uparrow$\\
&(0.1, 0.9, 0.5, 0.5)  &No   & 0.1965 &0.1650 $\downarrow$ &0.1574 $\downarrow$\\
&(0.3, 0.7, 0.2, 0.8)  &No   & 0.1974 &0.1645 $\downarrow$ &0.1574 $\downarrow$\\
    
\end{tabular}
\end{center}
\vspace{-0.2in}
\end{table}

{Since we assume clients within the same cluster are identical, and the local fairness performance for an algorithm can be computed as a weighted sum of the local fairness performance from each cluster, the cluster-wise \emph{average} local fairness gap under different models' optimal solution $\theta$ could be calculated as $\Delta_{f}(\theta) = p \Delta^\alpha_{f} + (1-p) \Delta^\beta_{f};  f\in\{\texttt{SP}, \texttt{EqOp}, \texttt{EO}\}$, where $p$ is the fraction of clients belonging to cluster $\mathcal{C}_\alpha$.}

In Table~\ref{table3} and ~\ref{table10}, we delve into different notions of cluster-wise average fairness gap achieved with different decision thresholds (optimal clustered FL solutions $\theta^{*}_C$ and FedAvg solutions $\theta^{*}_G$). In the following experiment, we keep the parameters in cluster $\mathcal{C}_\beta$ as constants while varying those in cluster $\mathcal{C}_\alpha$ to assess its impact on the corresponding fairness. From the results in Table~\ref{table3} and ~\ref{table10}, we can find that when both conditions are not satisfied (rows 5-6), there is a cluster size weight $p$ such that the FedAvg solutions would lead to better fairness performance for each cluster, consequently yielding a lower cluster-wise average fairness gap. However, when only one cluster satisfies the condition, meaning that there is a $p$ such that the FedAvg solutions would only make one cluster unfairer (rows 1-2 in Table~\ref{table3}), we could see that a relatively small $p$ would let the clustered FL solutions yield a better fairness performance because $\theta^*_G$ will move to the cluster with a smaller value of $p$ to account for the cluster size imbalance. Nevertheless, when $p$ is large, the FedAvg solutions will again have superior fairness performance than the clustered FL solutions, similar to the results in rows 3-4 in Table~\ref{table3} and ~\ref{table10}. Essentially, for each cluster $c$, there exists a range $(p^c_{low}, p^c_{high})$ such that, within this range, FedAvg solutions result in worse fairness performance compared to clustered FL solutions. Consequently, for any $p \in \cap_c (p^c_{low}, p^c_{high})$, clustered FL solutions yield a superior cluster-wise average fairness performance relative to FedAvg solutions. 

\begin{table}[ht]\small
\caption{Cluster-wise average \texttt{SP} fairness performance with equalized distance}
\label{table3}
\begin{center}
\begin{tabular}{cccccc}
\multicolumn{1}{c}{\bf Distribution} &\multicolumn{1}{c}{\bf Label rate}  &\multicolumn{1}{c}{\bf Condition} &\multirow{3}{*}{$p$} &\multirow{3}{*}{\bf $\Delta_{\texttt{SP}}(\theta^{*}_C)$} &\multicolumn{1}{c}{\multirow{3}{*}{\bf $\Delta_{\texttt{SP}}(\theta^{*}_G)$}}\\
\multicolumn{1}{c}{$\mathcal{C}_\alpha:$($\mu^1_a$, $\mu^0_a$, $\mu^1_b$, 
$\mu^0_b$, $\sigma$)} &\multicolumn{1}{c}{($\alpha^1_a$, $\alpha^0_a$, $\alpha^1_b$, 
$\alpha^0_b$)} &\multicolumn{1}{c}{} &\multicolumn{1}{c}{} &\multicolumn{1}{c}{} &\multicolumn{1}{c}{} \\
\multicolumn{1}{c}{$\mathcal{C}_\beta:$($\mu^1_a$, $\mu^0_a$, $\mu^1_b$, 
$\mu^0_b$, $\sigma$)} &\multicolumn{1}{c}{($\alpha^1_a$, $\alpha^0_a$, $\alpha^1_b$, 
$\alpha^0_b$)} &\multicolumn{1}{c}{\bf satisfied} &\multicolumn{1}{c}{} &\multicolumn{1}{c}{} &\multicolumn{1}{c}{} \\

\hline \vspace{-0.1in}\\
(7, 4, 6, 3, 2) & (0.5, 0.5, 0.5, 0.5)  &No   & 4/5 &0.147 &0.145 $\downarrow$\\
(10, 7, 9, 6, 1) &(0.5, 0.5, 0.5, 0.5)  &Yes  & 1/3  &0.141 &0.160 $\uparrow$\\
\hline \vspace{-0.1in} \\
(7, 4, 6, 3, 2)  &(0.8, 0.2, 0.7, 0.3)  &Yes  & 3/4 &0.139 &0.107 $\downarrow$\\
(10, 7, 9, 6, 1) &(0.5, 0.5, 0.5, 0.5)  &Yes  & 1/2 &0.138 &0.178 $\uparrow$\\
\hline \vspace{-0.1in} \\
(7, 4, 6, 3, 2)  &(0.5, 0.5, 0.5, 0.5)  &No   & 1/3 &0.303 &0.283 $\downarrow$\\
(10, 7, 9, 6, 1) &(0.7, 0.3, 0.4, 0.6)  &No   & 2/3 &0.227 &0.200 $\downarrow$\\
\end{tabular}
\end{center}
\vspace{-0.1in}
\end{table}

\begin{table}[ht]\small
\caption{Cluster-wise average \texttt{EqOp} fairness performance with equalized distance}
\label{table10}
\begin{center}
\begin{tabular}{cccccc}
\multicolumn{1}{c}{\bf Distribution} &\multicolumn{1}{c}{\bf Label rate}  &\multicolumn{1}{c}{\bf Condition} &\multirow{3}{*}{$p$} &\multirow{3}{*}{\bf $\Delta_{\texttt{EqOp}}(\theta^{*}_C)$} &\multicolumn{1}{c}{\multirow{3}{*}{\bf $\Delta_{\texttt{EqOp}}(\theta^{*}_G)$}}\\
\multicolumn{1}{c}{$\mathcal{C}_\alpha:$($\mu^1_a$, $\mu^0_a$, $\mu^1_b$, 
$\mu^0_b$, $\sigma$)} &\multicolumn{1}{c}{($\alpha^1_a$, $\alpha^0_a$, $\alpha^1_b$, 
$\alpha^0_b$)} &\multicolumn{1}{c}{} &\multicolumn{1}{c}{} &\multicolumn{1}{c}{} &\multicolumn{1}{c}{} \\
\multicolumn{1}{c}{$\mathcal{C}_\beta:$($\mu^1_a$, $\mu^0_a$, $\mu^1_b$, 
$\mu^0_b$, $\sigma$)} &\multicolumn{1}{c}{($\alpha^1_a$, $\alpha^0_a$, $\alpha^1_b$, 
$\alpha^0_b$)} &\multicolumn{1}{c}{\bf satisfied} &\multicolumn{1}{c}{} &\multicolumn{1}{c}{} &\multicolumn{1}{c}{} \\

\hline \vspace{-0.1in}\\
(7, 4, 6, 3, 2) & (0.3, 0.7, 0.2, 0.8)  &No   & 1/3 &0.156 &0.133 $\downarrow$\\
(10, 7, 9, 6, 1) &(0.5, 0.5, 0.5, 0.5)  &No  & 2/3  &0.177 &0.139 $\downarrow$\\
\hline \vspace{-0.1in} \\
(7, 4, 6, 3, 2)  &(0.8, 0.2, 0.7, 0.3)  &Yes  & 3/4 &0.082 &0.050 $\downarrow$\\
(10, 7, 9, 6, 1) &(0.5, 0.5, 0.5, 0.5)  &No  & 1/2 &0.100 &0.109 $\uparrow$\\
\hline \vspace{-0.1in} \\
(7, 4, 6, 3, 2)  &(0.3, 0.7, 0.2, 0.8)  &No   & 1/3 &0.224 &0.187 $\downarrow$\\
(10, 7, 9, 6, 1) &(0.3, 0.7, 0.2, 0.8)  &No   & 2/3 &0.211 &0.149 $\downarrow$\\
\end{tabular}
\end{center}
\vspace{-0.2in}
\end{table}

\subsection{Experiments under Gaussian distribution with equalized distance and balanced label rate}\label{app_experiment_equal_distance_label}

Compared to the experiments focused on an all balanced setting in Table~\ref{table1}, the following experiments relax the group rates setting in the cluster $\mathcal{C}_\alpha$, while we keep other settings (i.e., balanced label rate and equalized distance) and data information for $\mathcal{C}_\beta$ unchanged. 

\begin{table}[ht]\scriptsize
\caption{Cluster $\mathcal{C}_\alpha$ fairness performance under Gaussian distribution with equalized distance and label rate, but not group rate}
\label{table5}
\begin{center}
\begin{tabular}{ccccc|ccc}
\multicolumn{1}{c}{\bf Distribution}  &\multicolumn{1}{c}{\bf Group rate}  &\multirow{2}{*}{\bf $\Delta^\alpha_{\texttt{SP}}(\theta^{*}_\alpha)$} &\multicolumn{2}{c}{\multirow{1}{*}{\bf $\Delta^\alpha_{\texttt{SP}}(\theta^{*}_G)$}} &\multirow{2}{*}{\bf $\Delta^\alpha_{\texttt{EqOp}}(\theta^{*}_\alpha)$} &\multicolumn{2}{c}{\multirow{1}{*}{\bf $\Delta^\alpha_{\texttt{EqOp}}(\theta^{*}_G)$}}\\
\multicolumn{1}{c}{($\mu^1_a$, $\mu^0_a$, $\mu^1_b$, 
$\mu^0_b$, $\sigma$)}&\multicolumn{1}{c}{$(r_a, r_b)$} &\multicolumn{1}{c}{} &\multicolumn{1}{c}{$p=\frac{2}{3}$} &\multicolumn{1}{c}{$p=\frac{1}{2}$} &\multicolumn{1}{c}{} &\multicolumn{1}{c}{$p=\frac{2}{3}$} &\multicolumn{1}{c}{$p=\frac{1}{2}$}\\
\hline \\
\multirow{5}{*}{(7, 4, 6, 3, 1)}   &(0.5, 0.5)  & 0.1359 &0.1814 $\uparrow$ &0.1945 
 $\uparrow$ & 0.1359 &0.3413 $\uparrow$ &0.3829 
 $\uparrow$ \\
   &(0.7, 0.3)   & 0.1388 &0.1558 $\uparrow$ &0.1941 $\uparrow$ & 0.1780 &0.2594 $\uparrow$ &0.3828 
 $\uparrow$\\
   &(0.9, 0.1)   & 0.1465 &0.1702 $\uparrow$ &0.1941 $\uparrow$ & 0.2217 &0.3076 $\uparrow$ &0.3828 
 $\uparrow$\\
   &(0.3, 0.7)  & 0.1388 &0.1359 $\downarrow$ &0.1558 $\uparrow$ & 0.0996 &0.1359 $\uparrow$ &0.2594 $\uparrow$\\
   &(0.4, 0.6)  & 0.1367 &0.1372 $\uparrow$ &0.1931 $\uparrow$ & 0.1161 &0.1634 $\uparrow$ &0.3759 $\uparrow$
\end{tabular}
\end{center}
\end{table}

From Table~\ref{table5}, we can see that the changes in the group rate do not affect the fairness performance comparison. There exists a cluster size weight $p$ such that the FedAvg solutions would lead to worse \texttt{SP} and \texttt{EqOp} fairness performance compared to the clustered FL solutions. This observation is also consistent with our findings in the Proposition~\ref{prop3} and ~\ref{prop1}.

\subsection{Experiments under Gaussian distribution with equalized distance}\label{app_experiment_equal_distance}

Similar to experiments in \ref{app_experiment_equal_distance_label}, we further relax balanced label rate setting in the following experiments, while we keep other settings (i.e., equalized distance) and data information for $\mathcal{C}_\beta$ unchanged. 

\begin{table}[t]\small
\caption{Cluster $\mathcal{C}_\alpha$ \texttt{SP} fairness performance under Gaussian distribution with equalized distance, but not label rate and group rate}
\label{table6}
\begin{center}
\begin{tabular}{cccccc}
\multicolumn{1}{c}{\bf Distribution} &\multicolumn{1}{c}{\bf Label rate}  &\multicolumn{1}{c}{\bf Group rate}  &\multirow{2}{*}{\bf $\Delta^\alpha_{\texttt{SP}}(\theta^{*}_\alpha)$} &\multicolumn{2}{c}{\multirow{1}{*}{\bf $\Delta^\alpha_{\texttt{SP}}(\theta^{*}_G)$}}\\
\multicolumn{1}{c}{($\mu^1_a$, $\mu^0_a$, $\mu^1_b$, 
$\mu^0_b$, $\sigma$)} &\multicolumn{1}{c}{($\alpha^1_a$, $\alpha^0_a$, $\alpha^1_b$, 
$\alpha^0_b$)} &\multicolumn{1}{c}{$(r_a, r_b)$} &\multicolumn{1}{c}{} &\multicolumn{1}{c}{$p=\frac{2}{3}$} &\multicolumn{1}{c}{$p=\frac{1}{2}$} \\
\hline \\
\multirow{16}{*}{(7, 4, 6, 3, 1)} &\multirow{3}{*}{(0.7, 0.3, 0.6, 0.4)}  &(0.5, 0.5)  & 0.2062 &0.2146 $\uparrow$ &0.2463 $\uparrow$\\
 &  &(0.3, 0.7)  & 0.2024 &0.2056 $\uparrow$ &0.2167 $\uparrow$\\
 &  &(0.7, 0.3)  & 0.2136 &0.2309 $\uparrow$ &0.2793 $\downarrow$\\
 \cline{2-6}\\
 
 &\multirow{3}{*}{(0.6, 0.4, 0.7, 0.3)}  &(0.5, 0.5)   & 0.0453 &0.0514 $\uparrow$ &0.0813 $\uparrow$\\
& &(0.3, 0.7)   & 0.0460 &0.0446 $\downarrow$ &0.0482 $\uparrow$\\
& &(0.7, 0.3)   & 0.0535 &0.0751 $\uparrow$ &0.1467 $\uparrow$\\
 \cline{2-6}\\
 &\multirow{3}{*}{(0.7, 0.3, 0.4, 0.6)}  &(0.5, 0.5)   & 0.3797 &0.3858 $\uparrow$ &0.3926 $\uparrow$\\
&  &(0.3, 0.7)   & 0.3780 &0.3819 $\uparrow$ &0.3451 $\downarrow$\\
&  &(0.7, 0.3)   & 0.3821 &0.3899 $\uparrow$ &0.3936 $\uparrow$ \\
\cline{2-6}\\
 &\multirow{4}{*}{(0.4, 0.6, 0.7, 0.3)}  &(0.7, 0.3)  & 0.1005 &0.0662 $\downarrow$ &0.0766 $\downarrow$\\
 &  &(0.9, 0.1)  & 0.0725 &0.0078 $\downarrow$ &0.0868 $\uparrow$\\
 &  &(0.3, 0.7)  & 0.1013 &0.1084 $\uparrow$ &0.1090 $\downarrow$\\
  &  &(0.1, 0.9)  & 0.0767 &0.0860 $\uparrow$ &0.0972 $\uparrow$
\end{tabular}
\end{center}
\vspace{-0.25in}
\end{table}

\begin{table}[t]\small
\caption{Cluster $\mathcal{C}_\alpha$ \texttt{EqOp} fairness performance under Gaussian distribution with equalized distance, but not label rate and group rate}
\label{table8}
\begin{center}
\begin{tabular}{cccccc}
\multicolumn{1}{c}{\bf Distribution} &\multicolumn{1}{c}{\bf Label rate}  &\multicolumn{1}{c}{\bf Group rate}  &\multirow{2}{*}{\bf $\Delta^\alpha_{\texttt{EqOp}}(\theta^{*}_\alpha)$} &\multicolumn{2}{c}{\multirow{1}{*}{\bf $\Delta^\alpha_{\texttt{EqOp}}(\theta^{*}_G)$}}\\
\multicolumn{1}{c}{($\mu^1_a$, $\mu^0_a$, $\mu^1_b$, 
$\mu^0_b$, $\sigma$)} &\multicolumn{1}{c}{($\alpha^1_a$, $\alpha^0_a$, $\alpha^1_b$, 
$\alpha^0_b$)} &\multicolumn{1}{c}{$(r_a, r_b)$} &\multicolumn{1}{c}{} &\multicolumn{1}{c}{$p=\frac{2}{3}$} &\multicolumn{1}{c}{$p=\frac{1}{2}$} \\
\hline \\
\multirow{16}{*}{(7, 4, 6, 3, 2)} &\multirow{3}{*}{(0.7, 0.3, 0.6, 0.4)}  &(0.5, 0.5)  & 0.0998 &0.1222 $\uparrow$ &0.1807 $\uparrow$\\
 &  &(0.3, 0.7)  & 0.0952 &0.1109 $\uparrow$ &0.1784 $\uparrow$\\
 &  &(0.7, 0.3)  & 0.1044 &0.1386 $\uparrow$ &0.1825 $\uparrow$\\
 \cline{2-6}\\
 &\multirow{3}{*}{(0.6, 0.4, 0.7, 0.3)}  &(0.5, 0.5)   & 0.0975 &0.1198 $\uparrow$ &0.1796 $\uparrow$\\
& &(0.3, 0.7)   & 0.0798 &0.0874 $\uparrow$ &0.1799 $\uparrow$\\
& &(0.7, 0.3)   & 0.1180 &0.1957 $\uparrow$ &0.1792 $\uparrow$\\
 \cline{2-6}\\
 &\multirow{3}{*}{(0.1, 0.9, 0.5, 0.5)}  &(0.5, 0.5)   & 0.1965 &0.1650 $\downarrow$ &0.1574 $\downarrow$\\
&  &(0.3, 0.7)   & 0.1751 &0.1742 $\downarrow$ &0.1620 $\downarrow$\\
&  &(0.7, 0.3)   & 0.1869 &0.1569 $\downarrow$ &0.1537 $\downarrow$ \\
 \cline{2-6}\\
 &\multirow{3}{*}{(0.3, 0.7, 0.2, 0.8)}  &(0.5, 0.5)  & 0.1974 &0.1645 $\downarrow$ &0.1574 $\downarrow$\\
 &  &(0.3, 0.7)  & 0.1974 &0.1630 $\downarrow$ &0.1569 $\downarrow$\\
 &  &(0.7, 0.3)  & 0.1973 &0.1660 $\downarrow$ &0.1585 $\downarrow$
\end{tabular}
\end{center}
\vspace{-0.25in}
\end{table}

From Table~\ref{table6}, we can observe that for the \texttt{SP} fairness, when the majority of samples are labeled 1 (rows 1-6), the changes in the group rate do not affect the fairness performance comparison in the cluster $\mathcal{C}_\alpha$. There exists a cluster size weight $p$ such that the FedAvg solution would lead to a worse fairness performance compared to the clustered FL solutions. From Table~\ref{table8}, when the condition $\alpha^1_g \geq \alpha^0_g$ holds, there exists a combination of group rates (rows 1-6) such that the FedAvg solution would lead to a worse \texttt{EqOp} fairness performance. These observations from Table~\ref{table6} and \ref{table8} are also consistent with our findings in the Proposition~\ref{prop3} and ~\ref{prop1}.

\subsection{Additional experiments under Gaussian distribution}\label{app_experiment}

Similar to experiments in \ref{app_experiment_equal_distance_label} and \ref{app_experiment_equal_distance}, we now release all settings we imposed before, while we data information for $\mathcal{C}_\beta$ unchanged. 

\begin{table}[ht]\small
\caption{Cluster $\mathcal{C}_\alpha$ \texttt{SP} fairness performance under Gaussian distribution without equalized distance, label rate and group rate}
\label{table7}
\begin{center}
\begin{tabular}{cccccc}
\multicolumn{1}{c}{\bf Distribution} &\multicolumn{1}{c}{\bf Label rate}  &\multicolumn{1}{c}{\bf Group rate}  &\multirow{2}{*}{\bf $\Delta^\alpha_{\texttt{SP}}(\theta^{*}_\alpha)$} &\multicolumn{2}{c}{\multirow{1}{*}{\bf $\Delta^\alpha_{\texttt{SP}}(\theta^{*}_G)$}}\\
\multicolumn{1}{c}{($\mu^1_a$, $\mu^0_a$, $\mu^1_b$, 
$\mu^0_b$, $\sigma$)} &\multicolumn{1}{c}{($\alpha^1_a$, $\alpha^0_a$, $\alpha^1_b$, 
$\alpha^0_b$)} &\multicolumn{1}{c}{$(r_a, r_b)$} &\multicolumn{1}{c}{} &\multicolumn{1}{c}{$p=\frac{2}{3}$} &\multicolumn{1}{c}{$p=\frac{1}{2}$} \\
\hline \\
\multirow{6}{*}{(7, 4.5, 6, 3, 1)} &\multirow{3}{*}{(0.7, 0.3, 0.6, 0.4)}  &(0.5, 0.5)  & 0.2598 &0.2649 $\uparrow$ &0.2902 $\uparrow$\\
 &  &(0.3, 0.7)  & 0.2589  &0.2593 $\uparrow$ &0.2655 $\uparrow$\\
 &  &(0.7, 0.3)  & 0.2646&0.2781 $\uparrow$ &0.3074 $\uparrow$\\
 \cline{2-6}\\
 &\multirow{3}{*}{(0.7, 0.3, 0.4, 0.6)}  &(0.5, 0.5)   & 0.4263 &0.4220 $\downarrow$ &0.3917 $\downarrow$\\
&  &(0.3, 0.7)   & 0.4288 &0.4248 $\downarrow$&0.3222 $\downarrow$\\
&  &(0.7, 0.3)   & 0.4240 &0.4198 $\downarrow$ &0.3971 $\downarrow$\\
\hline \\
\multirow{6}{*}{(7, 4, 6, 3.5, 1)} &\multirow{3}{*}{(0.7, 0.3, 0.6, 0.4)}  &(0.5, 0.5)  & 0.1871 &0.2046 $\uparrow$ &0.2483 $\uparrow$\\
 &  &(0.3, 0.7)  & 0.1785  &0.1910 $\uparrow$ &0.2167 $\uparrow$\\
 &  &(0.7, 0.3)  & 0.1984&0.2236 $\uparrow$ &0.2784 $\uparrow$\\
 \cline{2-6}\\
 &\multirow{3}{*}{(0.7, 0.3, 0.4, 0.6)}  &(0.5, 0.5)   & 0.3576 &0.3752 $\uparrow$ &0.3882 $\uparrow$\\
&  &(0.3, 0.7)   & 0.3538 &0.3697 $\uparrow$&0.3335 $\downarrow$\\
&  &(0.7, 0.3)   & 0.3620 &0.3798 $\uparrow$ &0.3903 $\uparrow$
\end{tabular}
\end{center}
\vspace{-0.25in}
\end{table}

\begin{table}[ht]\small
\caption{Cluster $\mathcal{C}_\alpha$ \texttt{EqOp} fairness performance under Gaussian distribution without equalized distance, label rate and group rate}
\label{table9}
\begin{center}
\begin{tabular}{cccccc}
\multicolumn{1}{c}{\bf Distribution} &\multicolumn{1}{c}{\bf Label rate}  &\multicolumn{1}{c}{\bf Group rate}  &\multirow{2}{*}{\bf $\Delta^\alpha_{\texttt{EqOp}}(\theta^{*}_\alpha)$} &\multicolumn{2}{c}{\multirow{1}{*}{\bf $\Delta^\alpha_{\texttt{EqOp}}(\theta^{*}_G)$}}\\
\multicolumn{1}{c}{($\mu^1_a$, $\mu^0_a$, $\mu^1_b$, 
$\mu^0_b$, $\sigma$)} &\multicolumn{1}{c}{($\alpha^1_a$, $\alpha^0_a$, $\alpha^1_b$, 
$\alpha^0_b$)} &\multicolumn{1}{c}{$(r_a, r_b)$} &\multicolumn{1}{c}{} &\multicolumn{1}{c}{$p=\frac{2}{3}$} &\multicolumn{1}{c}{$p=\frac{1}{2}$} \\
\hline \\
\multirow{6}{*}{(7, 4.5, 6, 3, 2)} &\multirow{3}{*}{(0.7, 0.3, 0.6, 0.4)}  &(0.5, 0.5)  & 0.0993 &0.1251 $\uparrow$ &0.1796 $\uparrow$\\
 &  &(0.3, 0.7)  & 0.0947  &0.1115 $\uparrow$ &0.1780 $\uparrow$\\
 &  &(0.7, 0.3)  & 0.1045 &0.1635 $\uparrow$ &0.1814 $\uparrow$\\
 \cline{2-6}\\
 &\multirow{3}{*}{(0.3, 0.7, 0.2, 0.8)}  &(0.5, 0.5)   & 0.1948 &0.1610 $\downarrow$ &0.1558 $\downarrow$\\
&  &(0.3, 0.7)   & 0.1967 &0.1610 $\downarrow$&0.1558 $\downarrow$\\
&  &(0.7, 0.3)   & 0.1924 &0.1615 $\downarrow$ &0.1558 $\downarrow$\\
\hline \\
\multirow{6}{*}{(7, 4, 6, 3.5, 2)} &\multirow{3}{*}{(0.7, 0.3, 0.6, 0.4)}  &(0.5, 0.5)  & 0.1051 &0.1409 $\uparrow$ &0.1799 $\uparrow$\\
 &  &(0.3, 0.7)  & 0.1016  &0.1293 $\uparrow$ &0.1776 $\uparrow$\\
 &  &(0.7, 0.3)  & 0.1080&0.1564 $\uparrow$ &0.1822 $\uparrow$\\
 \cline{2-6}\\
 &\multirow{3}{*}{(0.3, 0.7, 0.2, 0.8)}  &(0.5, 0.5)   & 0.1959 &0.1625 $\downarrow$ &0.1569 $\downarrow$\\
&  &(0.3, 0.7)   & 0.1950 &0.1605  $\downarrow$&0.1553  $\downarrow$\\
&  &(0.7, 0.3)   & 0.1964 &0.1650  $\downarrow$ &0.1579 $\downarrow$
\end{tabular}
\end{center}
\vspace{-0.25in}
\end{table}

From Table~\ref{table7}, we can observe that when the majority of samples are labeled 1 (rows 1-3 and 7-9), there exists a cluster size weight $p$ such that the FedAvg solution would lead to a worse \texttt{SP} fairness performance compared to the clustered FL solutions, which also experimentally extends our findings in the Proposition~\ref{prop1} to the case of an unequalized gap. However, when the majority of samples are labeled differently (rows 4-6 and 10-12), we could find that when $\mu^1_a - \mu^0_a > \mu^1_b - \mu^0_b$, there exists a $p$ such that the FedAvg solution would lead to a worse \texttt{SP} fairness performance, and a distinct outcome occurs when $\mu^1_a - \mu^0_a < \mu^1_b - \mu^0_b$. One reason for the distinct behaviors is that the corresponding condition is not satisfied for the experiments in rows 4-6. Additionally, we find that as $p$ enlarges in row 11, the fairness gap decreases, and it could have better fairness performance than using the clustered FL solution. This observation is also consistent with the previous finding that the fairness gap would increase initially and then decrease in the proof of Proposition~\ref{prop1}. As we described earlier, it is clearly that for row 11, $p=1/2$ is not in the range of $(p^{\mathcal{C}_\alpha}_{low}, p^{\mathcal{C}_\alpha}_{high})$. 

From Table~\ref{table9}, we could observe that when the condition $\alpha^1_g \geq \alpha^0_g$ holds (rows 1-3 and 7-9), the changes in the group rates, label rates, and distribution distance do not affect the \texttt{EqOp} fairness performance in the cluster $\mathcal{C}_\alpha$. There exists a cluster size weight $p$ such that the FedAvg solution would lead to a worse fairness performance. However, when the condition is not met (rows 4-6 and 10-12), the FedAvg solution would have a better \texttt{EqOp} fairness performance.

\end{document}